\documentclass[10pt]{article} 
\usepackage{wasysym}  

\usepackage[preprint]{tmlr}
\usepackage[toc,page]{appendix} 
\usepackage{xcolor}
\usepackage{tikz}
\usepackage{pifont}
\usepackage{fancyhdr}
\usepackage{graphicx}
\usepackage{float}
\usepackage{cancel}
\usepackage[mathscr]{eucal}
\usepackage{mathtools}
\usepackage{hyperref}
\usepackage{url}
\usepackage{amsfonts,bm, amssymb}

\usetikzlibrary{shapes,arrows,positioning,calc,matrix,patterns}

\definecolor{matrixA}{RGB}{144,238,144}  
\definecolor{matrixW}{RGB}{255,160,122}  
\definecolor{matrixB}{RGB}{173,216,230}  
\definecolor{matrixV}{RGB}{221,160,221}  

\usepackage[export]{adjustbox}
\usepackage{fancyhdr}	
\usepackage{graphicx}	
\usepackage{cancel}		
\usepackage[mathscr]{eucal}
\usepackage{mathtools}
\usepackage{latexsym}
\usepackage{booktabs}
\usepackage{algorithm2e}
\RestyleAlgo{ruled}
\SetKwProg{Fn}{Function}{:}{}
\SetKw{KwReturn}{return}
\usepackage{multicol}

\newcommand{\norm}[1]{\left \|#1 \right\|_{\sobolev} }

\newcommand{\sx}{ \mathrm{x}  }
\newcommand{\vx}{ \boldsymbol{\mathrm{x}}  }
\newcommand{\mx}{ \boldsymbol{\mathrm{X}}  }
\newcommand{\weight}{ \boldsymbol{\mathrm{w}} }
\newcommand{\tx}{ \boldsymbol{\mathcal{X}} }
\newcommand{\tT}{ \boldsymbol{\mathcal{T}} }
\newcommand{\domain}{\mathrm{D}}
\newcommand{\sobolev}{W^{k,p}(\domain)}

\newcommand{\sing}{\rho^{(\mathrm{x})}_{\mathrm{MAX}}}
\newcommand{\singw}{\rho^{(\weight)}_{\mathrm{MAX}}}
\newcommand{\singmin}{\rho^{(\weight)}_{\mathrm{MIN}}}
\newcommand{\clcost}{J(\weight(t),\psi(t),\tx(t))}
\newcommand{\clcostt}{J(\weight(t+1),\psi(t+1),\tx(t+1))}
\newcommand{\ea}{\left( \int_{\vx \in \tx(\tau)} 
       \ell(\fhat(\weight(t),\psi(t))( \vx ) ) \right)}
\newcommand{\partials}{\partial^1}

\usepackage{hyperref}
\usepackage{url}

\title{The Effect of Architecture During Continual Learning}
\author{\name Allyson Hahn \email ahahn2813@gmail.com\\
      \addr Mathematics and Computer Science \\
      Argonne National Laboratory
      \AND
      \name Krishnan Raghavan\footnote{Corresponding Author} \email kraghavan@anl.gov \\
      \addr Mathematics and Computer Science \\
      Argonne National Laboratory }

\newtheorem{theorem}{Theorem}
\newtheorem{lemma}[theorem]{Lemma}
\newtheorem{proposition}[theorem]{Proposition}
\newtheorem{definition}[theorem]{Definition}
\newtheorem{remark}{Remark}

\newtheorem{proof}{Proof}

\newcommand{\innerprod}[1]{\left\langle #1 \right\rangle}
\newcommand{\abs}[1]{\lvert #1 \rvert}

\newcommand{\fhat}{\hat{f}}

\def\month{MM}  
\def\year{YYYY} 
\def\openreview{\url{https://openreview.net/forum?id=XXXX}} 

\begin{document}
\maketitle
\begin{abstract}
Continual learning is a challenge for models with static architecture, as they fail to adapt to when data distributions evolve across tasks. We introduce a mathematical framework that jointly models architecture and weights in a Sobolev space, enabling a rigorous investigation into the role of neural network architecture in continual learning and its effect on the forgetting loss. We derive the necessary conditions for the continual learning solution and prove that learning only the model weights is insufficient to mitigate catastrophic forgetting under distribution shifts. Consequently, we prove that by learning the architecture and the weights simultaneously at each task, rather than the weights alone, we can reduce catastrophic forgetting. 

To learn weights and architecture simultaneously, we formulate the continual learning as a bilevel optimization problem: the upper level selects an optimal architecture for a given task, while the lower level computes optimal weights through dynamic programming over all tasks. To solve the upper-level problem, we introduce a derivative-free direct search algorithm to determine the optimal architecture. Once found, we must transfer knowledge from the current architecture to the optimal architecture. However, the optimal architecture will result in a weights parameter space different from the current architecture (i.e., dimensions of weights matrices will not match). To bridge the dimensionality gap, we develop a low-rank transfer mechanism to map knowledge across architectures of mismatched dimensions. Empirical studies across regression and classification problems, including feedforward, convolutional, and graph neural networks, demonstrate that our approach of learning the optimal architecture and weights simultaneously yields substantially improved performance~(up to two orders of magnitude), reduced forgetting, and enhanced robustness in the presence of noise compared with static architecture approaches. 
\end{abstract}
\section{Introduction}
With the advent of large-scale AI models, the computational expense of training such models has also increased drastically; training these models efficiently often requires large-scale high-performance computing infrastructure. Despite such drastic expenditure, these models quickly become stale. The reason  is that the data distribution shifts or new data  gets generated, 
requiring realignment. This necessity presents a quandary. Naive realignment leads to a phenomenon called catastrophic forgetting, where the model overwrites prior information. On the other hand, a complete retraining of the model would require significant resources, a solution that is extremely unattractive. 

A more viable option is  continual learning, where approaches seek to constrain the parameters of the model to reduce the amount of catastrophic forgetting the model experiences. Several works in the literature have attempted this direction, starting from early work on small multilayer perceptrons~\cite{mccloskey1989catastrophic} to recent works on large language models such as \cite{luo2025empirical, lai2025reinforcement, biderman2024lora} and \cite{lin2025continual}. One key commonality across all  these approaches is that they look to constrain the weight/parameters of the AI model to induce minimal forgetting. On the other hand, and more appropriately, some approaches seek a balance between forgetting and the learning of new data~(generalization); see, for example, ~\cite{raghavan2021formalizing, lu2025rethinking,lin2023theory}. Despite demonstrated success, this is not the full story. In fact,  under mild assumptions, it has been proven that simply changing the weight of the AI model is not sufficient to capture the drift in the data distribution, and the capacity of the neural network (its ability to represent tasks) eventually diverges if the data distribution keeps changing ~\cite{chakraborty2025understanding}. These intractability results from \cite{chakraborty2025understanding} highlight a conundrum:  if the capacity of the network eventually diverges, then should we still aim to continually train AI models?
In this paper we demonstrate that the answer to this question is indeed ``yes"; in fact, \textit{the solution is to reliably change the architecture of the AI model on the fly according to the needs of the data distribution.} 

However,  three key bottlenecks remain before attaining this goal. (i) Reliably changing the architecture requires a methodical understanding of ``how the weights and the architecture of the model interact with each other," that is, the coupling between the weights and the architecture. (ii) Moreover, to decide when to change the architecture, one must understand how the forgetting-generalization trade-off is affected by the aforementioned coupling. Loosely speaking, the first two bottlenecks have been heuristically addressed in the literature. For example, some recent works such as  CLEAS (Continual Learning with Efficient Architecture Search)~\cite{CLEASgao2022CLEAS} and SEAL (Searching Expandable Architectures for Incremental Learning)~\cite{SEALgambella2025SEAL} involve dynamically expanding the capacity of the AI model by adding layers and defining metrics that govern the necessity for a change in capacity. However, none of these approaches involves or enables a generic ability to change different aspects of the architecture, number of parameters, activation functions, layers, and so forth. The key reason is that if the architecture is generically changed, one must initialize the new network architecture at random. (iii) Then, the transfer of information about the weights from the previous architecture to the new architecture across parameter spaces of two different shapes is required. Doing so, however,  is currently impossible. In the absence of such a transfer mechanism, current approaches such as CLEAS and SEAL can  modify only those components that would not warrant information transfer between parameter space of two different sizes. 

To address these bottlenecks, we  present a comprehensive approach that includes a new formulation to understand the coupling between the architecture and the weights of the model, a methodical way of searching for a generic new architecture, and a novel method for transferring information between the old  and the new architecture.

\subsection{Contribution}
The key reason the coupling between the architecture and the weights is difficult to model is that architecture dependencies are observed on a function space across tasks and the weight dependencies are visible across the Euclidean parameter space. Any modeling in the weight space such as presented in \cite{chakraborty2025understanding, raghavan2021formalizing} or \cite{lu2025rethinking} cannot capture function space dependencies. To obviate this situation, we model the problem of continually training the AI model over a sequence of tasks in a Sobolev space. Using this theoretical framework, we describe how, in the Sobolev space of AI models parameterized by architecture choices and  weights, weights alone cannot capture the change in the distribution:  the architecture must be changed. We then employ a derivative-free architecture search  to determine the new architecture by searching for an appropriate number of neurons.  Although we focus on changing the number of parameters in the AI model, our work is general enough to extend to other architecture choices as well. Once the new architecture is chosen, we develop a new algorithm that can efficiently transfer information from the previous architecture to a new architecture across parameter spaces of different sizes with minimal loss of performance on the previous tasks.

We  empirically demonstrate that changing the architecture indeed results in better training loss compared with when the architecture is not changed. We also show that our algorithm achieves substantial improvements in terms of training loss when training over large numbers of tasks; it is robust to noise and scales from feedforward neural networks to graph neural networks seamlessly across regression and classification problems. We envision that jointly training the architecture and the weights in the continual learning paradigm is the best path forward in the field of continual learning, and we develop the basic theoretical and empirical infrastructure to allow for such training.

\subsection{Organization}
The paper is organized as follows. We begin with a description of continual learning (CL) and motivate the necessity for function space modeling. Next, we describe a collection of neural networks as functions in a Sobolev space. We then derive the necessary condition for the existence of a CL solution and demonstrate that weights alone are not enough to train a network efficiently. Thus, we describe the need to change the architecture of the network along with the weights of the AI model. We  present an algorithm to change the architecture on the fly and transfer information between the two architectures. We describe empirically the advantages of our approach and conclude with a brief summary.

\subsection{Notation}
The notation is adapted from \cite{kolda2009tensor}; please refer to the original paper for additional details. We use $\mathbb{N}$ to denote the set of natural numbers and $\mathbb{R}$ to denote the set of real numbers.  An $m^{th}$-order tensor is viewed as a multidimensional array contained in $\mathbb{R}^{I_1 \times I_2 \times I_3 \times I_4 \times \cdots I_{m}}$,  where the order can be thought of as the number of dimensions in the tensor. In this paper we will  be concerned mostly with tensors of order $0, 1, 2$ and $3$, which correspond to scalars, vectors, matrices, and a list of matrices. Therefore, we will write a tensor of order $0$ as a scalar with lowercase letters,  $\sx;$  a tensor of order one is denoted by lowercase bold letters such as $\vx.$ A tensor of order $2$ is 
denoted by uppercase bold letters $\mx$, and any tensor of order greater than $2$ is denoted by bolder Euler scripts letters such as $\tx.$  We also use $\|.\|$ to denote the Euclidean norm for vectors, Frobenius norm for matrices, and an appropriate tensor norm for tensors. Further, we will use $\innerprod{\cdot,\cdot}$ to denote the dot product for vectors, matrices, or tensors. We will make one exception in our notation regarding the tensors that represent learnable/user-defined parameters~(architecture, step-size/learning rate, etc.): we will denote these with Greek letters. The $i^{th}$ element of a vector $\vx$ is denoted by $\sx_{[i]},$ while the $(i,j)^{th}$ element of a matrix $\mx$ is denoted by $\sx_{[ij]}.$ Moreover, we denote the $i^{th}$ matrix in a tensor of order $3$ by $\mx_i.$  We will make the indices run from 1 to their capital letters such that $i = 1, \ldots, I.$

\section{Continual Learning -- Motivation}~\label{sec:motivation}
The problem of continual learning is that of an AI model learning a sequence of tasks. Here, we define a task as a set of data available to learn from. Without going into the specifics of the underlying space of an AI model, we will denote the AI model as $\hat{f}(\weight, \psi)$, where  $f$ denotes the composition of the weights$-\weight$ and the architecture$-\psi.$ We will define the particulars of these quantities in a later section; however, the key objective of this model is to learn a series of tasks. In this context, each task is represented by a data set obtained at a task instance, that is, $t \in \tT, \tT = \{0,1,\ldots, T\}.$ We will assume that the dataset~(a list of matrices, vectors, graphs) represented by $\tx(t)$ is provided at each task $t\in \tT,$ where $\tx(t)$ is sampled according to the distribution $\mathbb{P}$ and $\tx(t) \subset \domain \subset \mathbb{R}^n$ such that $\domain$  is a measurable set with a non-empty interior. Moreover, $(\domain, \mathcal{B}(\domain), \mathbb{P})$ forms a probability triplet with $\mathcal{B}(\domain)$ being the Borel sigma algebra over the domain $\domain$. 
\begin{figure}
    \centering
    \includegraphics[width=\linewidth]{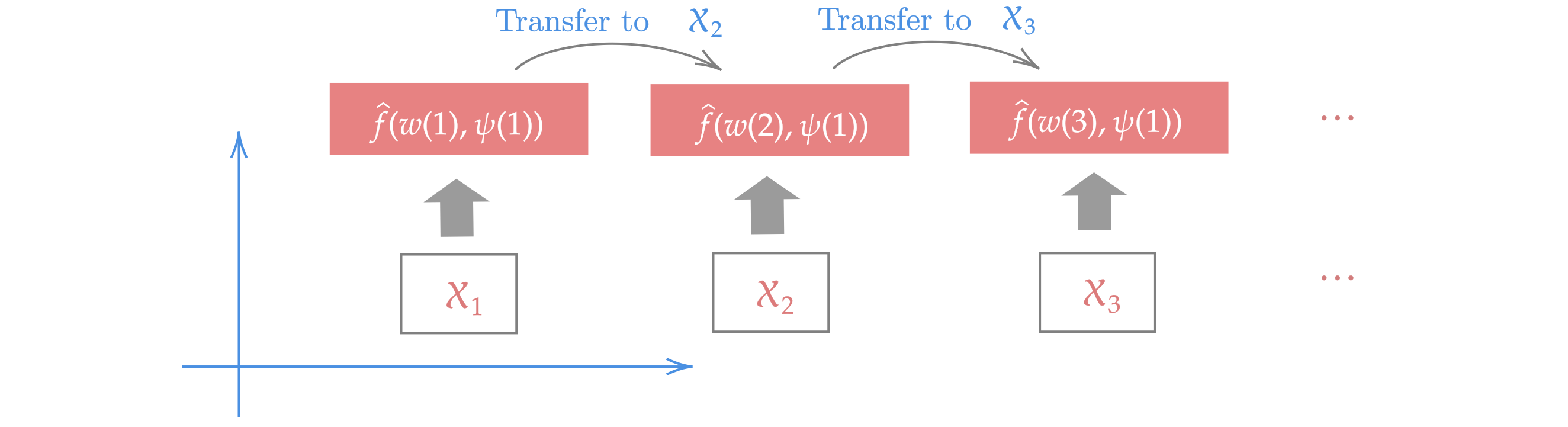}
    \caption{The basic problem of continual learning}
    \label{fig:CL1}
\end{figure}
The problem of interest to us is to understand the behavior of the AI model when it is learning tasks $\tx(t)$ for every task instance $t.$ Task instances can belong to $\mathbb{N}$ and $\mathbb{R}$ depending on the application at hand, but the key goal is to understand how the tasks are assimilated by the AI model. The basic notion of continual learning is described in Figure ~\ref{fig:CL1}, where there are three tasks at $t=1, 2$, and $3$. For each of these three tasks, an architecture $\psi(1)$ is chosen, and the three tasks are provided to the model in series. At $t=1$, the model $\hat{f}$ learns from the first task. In particular, we solve the optimization problem 
\begin{align}
min_{\weight(1)} \ell(\hat{f}(\weight(1), \psi(1)), \tx(1)),
\end{align}
where $\ell$ is some loss function that summarizes the model's performance on the task at $t=1.$ Then, at $t=2,$ the model transfers to learn information from the second task such that the first task is not forgotten. Implicitly, the optimization problem is rewritten as 
\begin{align}
    min_{\weight(2)} \left[ \ell(\hat{f}(\weight(2), \psi(1)), \tx(1)) + \ell(\hat{f}(\weight(2), \psi(1)), \tx(2)) \right].
\end{align}
Similarly, at $t=3,$ the model transfers to learn information from the third task, but both tasks 1 and 2 must not be forgotten. This implies that 
\begin{align}
    min_{\weight(3)} \left[ \ell(\hat{f}(\weight(3), \psi(1)), \tx(1)) + \ell(\hat{f}(\weight(3), \psi(1)), \tx(2)) + \ell(\hat{f}(\weight(3), \psi(1)), \tx(3))\right] .
\end{align}
As shown, the objective of the model is to remember all earlier tasks. This implies that the objective function of the continual learning problem is a sum that grows with every new task. Therefore, at any task $t,$ the objective of the CL problem is 
\begin{align}~\label{eq:CL_for}\tag{Forgetting Loss}
    min_{\weight(t)} \sum_{i=1}^t \left[\ell(\hat{f}(\weight(i), \psi), \tx(i)) \right], 
\end{align}
where we have removed the index from $\psi$ to indicate that the architecture is fixed. In the CL literature, \eqref{eq:CL_for} is typically optimized at the onset of every new task at $t.$ The most rudimentary approach for this optimization  involves an experience replay array that maintains a memory of all tasks in the past. To gain more insight into what happens at the onset of each new task, see Figure \ref{fig:param}.

\begin{figure}[!htb]
    \centering
    \includegraphics[width=0.6\linewidth]{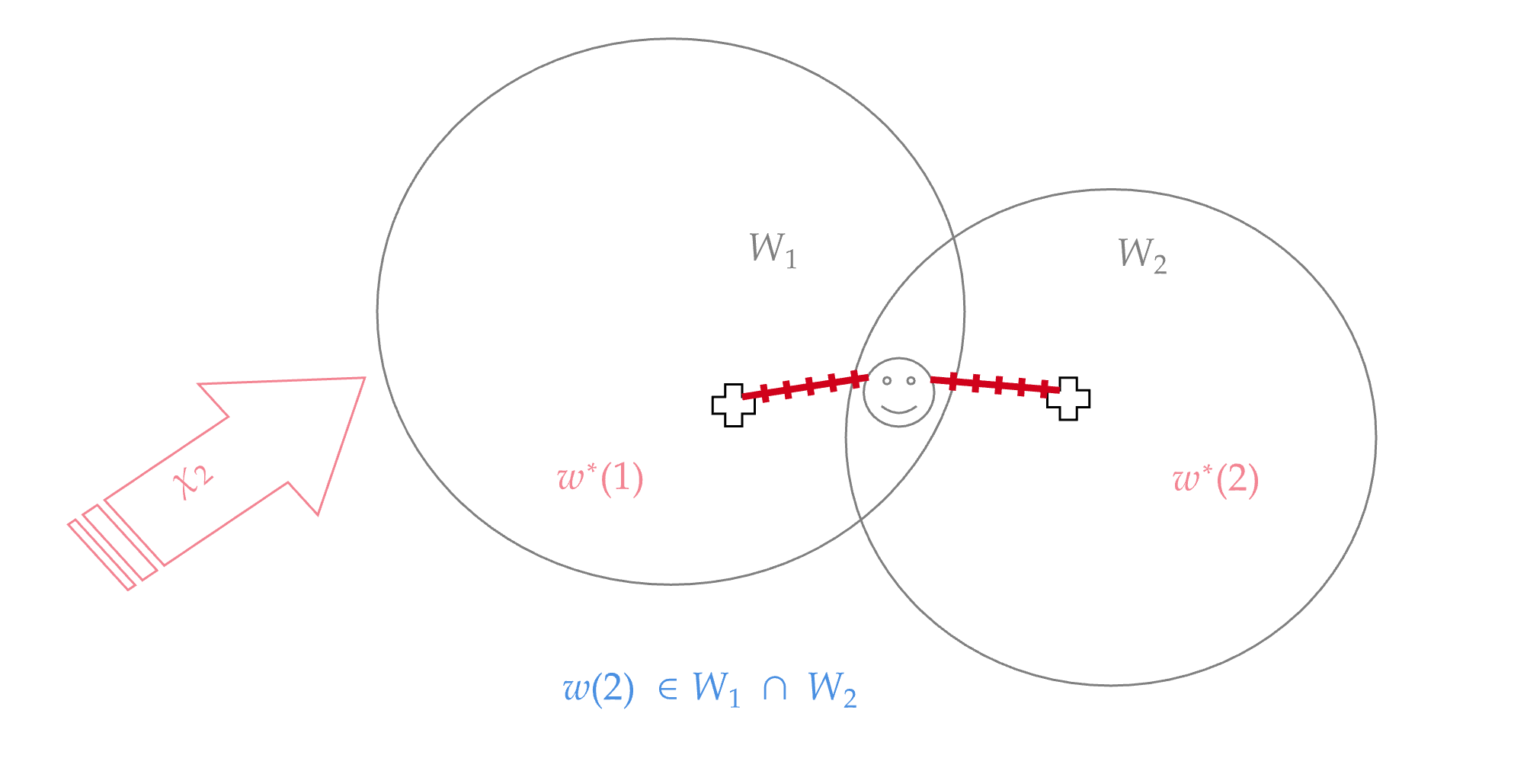}
    \caption{The basic problem of continual learning}
    \label{fig:param}
\end{figure}
For the sake of illustration, we assume that $\ell(\hat{f}(\weight(i), \psi, \tx(i))$ is a strongly convex function  described in Figure ~\ref{fig:param}, the level set~(the balls) in the parameter space corresponding to $\ell(\hat{f}(\weight(i), \psi, \tx(i)) \leq \eta$  and the plus sign corresponding to the minima for $i=1$ and $i=2.$   The threshold to identify acceptable loss values is $\eta$; it  defines the boundary of the level set and, by extension, the radius of the ball in Figure~\ref{fig:param}.  For instance, in the  MNIST dataset, the two balls may represent two tasks  representing digit 0 and digit 1. The boundary of the left ball represents all parameters that provide cross-entropy loss less than $\eta=1$ for digit 0, with the plus sign representing zero cross-entropy loss. Similarly, the right ball represents all parameters that provide cross-entropy loss less than $\eta=1$ for digit 1, and the plus sign corresponds to zero cross-entropy loss on digit 1. Within this example, consider an AI model~(maybe a convolutional neural network)  that trains on digit 0 and achieves zero loss on digit 0, that is, attains the plus sign in the parameter space. Then, the AI model starts training for the next task---digit 1. First, it starts from the plus sign achieved for digit 0; this implies that there is an inherent bias due to the local minima achieved for digit 0. Second, when the new task is observed, the model must converge to a solution that is common to both digit 0 and digit 1. This solution is trivial if  digits 0 and 1 are identical or very close to each other. Since this is not the case, however, the solution must lie in the intersection between the two balls, that is, the region of the smiley face. Since the size of the intersection space depends on how different tasks 0 and 1 are, the following informally stated theorem is observed.

\begin{theorem} 
\label{thm:task_nonstationary_weights}
 Fix the number of weight updates required to obtain the optimal value at each task. Assume that each subsequent task introduces a value of change into the model with the change described in some appropriate metric. Let $\ell$ be the Lipschitz continuous function, and choose a small learning rate. Then, capacity diverges as a function of the number of tasks.
 \textbf{Please see \cite{chakraborty2025understanding} for a full statement and its proof.}
\end{theorem}

The theorem implies that the AI model's ability to successfully represent tasks will deteriorate as more tasks are introduced when the new tasks are different from the previous one. In such a case, continually training the AI model is bound to fail if we just update the weights of the network. In this paper we hypothesize that the size of the intersection space in Figure.~\ref{fig:param} can be increased by modifying the architecture of the AI model, that is, by increasing the number of parameters in the model. Doing so, however,  would require us to model the coupling between the weights, data, and the architecture, a task impossible to imagine with the intuition laid out until now since the notion of architecture cannot be modeled with this intuition. To this end, we will describe a neural network as a member of a class of functions in a function space that we choose to be a Sobolev space.

\section{Neural Networks --  Members of a Class of Sobolev Space Functions}
The goal of this section is to formalize the mathematical modeling required to identify the dependency between the intersection space from Figure~\ref{fig:param}, the architecture and the AI model. We first define the AI model $\hat{f}(\weight, \psi)$ to belong to a class of functions contained in a Sobolev space with $k$-bounded weak derivatives. The notion of $k$-bounded weak derivatives is the key reason that Sobolev space is an appropriate choice for this problem, which we will discuss once we outline the formal definitions of the Sobolev space aligning with definitions provided in \cite{mahanNonclosednessSetsNeural2021a}.

\begin{definition}\label{defn:sobo}
    Let $k \in \mathbb{N},$ the domain-$\domain \subset \mathbb{R}^n$ a measurable set with non-empty interior, and $1< p < \infty.$ Then the Sobolev space $\sobolev$ consists of all functions $f$ on $\domain$ such that for all multi-indices $\alpha$ with $|\alpha| \leq k,$ the mixed partial derivative $\partial^{(\alpha)} f$ exists in the weak sense and belongs to $L^p(\domain)$. That is, 
    \[ \sobolev  = \{ f \in L^p(\domain) : \partial^{|\alpha|} h \in L^{p}(\domain) \forall |\alpha| \leq k \}.\]
    The number $k$ is the order of the Sobolev space, and the Sobolev space norm is defined as 
    \[ \|f\|_{\sobolev} := \sum_{ |\alpha| \leq k } \| \partial^{|\alpha|} f \|_{L^{p}(\domain)}.\]
\end{definition}
In Definition \ref{defn:sobo}, $\domain$ is our data space. For any standard AI application, we can assume that $\domain$ is a finite subset of $\mathbb{R}^n.$ This is an important practical consideration as, for AI model training to be efficient, one needs the numerical values of the data to be bounded, typically achieved through normalization~\cite{huang2023normalization}. As $\domain$ is a finite subset of $\mathbb{R}^n,$ it is a closed subset. Thus, by Heine--Borel, $\domain$ is also compact, see ~\cite{rudin1976principles}.

The next assumption we note is the presence of weak derivatives. In a standard neural network, derivatives are assumed to exist since twice differentiability is necessary for training and typical training involves gradient-based methods~\cite{kingmaAdamMethodStochastic2017}. In fact, the existence of weak derivatives can be summarized for different activation functions as in Table~\ref{tab:Sobolev_acti}. As noted in the table, all standard activation functions can be recast in the purview of Definition~\ref{defn:sobo}, and a Sobolev space can be constructed that represents the class of functions that can be approximated by neural networks. The appropriateness of approximation and the conditions over which such approximations are possible are discussed in \cite{petersenTopologicalPropertiesSet2021a}  and \cite{adegokeHigherDerivativesInverse2016}.
\begin{table}[!tb]
\adjustbox{width=\linewidth}{
\begin{tabular}{cccc}
\toprule
       Name  &  $\rho(\vx)$ &  Smoothness & Corresponding $W^{k,p}$\\
        &   & /Boundedness & \\
       \hline  \hline
Rectified Linear  & $sup\{0,\vx\}$ & absolutely  & $W^{0,p}(\domain)$ \\
 Unit (ReLU) &  &  continuous, $\rho' \in L^{p}(\domain)$ & 
\\ & &  &\\ 
 Exponential Linear  & $ \vx. \chi_{\vx\geq 0}$ & $C^1(\mathbb{R})$, absolutely  & $W^{1,p}(\domain)$ \\
 Unit (eLU) & $+ (e^{\vx}-1) \chi_{\vx<0} $  & continuous $\rho'' \in L^{p}(\domain)$ & 
\\ & & &  \\ 
Softsign & $\frac{\vx}{1+ |\vx|}$ & $C^1(\mathbb{R})$, $\rho'$ absolutely  & $W^{2,p}$  \\
& &  continuous, $\rho'' \in L^{p}(\domain)$ &  
\\ & & &  \\ 
Inverse Square  & $ \vx \chi_{\vx \geq 0}  $ & $C^2(\mathbb{R})$, $\rho''$ absolutely  & $W^{3,p}$  \\ 
Root Linear Unit & $ + \frac{\vx}{\sqrt{1+a\vx^2}} \chi_{\vx<0} $ & continuous, $\rho''' \in L^{p}(\domain)$ & 
\\ & & &  \\ 
Inverse Square Root Unit & $ \frac{\vx}{\sqrt{1+a\vx^2}} $ & real analytic,  & $W^{k,p} \forall k $ \\
&  & real, all derivatives bounded &  
\\ & & &  \\ 
Sigmoid & $ \frac{1}{1+ e^{\vx} } $ & real analytic,  & $W^{k,p} \forall k $ \\
&  & real, all derivatives bounded &  
\\ & & &  \\ 
tanh & $ \frac{e^{\vx} - e^{-\vx}}{e^{\vx} + e^{-\vx}} $ & real, analytic,  & $W^{k,p} \forall k $ \\
&  & real, all derivatives bounded &  \\
\bottomrule
\end{tabular}
}
\caption{Different activation functions and the corresponding Sobolev spaces. Many of these results are described in 
\cite{petersenTopologicalPropertiesSet2021a}  and \cite{adegokeHigherDerivativesInverse2016} . $\chi$ is an indicator function.}
\label{tab:Sobolev_acti}
\end{table}

To represent the class of functions described in Definition \ref{defn:sobo}, we  define a neural network as a function  $\fhat(\weight(t),  \psi(t) )$ with $d \in \mathbb{N}$ layers. In this notation, $\weight(t)$ is a long $\mathbb{R}^{m}$ vector that concatenates the weight parameters from the $d$ layers, and $\psi(t)$  comprises  the architecture and other user-defined parameters in the neural network.
\begin{definition} \label{defn:NN}
A $d$-layered neural network is given by an operator~(essentially a function) 
$\fhat(\weight(t), \psi(t) )\in \sobolev$ with $\sobolev$ being a Sobolev space. Furthermore, 
\begin{align}
  \fhat(\weight(t), \psi(t)) \big( . \big) & = \fhat_{d}(\mathrm{w}_{d}(t), \psi_{d}(t))\circ \fhat_{d-1}(\mathrm{w}_{d-1}(t), \psi_{d-1}(t))\circ \cdots \circ \fhat_{2}(\mathrm{w}_{2}(t), \psi_{2}(t))\circ \fhat_{1}(\mathrm{w}_{1}(t), \psi_{1}(t)) \big( . \big)
\end{align}
describes the layer-wise compositions, and $\big( . \big)$ represents the input tensor to which the operator is applied.
\end{definition}
\begin{remark}
    We indicate approximate quantities with $\hat{.}$ and true~(target) quantities without the hat. For instance, the target function $f$ in Definition~\ref{defn:sobo} is indicated without the hat, and the approximate function $\fhat$ in definition~\ref{defn:NN} is indicated with a hat.
\end{remark}
\begin{remark}
    In Definition~\ref{defn:NN} we describe the operation across different layers using the composition; in terms of notation, we hide this complexity through the definition of the operator. Therefore, Definition~\ref{defn:NN} is generic and  can be used to define feedforward, recurrent, convolutional, and graph  neural networks and even a large language model. Therefore, any analysis from this point is applicable to all types of neural networks. 
\end{remark}
Typically, the user-defined parameters are a combination of integer, categorical, and floating-point values; and we assume that the architecture can be varied with task instances. Therefore, the parameters corresponding to the architecture $\psi$ are assumed to be a function of the task instance $t.$
Since we cannot determine derivatives with respect to discrete parameters in the classical sense (i.e., derivative of $\psi(t)$), we need to define weak derivative with respect to these quantities such that the derivatives~(with respect to or of $\psi$) can be approximated. To describe the learning problem, we define a notion of loss function on the space of functions defined in Definition~\ref{defn:NN} and \ref{defn:sobo}. 
 \begin{definition}\label{defn:NN_learning}
        Let $f(t)\in \sobolev$ for all $t\in \{0,\ldots, T\}$ denote the \textit{target function} that is to be approximated, and let $\hat{f}(\weight(t),\psi(t))$ denote a neural network determined to approximate $f(t)$ at task $t.$ Then the \textit{loss function} (or error of approximation) for $\vx \in \tx(t) \subseteq \domain$ is given by 
        \begin{align*}
            \ell(\fhat(\weight(t),\psi(t)), \vx) & = \| \fhat(\weight(t),\psi(t))(\vx)-f(t)(\vx)\|_{\sobolev}.
        \end{align*} 
    \end{definition}
The function $f$ in the context of machine learning is an ideal forecasting model that can forecast temperature perfectly or an ideal classification model that performs some classification problem.
\begin{remark}
    In applications, we  assume $\domain$ is compact to guarantee that the target function $f(t)$ exists in a Sobolev space $\sobolev,$ by Corollary 3.5 in \cite{mahanNonclosednessSetsNeural2021a}. If, however,  compactness cannot be assumed, we find the minimum norm solution to the problem in the Sobolev space.
\end{remark}
The loss function in Definition~\ref{defn:NN_learning} is defined over $\sobolev$ using the Sobolev space norm. In practice, however, we can directly utilize samples corresponding to $f$ or $\hat{f}$ to construct the loss function. For instance, this has been done in \cite{son2021sobolev} in the context of  physics-informed neural networks---a class of neural networks  modeled by Definition~\ref{defn:NN}. In particular, the Sobolev space norm is defined as in Definition~\ref{defn:sobo}, where the loss function is a sum over $L^p$ norm over all $k$ weak derivatives. As a trivial case, choosing $k=0$ and $p=2$ gives the standard mean squared error loss function that,  according to  Table~\ref{tab:Sobolev_acti} and practical consideration, is applicable to various activation functions used in practical AI.

\subsection*{Continual Learning in $\sobolev$}
We  extend Definition~\ref{defn:NN_learning} to the case of CL  with the goal of learning a series of tasks. In particular, for each task instance $t$ we must find a $\weight^*$  with respect to  $\ell$ integrated over the interval $[0,t].$   This objective function denoted $J$ is known as the forgetting loss and was informally defined as 
\begin{align*}
    \clcost & = \int_0^t \ea \;  d~\tau .
\end{align*}
A traditional learning structure as in \cite{CASpasunuru2019continual} seeks to minimize the cost $J(t)$ through an  optimizer~(one such as Adam~\cite{kingmaAdamMethodStochastic2017}) to drive $J(t) \rightarrow 0$  when $t \rightarrow \infty$,  a structure mathematically guaranteed to converge, as shown in \cite{kingmaAdamMethodStochastic2017}. In practice,  however, when consecutive tasks are very different from each other,  we encounter the issue described in Figure~\ref{fig:param}.  In order to resolve this quandary, for each new task, the underlying optimization problem must be repeatedly resolved, leading to a  series of optimization problems, altogether forming a dynamic program~\cite{bertsekas2012dynamic}. A standard practice in the optimal control literature is to solve such dynamic programs~\cite{bertsekas2012dynamic} through a cumulative optimization problem over the interval $[0,T]$ as in 
\begin{align*}
    V^*(t,\weight(t)) & = \min_{\weight \in \mathcal{W}(\psi^*(t))} V(t, \weight(t) ),
\end{align*}
where $\mathcal{W}(\psi^*(t))$ is the weights search space and $V(t,\weight(t))$ is 
\begin{align*}
    V(t,\weight(t)) & = \int_t^T J(\weight(\tau),\psi^*(t), \tx(\tau)).
\end{align*}
An alternative learning process that solves this dynamic program was introduced in   \cite{raghavan2021formalizing, chakraborty2025understanding} that uses a fixed architecture~$\psi^*(t)$.  However, as argued in Section~\ref{sec:motivation}, fixed architecture may not be enough in the CL paradigm. Therefore, we are interested in a scenario where the architecture and the model parameters are both learned. To achieve this, we must complete a bi-level optimization with the bottom level being a dynamic program. The upper problem is outlined as follows:
\begin{align}\label{eq:arch} \tag{Architecture search}
    \psi^*(t) & = \mathrm{arg }\min_{\psi\in \Psi} J(\weight(t),\psi,\tx(t)),
\end{align}
where $\Psi$ is our architecture search space. Once the search is completed, we move to the lower problem to find the weights as follows:
\begin{align}\label{eq:clcum} \tag{Cumulative CL}
     V^*(t,\weight(t)) & = \min_{\weight \in \mathcal{W}(\psi^*(t))} V(t,\weight(t)).
\end{align}
As discussed earlier, naively solving this problem will involve repeatedly retraining the model from scratch for each new task and each new architecture. This is both computationally and mathematically unattractive. 

\section{Understanding the Existence of the CL Solution}
In this paper we will devise an alternative approach to solving the bi-level problem.
First, however,  we must understand the dependence between the architecture, weights, and the \ref{eq:CL_for}. 
To understand this behavior, let us reconsider the example from 
Section~\ref{sec:motivation} and redraw Figure ~\ref{fig:challenge}. 

\begin{figure}[!thb]
    \centering
    \includegraphics[width=0.8\linewidth]{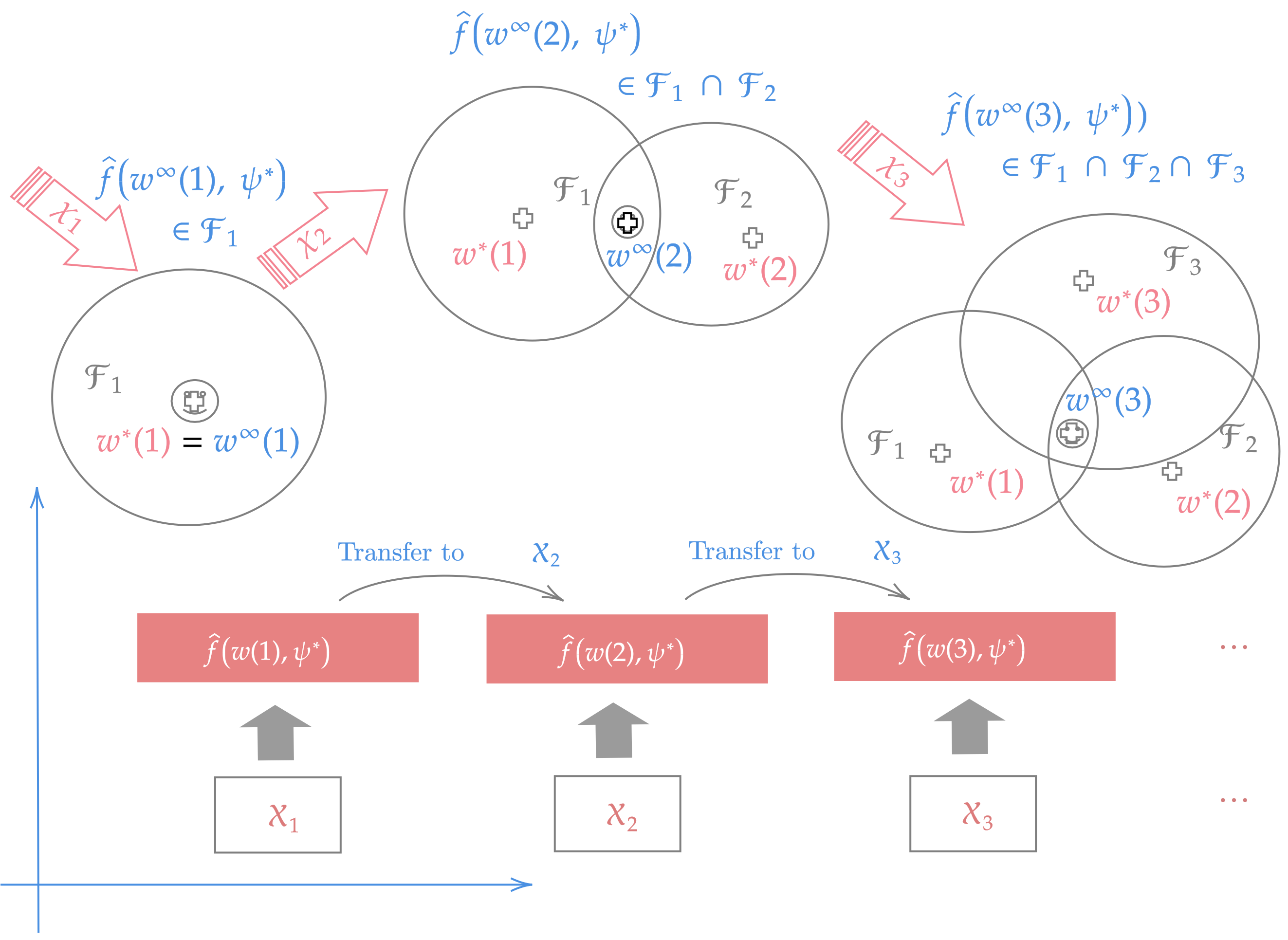}
    \caption{The actual problem}
    \label{fig:challenge}
\end{figure}

From Figure  \ref{fig:challenge}, let $\mathcal{F}_t = \{\fhat(\weight,\psi)\in \sobolev\}$ 
represent the search space for the set of  all possible neural network solutions for learning task 
$t$~(refer to  \ref{sec:motivation} for a similar construction). 
Since $\mathcal{F}_t$ is a subset of $\sobolev,$ the balls in Figure  \ref{fig:challenge} 
can be thought  of as a level set such that the boundary of the 
ball represents the set of all solutions that provide the 
loss value less than a threshold~(a threshold set by the user). 
At $t=1$ there is one ball, at task $t=2$ there are two balls, 
and at task $t=3$ there are three balls. 
The number of balls here corresponds to the number of tasks  
involved in CL. 
 
In each of the balls, \ding{58} represents the neural network solution for 
just task $t$, and the $\smiley{}$ face 
represents the solution that performs well among tasks 
between $0 \rightarrow t.$  Notably,  $\smiley{}$ must lie in 
the intersection of all the balls from $0$ to $t$. 
Note that unless the tasks are equivalent,  $\smiley{}$ 
and the \ding{58} will not coincide because no solution 
that is perfect for one task can be perfect for another. 
Furthermore, the more different the subsequent tasks are, the more likely
the intersection space is  to be empty, and the CL problem has no solution 
across all tasks from $0$ to $t$. From a different perspective, the non-emptiness 
of the intersection space implies 
that the union of all the balls is a connected set from a 
topological perspective. A continuous 
function such as $J$, defined on this 
connected set, will have a gradient that is 
defined completely over the union of these balls. Therefore, 
if the intersection space is empty, the connectivity of the union
of these balls is violated, and the gradient of $J$ is undefined, which will lead to 
a solution that cannot be found. The best one can do in this circumstance is a 
minimum norm solution as is done in the current literature. Mathematically, this intuition is described by a notion of continuity of $J$ on consecutive tasks, and a notion absolute continuity of $J$ over the interval  $[0,T]$~(the union of all balls in a complete interval).

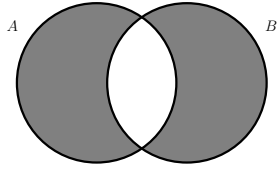
\begin{figure}
    \centering
    \scalebox{.4}{\begin{tikzpicture}
        \def\firstcircle{(-.5,0) circle (2.65cm)}
        \def\secondcircle{(2.5,0) circle (2.65cm)}
        
        \fill[even odd rule, gray]\firstcircle \secondcircle;
        
        \draw[draw=black,line width=.75mm] \firstcircle;
        \draw[draw = black, line width=.75mm] \secondcircle;
        \node at (-3.3,1.9) {\Large$A$};
        \node at (5.3,1.9) {\Large$B$};
    \end{tikzpicture}}
    \caption{Set symmetric difference $A\bigtriangleup B$ represented by the gray-shaded region }
    \label{fig:setsim}
\end{figure}

We begin with describing the notion of continuity we will use here. We define measure $\mu$ over the task probability space ~$\mathbb{P}$ and describe an $\varepsilon, \delta$ definition of 
continuity of a function w.r.t this measure defined over the set symmetric difference $\mu(A\bigtriangleup B)$.
\begin{definition} \label{defintersect}
     Let $\overline{\tx} = \bigcup_{t=0}^\infty \tx(t)$ with the power set $\mathcal{P(\overline{\tx})}$ as its topology. Set $\mathcal{B(\overline{\tx})}$ to be the Borel sigma algebra on $\overline{\tx}$ equipped with a probability measure denoted $\mu.$ Then $(\overline{\tx},\mathcal{B}(\overline{\tx}), \mu)$ forms a probability measure space. Let $g: \overline{\tx}\rightarrow \mathbb{R}$ be a measurable function, and set $F:\mathcal{B}(\overline{\tx})\rightarrow \mathbb{R}$ to be the function
     \begin{align*}
         F(A) & = \int_{\vx \in A} g(\vx) d\mu.
     \end{align*}
     Then we say that $F$ is \textit{continuous with respect to the measure} if for every $\varepsilon>0$ there exists $\delta >0$ such that $\abs{F(A)-F(B)} < \varepsilon$
     whenever $A,B\in \mathcal{B(\overline{\tx})}$ such that $\mu(A\bigtriangleup B)<\delta.$
 \end{definition}
 
Definition \ref{defintersect} can be deduced from basic measure 
theory results found in \cite{weaver2013measure} implying that 
sets that are similar~(with respect to the probability measure $\mu$) 
result in close values of $F(A)$, integrated over the respective sets. 
In particular, $\mu(A\bigtriangleup B)$ is the Lebesgue measure $\mu$
applied over $A\bigtriangleup B$ representing a set 
symmetric difference, which describes the elements 
two sets do not share~(see Figure  \ref{fig:setsim} 
for an illustration). Therefore, as the intersection space (i.e., $\mu(A\cap B)$) 
shrinks, $\mu(A\bigtriangleup B)$ gets larger, and $\abs{F(A)-F(B)}$ gets larger proportionally, thus providing a measure of continuity 
that is proportional to the number of common elements between the two sets. More common elements imply a smaller value of the measure, and less common elements 
imply a large value of the measure.

\subsection{Existence}
To begin, we prove that the expected value function, that is, $\int_{\vx\in A}
       \ell(\fhat(\weight,\psi
       )( \vx ) )d\mu$, is continuous with respect to measure $\mu$ when $\mu(A\triangle B)<\delta, \forall A, B \in \overline{\tx}.$ 
\begin{lemma}\label{lem:contMeasure}
    Let $(\overline{\tx},\mathcal{B}(\overline{\tx}), \mu)$ be the measure space as 
    defined in Definition~\ref{defintersect}. 
    Assume that the loss function $\ell$ is continuous 
    and bounded across all tasks, that is, $\forall t \in [0,T].$ 
    Then the expected value function 
    \begin{align}
        E(A) & = \int_{\vx\in A}
       \ell(\fhat(\weight,\psi
       )( \vx ) )d\mu.\\
    \end{align}
    is continuous with respect to measure $\mu$.
\end{lemma}

\begin{proof}
    Suppose that the constant $M_0>0$ is the value that bounds $\ell$ on every task. Let $\varepsilon>0$, and set $\delta = \varepsilon/M_0.$ Further, let $A,B\in \mathcal{B}(\tx)$ such that $\mu(A\triangle B)<\delta.$ By disjoint additivity of $\mu$ and the triangle inequality, we have
    \begin{align*}
        \abs{E(A)-E(B)} & = \vert \int_{\vx\in A}
       \ell(\fhat(\weight,\psi
       )( \vx ) )d\mu - \int_{\vx\in B}
       \ell(\fhat(\weight,\psi
       )( \vx ) )d\mu\vert\\
       & = \vert \int_{\vx\in A\setminus B}
       \ell(\fhat(\weight,\psi
       )( \vx ) )d\mu + \int_{\vx\in A\cap B}
       \ell(\fhat(\weight,\psi
       )( \vx ) )d\mu\\
       & - \int_{\vx\in B\setminus A}
       \ell(\fhat(\weight,\psi
       )( \vx ) )d\mu - \int_{\vx\in A\cap B}
       \ell(\fhat(\weight,\psi
       )( \vx ) )d\mu\vert\\
       & = \vert \int_{\vx\in A\setminus B}
       \ell(\fhat(\weight,\psi
       )( \vx ) )d\mu - \int_{\vx\in B\setminus A}
       \ell(\fhat(\weight,\psi
       )( \vx ) )d\mu\vert\\
       & \leq \int_{\vx\in A\setminus B}
       \abs{\ell(\fhat(\weight,\psi
       )( \vx ) )}d\mu + \int_{\vx \in B\setminus A}
       \abs{\ell(\fhat(\weight,\psi
       )( \vx ) )}d\mu.
    \end{align*}
    Then, by the boundedness of $\ell,$ we have
    \begin{align*}
        \int_{\vx\in A\setminus B}
       \abs{\ell(\fhat(\weight,\psi
       )( \vx ) )}d\mu + \int_{\vx\in B\setminus A}
       \abs{\ell(\fhat(\weight,\psi
       )( \vx ) )}d\mu & \leq M_0\mu(A\setminus B)+M_0\mu(B\setminus A).
    \end{align*}
    Hence,
    \begin{align*}
        M_0\mu(A\setminus B)+M_0\mu(B\setminus A) & = M_0 \mu(A\triangle B)\\
        & <M_0\delta\\
        & = M_0\frac{\varepsilon}{M_0}\\
        & = \varepsilon,
    \end{align*}
    as desired.
\end{proof}
Now, we show that if two tasks $\tx(t), \tx(t+1)$  are similar, then the corresponding loss values are similar. We prove this result in the following lemma. 
\begin{lemma}\label{lem: simtaskssimloss}
    Let $(\overline{\tx},\mathcal{B}(\overline{\tx}), \mu)$ be the measure space as 
    defined in Definition~\ref{defintersect}. 
    Assume that the loss function $\ell$ is continuous 
    and bounded across all tasks, that is, $\forall t \in [0,T].$ 
    Define $J$ over a task $\tx(t)$ as
    \begin{align}
        \clcost & =\int_0^t \ea \;  d~\tau  = \int_{ \bigcup_{\tau=0}^t\tx(\tau)}       \ell(\fhat(\weight(t),\psi(t))( \vx ) )d\mu . 
    \end{align}
    Let $\varepsilon >0$ and $\delta(\varepsilon)>0$ chosen such that $\ea$ is continuous w.r.t. measure $\mu$.
    Then, 
    $\mu(\bigcup_{\tau = 0}^{t+1} \tx(\tau)\triangle \bigcup_{\tau = 0}^{t} \tx(\tau))<\delta$, implies
    $\abs{\clcostt- \clcost} \leq \varepsilon.$
\end{lemma}

\begin{proof}
    Suppose that the constant $M_0>0$ bounds $\ell$ on every task. 
    Let $\varepsilon>0$, and  set $\delta = \varepsilon/M_0.$  By disjoint additivity of $\mu$ and triangle inequality, we have
    \begin{align*}
       & \left\vert   
       \int_0^{t+\Delta t} \ea d\mu\; d~\tau- 
       \int_0^t \ea d\mu\; d~\tau\right\vert \\
        & = \left\vert \displaystyle\int_{\bigcup_{\tau = 0}^{t+\Delta t} \tx(\tau)} \ell(\hat{f}(\weight,\psi))d\mu- \displaystyle\int_{\bigcup_{\tau = 0}^{t} \tx(\tau)} \ell(\hat{f}(\weight,\psi))d\mu\right\vert \\
        & = \Big\vert \displaystyle\int_{\bigcup_{\tau = 0}^{t+\Delta t} \tx(\tau)\setminus \bigcup_{\tau = 0}^{t} \tx(\tau)} \ell(\hat{f}(\weight,\psi))d\mu+ \displaystyle\int_{\bigcup_{\tau = 0}^{t} \tx(\tau)\cap\bigcup_{\tau = 0}^{t+\Delta t} \tx(\tau)} \ell(\hat{f}(\weight,\psi))d\mu\\
        & -\displaystyle\int_{\bigcup_{\tau = 0}^{t} \tx(\tau)\setminus \bigcup_{\tau = 0}^{t+\Delta t} \tx(\tau)} \ell(\hat{f}(\weight,\psi))d\mu- \displaystyle\int_{\bigcup_{\tau = 0}^{t} \tx(\tau)\cap \bigcup_{\tau = 0}^{t+\Delta t} \tx(\tau)} \ell(\hat{f}(\weight,\psi))d\mu\Big\vert\\
        & = \Big\vert \displaystyle\int_{\bigcup_{\tau = 0}^{t+\Delta t} \tx(\tau)\setminus \bigcup_{\tau = 0}^{t} \tx(\tau)} \ell(\hat{f}(\weight,\psi))d\mu-\displaystyle\int_{\bigcup_{\tau = 0}^{t} \tx(\tau)\setminus \bigcup_{\tau = 0}^{t+\Delta t} \tx(\tau)} \ell(\hat{f}(\weight,\psi))d\mu\Big\vert\\
        & \leq \displaystyle\int_{\bigcup_{\tau = 0}^{t+\Delta t} \tx(\tau)\setminus \bigcup_{\tau = 0}^{t} \tx(\tau)} \abs{\ell(\hat{f}(\weight,\psi))}d\mu+\displaystyle\int_{\bigcup_{\tau = 0}^{t} \tx(\tau)\setminus \bigcup_{\tau = 0}^{t+\Delta t} \tx(\tau)} \abs{\ell(\hat{f}(\weight,\psi))}d\mu \\.
    \end{align*}
    Then, by the boundedness of $\ell,$
    \begin{align*}
        \displaystyle\int_{\bigcup_{\tau = 0}^{t+\Delta t} \tx(\tau)\setminus \bigcup_{\tau = 0}^{t} \tx(\tau)} \abs{\ell(\hat{f}(\weight,\psi))}d\mu+\displaystyle\int_{\bigcup_{\tau = 0}^{t} \tx(\tau)\setminus \bigcup_{\tau = 0}^{t+\Delta t} \tx(\tau)} \abs{\ell(\hat{f}(\weight,\psi))}d\mu
        & \leq M_0 \mu\left(\bigcup_{\tau = 0}^{t+\Delta t} \tx(\tau)\setminus \bigcup_{\tau = 0}^{t} \tx(\tau)\right)\\
        & \hspace{10mm }+ M_0\mu\left(\bigcup_{\tau = 0}^{t} \tx(\tau)\setminus \bigcup_{\tau = 0}^{t+\Delta t} \tx(\tau)\right).
    \end{align*}
    Hence,
    \begin{align*}
        M_0 \mu\left(\bigcup_{\tau = 0}^{t+\Delta t} \tx(\tau)\setminus \bigcup_{\tau = 0}^{t} \tx(\tau)\right) + M_0\mu\left(\bigcup_{\tau = 0}^{t} \tx(\tau)\setminus \bigcup_{\tau = 0}^{t+\Delta t} \tx(\tau)\right) 
        & = M_0\mu\left(\bigcup_{\tau = 0}^{t+\Delta t} \tx(\tau)\triangle \bigcup_{\tau = 0}^{t} \tx(\tau)\right)\\
        & < M_0\delta\\
        & = M_0 \frac{\varepsilon}{M_0}\\
        & = \varepsilon,
    \end{align*}
    as desired with $\Delta t = 1.$
\end{proof}

\begin{remark}\label{boundedremark}
    It is reasonable to assume that there is some constant $M_0<\infty$ for which $M_i\leq M_0$ 
    for all $i\in \mathbb{N}$ because 
    if one did not exist, then the problem would be unsolvable since
    an unbounded loss function would imply that 
    the neural network is not learning anything useful from the data. 
\end{remark}

\begin{remark}\label{unionremark}
    The  equality  $\int_0^t \ea \;  d~\tau  = \int_{ \bigcup_{\tau=0}^t\tx(\tau)}       \ell(\fhat(\weight(t),\psi(t))( \vx ) )d\mu$ is not always guaranteed as duplicates among all $\tx(\tau), \forall \tau$  are counted only once. However, we assume that within the construction of $\ell,$ there is an average between all duplications, and only one instance of these duplicates survives. This is reasonable because, in practice, a batch of data is uniformly sampled and averaged over.
\end{remark}

Lemma~\ref{lem: simtaskssimloss} solidifies the notion that similar 
tasks produce similar loss values and that $\clcost$ is continuous with respect 
to the measure $\mu$ for any consecutive tasks $ \tx(t), \tx(t+1) \in \mathcal{B}(\overline{\tx}).$  Moreover, the  lemma provides the conditions for when a solution between two consecutive tasks  exists. If  Lemma~\ref{lem: simtaskssimloss} is upheld for every $t \in [0,T],$ then the CL problem will have a solution for the whole interval $[0,T].$ In the following theorem we prove that there exists a solution for the cumulative learning problem. 

\begin{remark}
    We can view Theorem \ref{thm:contMeasure} result as showing that the function $J$ satisfies a notion of $\textit{absolute continuity},$ as we see the condition relies on a partition of the interval $[0,T].$
\end{remark}

\begin{theorem}\label{thm:contMeasure}
    Let $(\overline{\tx},\mathcal{B}(\overline{\tx}), \mu)$ be the measure space as defined  in Definition~\ref{defintersect}. For every $\varepsilon>0$ there exists $\delta >0$ such that if
    \begin{align*}
        \sum_{\gamma = t}^{T-1} \mu(\bigcup_{\tau = 0}^{\gamma+1} \tx(\tau)\triangle \bigcup_{\tau = 0}^{\gamma} \tx(\tau))<\delta,
    \end{align*}
    then
    \begin{align*}
        \sum_{\gamma = t}^{T-1} \abs{J(\weight(\gamma+1), \psi(\gamma+1), \tx(\gamma+1))-J(\weight(\gamma), \psi(\gamma), \tx(\gamma))} & <\varepsilon.
    \end{align*}
\end{theorem}

\begin{proof}
    Let $\varepsilon>0$ and choose $\delta = \varepsilon/M_0.$ Notice,
    \begin{align*}
        \sum_{\gamma = t}^{T-1} \abs{J(\weight(\gamma+1), \psi(\gamma+1), \tx(\gamma+1))-J(\weight(\gamma), \psi(\gamma), \tx(\gamma))} & = \sum_{\gamma = t}^{T-1} \Big\vert   
       \int_0^{\gamma + 1} \Big(\int_{\vx\in \tx(\tau)} \ell(\fhat(\weight(\gamma),\psi(\gamma))(\vx))d\mu\Big) d~\tau\\
       & - 
       \int_0^\gamma \Big(\int_{\vx\in \tx(\tau)} \ell(\fhat(\weight(\gamma),\psi(\gamma))(\vx))d\mu\Big) \d~\tau\Big\vert\\
       & =  \sum_{\gamma = t}^{T-1} \left\vert \displaystyle\int_{\bigcup_{\tau = 0}^{\gamma + 1} \tx(\tau)} \ell(\hat{f}(\weight,\psi))d\mu- \displaystyle\int_{\bigcup_{\tau = 0}^{\gamma} \tx(\tau)} \ell(\hat{f}(\weight,\psi))d\mu\right\vert \\
       & = \sum_{\gamma = t}^{T-1} \Big\vert \displaystyle\int_{\bigcup_{\tau = 0}^{\gamma+1} \tx(\tau)\setminus \bigcup_{\tau = 0}^{\gamma} \tx(\tau)} \ell(\hat{f}(\weight,\psi))d\mu\\
       & + \displaystyle\int_{\bigcup_{\tau = 0}^{\gamma} \tx(\tau)\cap\bigcup_{\tau = 0}^{\gamma+1} \tx(\tau)} \ell(\hat{f}(\weight,\psi))d\mu\\
        & -\displaystyle\int_{\bigcup_{\tau = 0}^{\gamma} \tx(\tau)\setminus \bigcup_{\tau = 0}^{\gamma + 1} \tx(\tau)} \ell(\hat{f}(\weight,\psi))d\mu\\
        & - \displaystyle\int_{\bigcup_{\tau = 0}^{\gamma} \tx(\tau)\cap \bigcup_{\tau = 0}^{\gamma +1} \tx(\tau)} \ell(\hat{f}(\weight,\psi))d\mu\Big\vert\\
        & \sum_{\gamma = t}^{T-1} \Big\vert \displaystyle\int_{\bigcup_{\tau = 0}^{\gamma + 1} \tx(\tau)\setminus \bigcup_{\tau = 0}^{\gamma} \tx(\tau)} \ell(\hat{f}(\weight,\psi))d\mu\\
        & -\displaystyle\int_{\bigcup_{\tau = 0}^{\gamma} \tx(\tau)\setminus \bigcup_{\tau = 0}^{\gamma +1} \tx(\tau)} \ell(\hat{f}(\weight,\psi))d\mu\Big\vert\\
        & \leq \sum_{\gamma = t}^{T-1} \Big(\displaystyle\int_{\bigcup_{\tau = 0}^{\gamma + 1} \tx(\tau)\setminus \bigcup_{\tau = 0}^{\gamma} \tx(\tau)} \abs{\ell(\hat{f}(\weight,\psi))}d\mu\\
        & +\displaystyle\int_{\bigcup_{\tau = 0}^{\gamma} \tx(\tau)\setminus \bigcup_{\tau = 0}^{\gamma + 1} \tx(\tau)} \abs{\ell(\hat{f}(\weight,\psi))}d\mu\Big).
    \end{align*}
    Then, by the boundedness of $\ell,$
    \begin{align*}
    \sum_{\gamma = t}^{T-1} \displaystyle\int_{\bigcup_{\tau = 0}^{\gamma+1} \tx(\tau)\setminus \bigcup_{\tau = 0}^{\gamma} \tx(\tau)} \abs{\ell(\hat{f}(\weight,\psi))}d\mu+\displaystyle\int_{\bigcup_{\tau = 0}^{\gamma} \tx(\tau)\setminus \bigcup_{\tau = 0}^{\gamma+1} \tx(\tau)} \abs{\ell(\hat{f}(\weight,\psi))}d\mu
        & \leq \sum_{\gamma = t}^{T-1}\Big[M_0 \mu\left(\bigcup_{\tau = 0}^{\gamma+1} \tx(\tau)\setminus \bigcup_{\tau = 0}^{\gamma} \tx(\tau)\right)\\
        & \hspace{10mm }+ M_0\mu\left(\bigcup_{\tau = 0}^{\gamma} \tx(\tau)\setminus \bigcup_{\tau = 0}^{\gamma+1} \tx(\tau)\right)\Big].
    \end{align*}
    Hence,
    \begin{align*}
        \sum_{\gamma = t}^{T-1} M_0 \mu\left(\bigcup_{\tau = 0}^{\gamma+1} \tx(\tau)\setminus \bigcup_{\tau = 0}^{\gamma} \tx(\tau)\right) + M_0\mu\left(\bigcup_{\tau = 0}^{t} \tx(\tau)\setminus \bigcup_{\tau = 0}^{\gamma+1} \tx(\tau)\right) 
        & = M_0\sum_{\gamma = t}^{T-1} \mu\left(\bigcup_{\tau = 0}^{\gamma+1} \tx(\tau)\triangle \bigcup_{\tau = 0}^{\gamma} \tx(\tau)\right)\\
        & < M_0\delta\\
        & = M_0 \frac{\varepsilon}{M_0}\\
        & = \varepsilon,
    \end{align*}
    as desired.
\end{proof}

Therefore, the existence of a solution to \eqref{eq:clcum} depends on whether we can ensure a notion of absolute continuity of $\int_0^t \ea \;  d~\tau$ or not. In other words, if we can guarantee similar tasks, cumulatively speaking.

\subsection{Reachability to a Solution in the Presence of Dissimilar Tasks}
Theorem~\ref{thm:contMeasure} shows the condition of existence of the CL solution for \eqref{eq:clcum}. 
The contrapositive reveals that dissimilar loss values imply dissimilar tasks.  For an arbitrary CL problem, however,
dissimilar tasks may or may not result in dissimilar loss values. 
That is, if we have two tasks $\tx(t)$ and $\tx(t+\Delta t)$ 
such that $\mu \left( \tx(t)\bigtriangleup \tx(t+\Delta t) \right) \geq \delta$, 
 we do not know whether $\abs{\clcost-\clcostt} \leq \varepsilon$ 
(i.e., similar loss values) or $\abs{\clcost-\clcostt} \geq \varepsilon$ 
(i.e., dissimilar loss values).
Nonetheless, if dissimilar tasks do produce dissimilar loss values and the number of tasks increases, the value of $\abs{\clcost-\clcostt}$ will tend to increase  and eventually exceed $\varepsilon.$
Therefore, the model will gradually forget prior tasks, and eventually CL will fail, thus violating Theorem~\ref{thm:contMeasure}. 

In order to counteract this situation, the key condition for absolute continuity must be upheld; that is, for a series of tasks, the upper bound on 
$ \sum_{t}^T \left[ \left\vert J(\weight(\tau+\Delta \tau),\psi(\tau+\Delta \tau),\tx(\tau+\Delta \tau)) - \clcost \right\vert \right]$ must be guaranteed to be less than or equal to $\varepsilon$. The key here is to identify the knobs in the model that can be tuned to ensure that $\varepsilon$ is small even in the presence of dissimilar tasks. In the typical neural network learning problem, the two knobs that can be tuned are the weights and architecture of the network. We will cement these ideas in the following and begin by elucidating the dependency of $\ea$ on the weights and the architecture.

\begin{lemma}\label{lem:taylor}
    Suppose $\tx(t)$ and $\tx(t+\Delta T)$ are two consecutive tasks in the measure space $(\overline{\tx}, \mathcal{B}(\overline{\tx}),\mu).$  Set $M_0$ to be the upper bound on the loss function $\ell$ across all tasks. Then for $\mu(\tx(t)\bigtriangleup \tx(t+\Delta t))\geq \delta,$ the following holds:
    \begin{align}
        E\left(\bigcup_{\tau = 0}^{t+1}\tx(\tau)\right) & = E\left(\bigcup_{\tau = 0}^{t}\tx(\tau)\right) + \Delta t\Big[\mu\left( \bigcup_{\tau = 0}^{t}\tx(\tau)\Delta \bigcup_{\tau = 0}^{t+1}\tx(\tau)\right)\cdot \displaystyle\int_{\bigcup_{\tau = 0}^{t}\tx(\tau)\Delta \bigcup_{\tau = 0}^{t+1}\tx(\tau)} \ell(\hat{f}(\weight,\psi))d\mu \nonumber \\
        & + \displaystyle\int_{\bigcup_{\tau = 0}^{t}\tx(\tau)} \ell'(\hat{f}(\weight,\psi))\cdot \partial_{\weight}^1 \hat{f}(\weight,\psi)\cdot \Delta w \hspace{1mm} d\mu \nonumber \\
        & +\displaystyle\int_{\bigcup_{\tau = 0}^{t}\tx(\tau)} \ell'(\hat{f}(\weight,\psi))\cdot \partial_\psi^1 \hat{f}(\weight,\psi)\cdot \Delta \psi \hspace{1mm}  d\mu\Big]+ o(\Delta t).\label{archfinaltaylor}
    \end{align}
\end{lemma}
\begin{proof}
    Note that the first-order Taylor series expansion of $E(\tx(t+\Delta t))$ about  $t$ is given by
    \begin{align}
        E\left(\bigcup_{\tau = 0}^{t+1}\tx(\tau)\right) & = E\left(\bigcup_{\tau = 0}^{t}\tx(\tau)\right) + \Delta t\Big[ E_{\tx} \left(\bigcup_{\tau = 0}^{t+1}\tx(\tau)\Delta \bigcup_{\tau = 0}^{t}\tx(\tau)\right) \nonumber \\
        & + E_{\weight}\left(\bigcup_{\tau = 0}^{t}\tx(\tau)\right) + E_\psi\left(\bigcup_{\tau = 0}^{t+1}\tx(\tau)\right)\Big] + o(\Delta t),\label{taylorexparch}
    \end{align}
    where $E_{\tx} , E_{\weight},$ and $E_\psi$ are the partial derivatives of the expected value function for the data, weights, and architecture, respectively. Now, we work on acquiring each partial derivative. Toward that end, 
    \begin{align}
        E_{\tx} \left(\bigcup_{\tau = 0}^{t+1}\tx(\tau)\Delta \bigcup_{\tau = 0}^{t}\tx(\tau)\right) & =\mu\left( \left(\bigcup_{\tau = 0}^{t+1}\tx(\tau)\Delta \bigcup_{\tau = 0}^{t}\tx(\tau)\right)\right)\cdot \displaystyle\int_{\left(\bigcup_{\tau = 0}^{t+1}\tx(\tau)\Delta \bigcup_{\tau = 0}^{t}\tx(\tau)\right)} \ell(\hat{f}(\weight,\psi))d\mu.\label{E_X1}
    \end{align}
    To determine the remaining two derivatives, we use the Sobolev function chain rule found in \cite{evans2022partial}. We can do so because $\ell$ is real-valued and bounded and $\ell'$ is continuously differentiable. Then,
    \begin{align}
        E_{\weight}\left(\bigcup_{\tau = 0}^{t}\tx(\tau)\right) & = \displaystyle\int_{\bigcup_{\tau = 0}^{t}\tx(\tau)} \ell'(\hat{f}(\weight,\psi))\cdot \partial_{\weight}^1 \hat{f}(\weight,\psi)\cdot \Delta w \hspace{1mm}  d\mu.\label{E_w},\\
        E_{\psi}\left(\bigcup_{\tau = 0}^{t}\tx(\tau)\right) & = \displaystyle\int_{\bigcup_{\tau = 0}^{t}\tx(\tau)} \ell'(\hat{f}(\weight,\psi))\cdot \partial_\psi^1 \hat{f}(\weight,\psi)\cdot \Delta \psi \hspace{1mm}  d\mu.\label{E_psi}
    \end{align}
    Substituting (\ref{E_X1}), (\ref{E_w}), and (\ref{E_psi}) into the Taylor series expansion (\ref{taylorexparch}), we have
    \begin{align*}
        E\left(\bigcup_{\tau = 0}^{t+1}\tx(\tau)\right) & = E\left(\bigcup_{\tau = 0}^{t}\tx(\tau)\right) + \Delta t\Big[\mu\left( \bigcup_{\tau = 0}^{t}\tx(\tau)\Delta \bigcup_{\tau = 0}^{t+1}\tx(\tau)\right)\cdot \displaystyle\int_{\bigcup_{\tau = 0}^{t}\tx(\tau)\Delta \bigcup_{\tau = 0}^{t+1}\tx(\tau)} \ell(\hat{f}(\weight,\psi))d\mu\\
        & + \displaystyle\int_{\bigcup_{\tau = 0}^{t}\tx(\tau)} \ell'(\hat{f}(\weight,\psi))\cdot \partial_{\weight}^1 \hat{f}(\weight,\psi)\cdot \Delta w \hspace{1mm} d\mu \\ &+\displaystyle\int_{\bigcup_{\tau = 0}^{t}\tx(\tau)} \ell'(\hat{f}(\weight,\psi))\cdot \partial_\psi^1 \hat{f}(\weight,\psi)\cdot \Delta \psi \hspace{1mm}  d\mu\Big] \\ &+ o(\Delta t).
    \end{align*}
\end{proof}

Colloquially, in the CL problem only the weights are updated, but in this paper we argue that this is not enough. We show that  with just the weights being updated, the $\clcost$ eventually violates Theorem~\ref{thm:contMeasure}. To prove this, we first derive a lower bound on $\ea.$ For the remaining analysis, we will set  $ \Delta t = 1.$

\begin{lemma}\label{lem:upper_weight}
    Suppose $\tx(t)$ and $\tx(t+\Delta t)$ are two consecutive tasks in the measure space $(\overline{\tx}, \mathcal{B}(\overline{\tx}),\mu).$ Let $L_1>0$ and $0\leq \delta\leq 1$  be constants. Set $M_0$ to be the upper bound on the loss function $\ell$ across all tasks, and assume $\mu\left( \bigcup_{\tau = 0}^{t}\tx(\tau)\Delta \bigcup_{\tau = 0}^{t+\Delta t}\tx(\tau)\right)\geq \delta$ for all $t\in [0,T].$ Then for $\Delta t = 1$ the following holds:
    \begin{align*}
        \sum_{\tau = t}^{T} \left(J(\tx(\tau))-J(\tx(\tau+1)\right)
        & \leq \sum_{\tau =t}^{T} \Big( M_0\cdot \delta - \displaystyle\int_{\bigcup_{\nu = 0}^\tau\tx(\nu)} \left(\ell'(\hat{f}(\weight,\psi))\cdot \partial_{\weight}^1 \hat{f}(\weight,\psi)\right)^2 \hspace{1mm}  d\mu \Big)
    \end{align*}
    for $\displaystyle\int_{\tx(t)} \ell'(\hat{f}(\weight,\psi))\cdot \partial_\psi^1 \hat{f}(\weight,\psi)\cdot \Delta \psi \; \;  d\mu+ o(\Delta t) = 0$ as $\psi(t) = \psi, \forall t.$
\end{lemma}
\begin{proof}
    The proof is in Section~\ref{sec:upper_weight}.
\end{proof}
A bound similar to the one in Lemma~\ref{lem:upper_weight} can be attained for the architecture terms as well. 
\begin{lemma}\label{lem:upper_arch}
    Suppose $\tx(t)$ and $\tx(t+\Delta t)$ are two consecutive tasks in the measure space $(\overline{\tx}, \mathcal{B}(\overline{\tx}),\mu).$ Let $L_1>0$ and $0\leq \delta\leq 1$  be constants. Set $M_0$ to be the upper bound on the loss function $\ell$ across all tasks, and assume $\mu\left( \bigcup_{\tau = 0}^{t}\tx(\tau)\Delta \bigcup_{\tau = 0}^{t+\Delta t}\tx(\tau)\right)\geq \delta$ for all $t\in [0,T].$ Then for $\Delta t = 1$ the following holds:
    \begin{align*}
        \sum_{\tau = t}^{T} \left(J(\tx(\tau))-J(\tx(\tau+1)\right)
        & \leq \sum_{\tau =t}^{T} \Big( M_0\cdot \delta - \displaystyle\int_{\bigcup_{\nu = 0}^\tau\tx(\nu)} \left(\ell'(\hat{f}(\weight,\psi))\cdot \partial_{\weight}^1 \hat{f}(\weight,\psi)\right)^2 \hspace{1mm}  d\mu\\
        & -\displaystyle\int_{\bigcup_{\nu = 0}^\tau\tx(\nu)} \left(\ell'(\hat{f}(\weight,\psi))\cdot \partial_\psi^1 \hat{f}(\weight,\psi)\right)^2\hspace{1mm}  d\mu\Big).
    \end{align*}
\end{lemma}
\begin{proof}
    The proof is  in Section~\ref{sec:upper_arch}.
\end{proof}

The following result tells us that in the case of dissimilar tasks and with dissimilar loss values ($\sum_{\tau = t}^{T} \left(J(\tx(\tau)) - J(\tx(\tau+1))\right)\geq \varepsilon)$ for a fixed architecture across tasks, we can learn the architecture at each task such that we can acquire similar loss values tasks ($\sum_{\tau = t}^{T} \left(J(\tx(\tau)) - J(\tx(\tau+1))\right) < \varepsilon).$ 
\begin{theorem}\label{thm:intersection}
    Suppose $\tx(t)$ and $\tx(t+\Delta t)$ are two consecutive tasks in the measure space $(\overline{\tx}, \mathcal{B}(\overline{\tx}),\mu).$ Set $M_0$ to be the upper bound on the loss function $\ell$ across all tasks. Moreover, let $\varepsilon>0$ such that
    $\sum_{\tau = t}^{T} \left(J(\tx(\tau)) - J(\tx(\tau+1))\right) \geq \varepsilon$ and assume $\mu\left( \bigcup_{\tau = 0}^{t}\tx(\tau)\Delta \bigcup_{\tau = 0}^{t+1}\tx(\tau)\right)\geq \delta$ for all $t\in [0,T].$ If for a fixed architecture $\psi^*$ for all $t\in [0,T]$
    \begin{align*}
        \sum_{\tau = t}^{T} \left(J(\weight(\tau), \psi^*, \tx(\tau)) - J(\weight(\tau+1), \psi^*, \tx(\tau+1))\right) & \geq \varepsilon,
    \end{align*}
    then we can choose a new architecture $\psi(t)$ for $[0,T]$ such that we can have
    \begin{align*}
        \sum_{\tau = t}^{T} \left(J(\weight(\tau), \psi^*, \tx(\tau)) - J(\weight(\tau+1), \psi(\tau), \tx(\tau+1))\right) & < \varepsilon.
    \end{align*}
\end{theorem}

\begin{proof}
    Let $\varepsilon >0$, and choose $\delta$ $J$ to be absolutely continuous. As
    \begin{align*}
        \displaystyle\int_{\bigcup_{\nu = 0}^\tau\tx(\nu)} \left(\ell'(\hat{f}(\weight,\psi))\cdot \partial_{\psi}^1 \hat{f}(\weight,\psi)\right)^2 & >0,
    \end{align*}
    notice that
    \begin{align*}
         \sum_{\tau = t}^{T} \left(J(\tx(\tau)) - J(\tx(\tau+1))\right)
        & \leq \sum_{\tau =t}^{T} \Big( M_0\cdot \delta - \displaystyle\int_{\bigcup_{\nu = 0}^\tau\tx(\nu)} \left(\ell'(\hat{f}(\weight,\psi))\cdot \partial_{\weight}^1 \hat{f}(\weight,\psi)\right)^2 \hspace{1mm}  d\mu\\
        & -\displaystyle\int_{\bigcup_{\nu = 0}^\tau\tx(\nu)} \left(\ell'(\hat{f}(\weight,\psi))\cdot \partial_{\psi}^1 \hat{f}(\weight,\psi)\right)^2 \hspace{1mm}  d\mu \Big)\\
        & \leq \sum_{\tau =t}^{T} \Big( M_0\cdot \delta - \displaystyle\int_{\bigcup_{\nu = 0}^\tau\tx(\nu)} \left(\ell'(\hat{f}(\weight,\psi))\cdot \partial_{\weight}^1 \hat{f}(\weight,\psi)\right)^2 \hspace{1mm}  d\mu \Big).
    \end{align*}
    Thus, by changing the architecture we may be able to achieve
    \begin{align*}
         \sum_{\tau = t}^{T} \left(J(\tx(\tau)) - J(\tx(\tau+1))\right)
        & \leq \sum_{\tau =t}^{T} \Big( M_0\cdot \delta - \displaystyle\int_{\bigcup_{\nu = 0}^\tau\tx(\nu)} \left(\ell'(\hat{f}(\weight,\psi))\cdot \partial_{\weight}^1 \hat{f}(\weight,\psi)\right)^2 \hspace{1mm}  d\mu\\
        & -\displaystyle\int_{\bigcup_{\nu = 0}^\tau\tx(\nu)} \left(\ell'(\hat{f}(\weight,\psi))\cdot \partial_{\psi}^1 \hat{f}(\weight,\psi)\right)^2 \hspace{1mm}  d\mu \Big)\\
        & <\varepsilon\\
        & \leq \sum_{\tau =t}^{T} \Big( M_0\cdot \delta - \displaystyle\int_{\bigcup_{\nu = 0}^\tau\tx(\nu)} \left(\ell'(\hat{f}(\weight,\psi))\cdot \partial_{\weight}^1 \hat{f}(\weight,\psi)\right)^2 \hspace{1mm}  d\mu \Big).
        \end{align*}
    Thus, we have reduced the difference in loss values below the $\varepsilon$ threshold to gain similar loss values when we began with dissimilar tasks.
\end{proof}
The implications of Theorem~\ref{thm:intersection} indicate that it is possible to shrink the effect of the new task on the forgetting cost $J$ by changing the architecture of the network by pushing $\varepsilon$ closer to zero. The closer $\varepsilon$ is to zero, the more the cardinality of the intersection space and the more the task sets are connected. Therefore, upholding the absolute continuity of $J$ is feasible. 

It is now of interest to identify how to change the architecture and efficiently learn over a series of tasks. In particular, we must examine to what extent we can control the effect of architecture change on the CL problem.

\begin{figure}
    \centering
    \includegraphics[width=\linewidth]{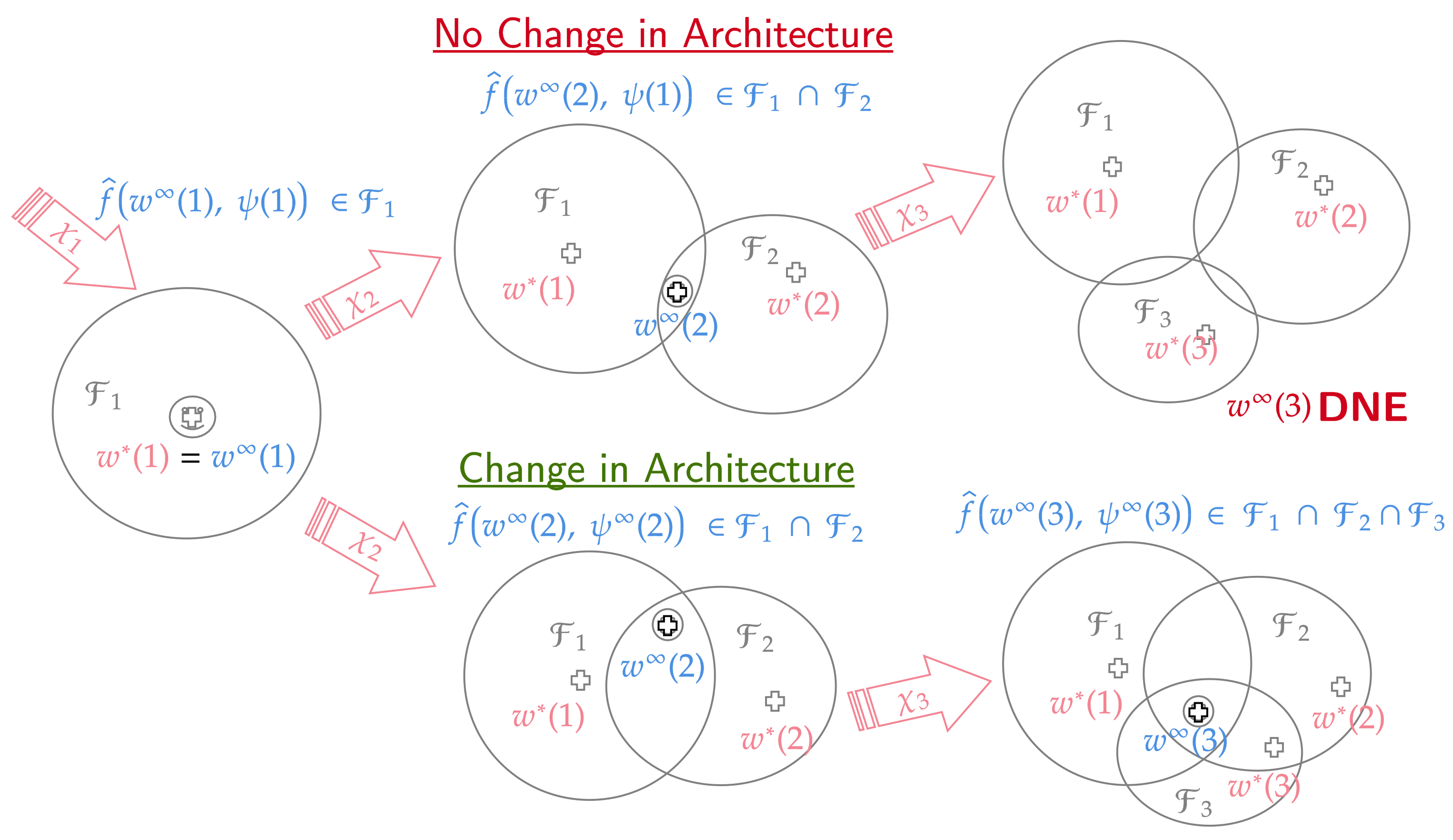}
    \caption{CL solution, where we change the size of the intersection space by introducing more capacity, through choosing novel hyperparameters}
    \label{fig:solution}
\end{figure}
\section{CL Solution}
To understand how to change the architecture of the model and learn over a series of tasks within the CL paradigm, we consider Figure~\ref{fig:solution}. We begin at task $t = 1$ and determine an optimal weight denoted $w^*(1) = w^\infty(1).$ This gives us a neural network solution in $\mathcal{F}_1.$ Then, as the next task is observed, we perform a derivative-free hyperparameter search, in particular, the directional direct search method~\cite{larson2019derivative}, to obtain a new architecture. The solution to this search is denoted $\psi^\infty(2)$ as in  Figure ~\ref{fig:solution}. Notably,  a new quandary arises since the new architecture introduces a parameter space that may have a different size from the one used for the previous task. Unfortunately, it is not feasible to trivially transfer information from the previous architecture to the new one. The common and state-of-the-art solution to this problem is to randomly re-initialize the parameters in the new architecture from retrain on all available tasks, which is resource heavy as well as impractical. Instead, we develop a low-rank transfer algorithm that seeks to transfer information between the previous architecture to the new architecture in an efficient way. Mathematically, our goal is to solve the following bi-level problem:
\begin{align}\label{eq:bilevel} \tag{bi-level}
     V^*(t,\weight(t))  = \min_{\weight \in \mathcal{W}(\psi^*(t))} V(t,\weight(t)),\quad \text{ where}\; \;  \psi^*(t)  = \mathrm{arg}\min_{\psi\in \Psi} J(\weight(t),\psi,\tx(t)).
\end{align}
Noting that the weight search space at each $t$ is a function of the optimal architecture, we change the weights from $t\rightarrow t+\Delta~t$ as 
\begin{align}
    \weight(t+\Delta~t) = \mathrm{A}^*(t) \weight(t) \mathrm{B}(t)^{*T},
\end{align}
where $\mathrm{A}$ and $\mathrm{B}$ are matrices that enable the transfer of information between two different parameter spaces and where $\mathrm{A}^*, \mathrm{B}^*$ are the optimal values such that loss of information is minimal. This change in the weight space for each new task introduces dynamics into the continual learning problem defined by $V^*.$ The following result therefore provides the total variation~(the dynamics) in $V^*$ as a function of tasks.
\begin{proposition}\label{HJB}
    The total variation in $V^*(t)$ at any given task $t$ is given by
    \begin{align}\label{eq:HJB}
        -\partials_t V^*(t)& = \min_{\weight \in \mathcal{W}(\psi^*(t))} J(\weight(t),\psi^*(t),\tx(t)) + \partials_{\tx}d_{t} \tx+ \partials_{\weight} V^*(t) \left[\mathrm{A}^*(t)\weight(t)( \mathrm{B}^*(t))^T + u(t)\right],
    \end{align}
    where $\mathrm{A}^*(t),B^*(t)$ are optimal $\mathrm{A}(t)$ and $\mathrm{B}(t)$ for task $t$ and where $u(t)$ represents the updates made to the each weights matrix of the new dimensions. 
\end{proposition}
\begin{proof}
    The proof is in Section~\ref{sec:HJB}.
\end{proof}

Observe that \eqref{eq:HJB} is missing the term $\frac{\partials V^*}{\partials \psi}\frac{\partials \psi}{\partials t}.$ This omission is intentional because the exact effects of architecture cannot be measured through derivatives in this partial differential equation (PDE). Instead, in this PDE the effect is absorbed through the change in tensors $\mathrm{A}(t)$ and $\mathrm{B}(t)$.
Because of this absorption we can now define a lower bound on $\partials_{t} V^*(t).$ This lower bound, which is essentially a lower bound on the variation, is equivalent to the lower bound on $\varepsilon$ in Theorem~\ref{thm:intersection} with the distinction that the lower bound is derived on $V^*$ instead of just the forgetting cost $J.$ The idea covers the effects of the dynamic program. We derive this bound below.
\begin{theorem}\label{thm:lower_HJB}
Given proposition~\ref{HJB}, assume there exists a stochastic-gradient-based optimization procedure,  and choose $ u(t) =- \sum^{I} \alpha(i) \mathrm{g}^{(i)},$ where $I$ is the number of updates and $\alpha(i)$ is some learning rate and $\mathrm{g}^{(i)}$ is the gradient. Choose $\alpha(i)$ such that $\norm{\mathrm{g}_{\mathrm{MIN} }}  \norm{\sum^{I} \alpha(i) } \rightarrow 0$ as $I \rightarrow \infty$ and $\min_{\weight \in \mathcal{W}(\psi^*(t))} J(\weight(t),\psi^*(t),\tx(t)) \leq \varepsilon.$ Then, as long as the architecture is chosen such that $   \sing \delta = \left[ \singmin \norm{\mathrm{g}_{\mathrm{MIN} }}  \norm{\sum^{I} \alpha(i) } - \singw \norm{\mathrm{A}^*(t)\weight(t)( \mathrm{B}^*(t))^T} \right],$  the total variance in $V^*$ is bounded by $\varepsilon.$ 
\end{theorem}
\begin{proof}
    The proof is in Section~\ref{sec:lower_HJB}.
\end{proof}
Therefore, even when $\alpha(i)$ is chosen such that $\norm{\mathrm{g}_{\mathrm{MIN} }}  \norm{\sum^{I} \alpha(i) } $ as $I \rightarrow \infty,$ $\sing \delta$ can be countered by 
$\singw \norm{\mathrm{A}^*(t)\weight(t)( \mathrm{B}^*(t))^T}$, and this bound is tunable because of the change in the architecture. Therefore, as long as $\singw \norm{\mathrm{A}^*(t)\weight(t)( \mathrm{B}^*(t))^T} = \sing \delta$, the total variance is bounded for all $t$. With the presence of bounded variance across all $t$, the sum of these variances is bounded. Therefore, it is clear that the $V^*$ is absolutely continuous as long as $\varepsilon$ is small for all $\delta >0.$
\begin{remark}
    Theorem~\ref{thm:lower_HJB} is crucial because it shows that, in the presence of a perfect stochastic gradient algorithm, the change in the tasks can be countered by the choice of the architecture. By choosing a new architecture and introducing the $A$ and $B$ matrices such that within the CL problem the effect of varying data distribution can be mitigated,  the CL problem can be acceptably solved.
\end{remark}

Given this guarantee, the rest of this section is dedicated to constructing an algorithm such that the condition to ensure this bounded variance is satisfied. In summary~(refer to Figure~\ref{fig:method}), for every new task, we utilize an off-the-shelf algorithm to search for a new optimal architecture. Once a new architecture is found,  we complete a low-rank transfer from the previous task's optimal weights to the new task parameter space. Then, using an off-the-shelf continual learning algorithm, we train the new architecture for the new task while balancing memory on all the prior tasks.  In the process of low-rank transfer, we guarantee transfer of learning by requiring that the $\partials_{t} \ell \leq \eta$, where $\eta$ is the acceptable tolerance on the loss value. Such tolerance allows the intersection of the neural network search spaces between subsequent tasks to increase such that the loss is small. Thus, our algorithm comprises three components: step 1 --  architecture search, step 2 --  low-rank transfer, and step 3 --  continual learning. While step 3 is directly adapted from \cite{chakraborty2025understanding, raghavan2021formalizing}, steps 1 and 2 are novel contributions for this work.


\begin{figure}
    \centering
    \includegraphics[width=1.1\linewidth]{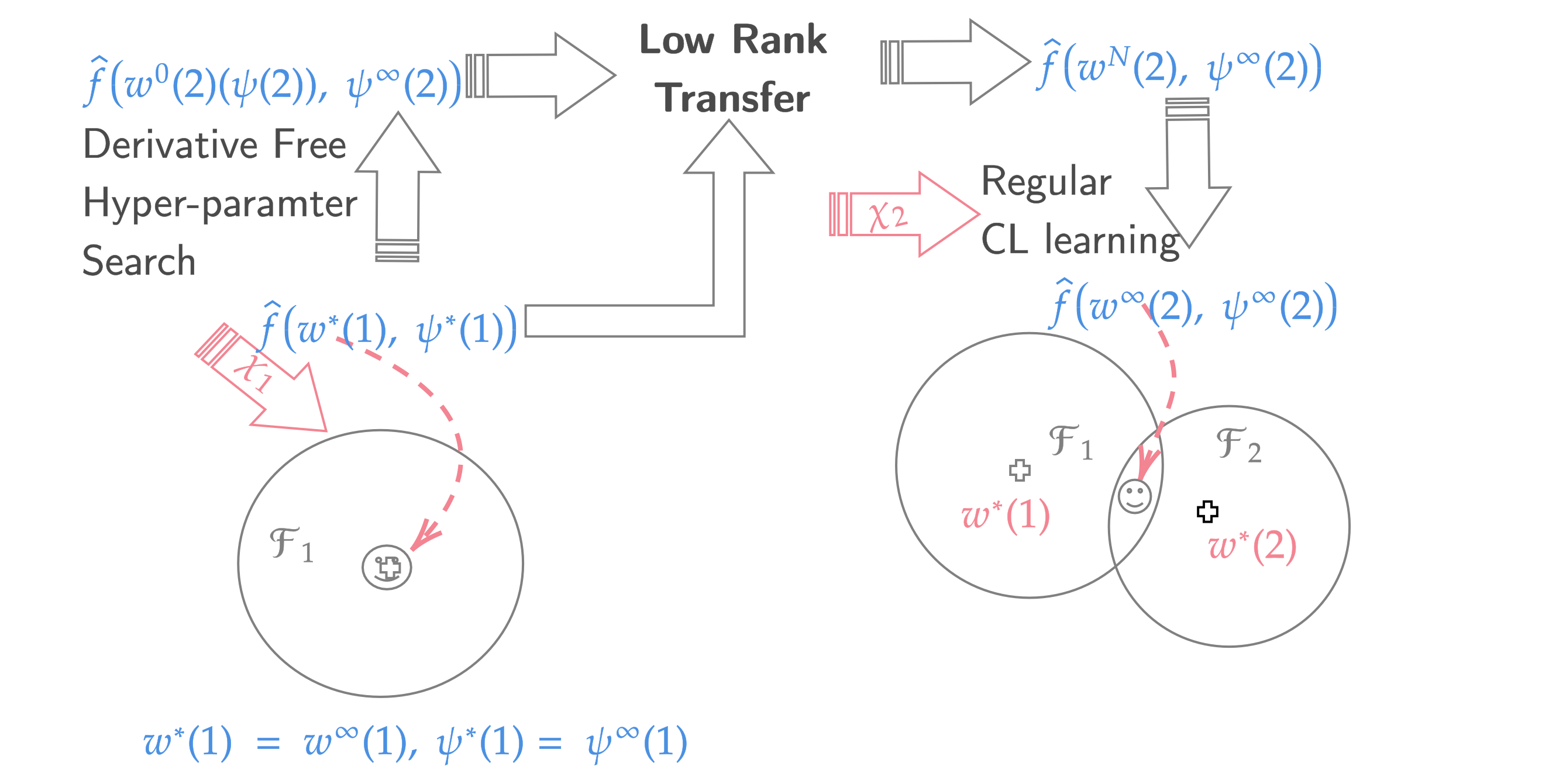}
    \caption{How we do this?}
    \label{fig:method}
\end{figure}

\subsection{\underline{Step 1:}  Architecture Search}\label{archsearch}

Although there are numerous architecture search methods that we could employ, we chose a derivative-free approach that completes a local search using finite difference approximations. This choice was made to mimic the weak derivatives available to us from viewing neural networks as functions of Sobolev spaces. As will become evident in Section \ref{HJB}, a notion of the change in the architecture is needed in order to model the dynamics and solve the lower optimization. We call the search method used a neighborhood directional direct-search (NDDS). The standard directional direct-search method is described in the survey in  \cite{larson2019derivative}.

The intuition behind NDDS is that we check ``neighboring" architectures and compare the current architecture with the expected value of the neural network when it trains on the neighboring architecture. If the neighboring architecture has a smaller expected value,
then we ``move" in that direction by selecting it as the new architecture. We then check the ``neighbors" of this new architecture. Since the architecture is a discrete variable, this search method provides a notion of a ``gradient." 

With this intuition, we now describe NDDS in depth. As we discuss this method, consult Algorithm \ref{DDSalg}. For task $t\in \tT,$ we set $\tx = \mathcal{Y}$ to be our training data and $J(\weight(t),\psi(t), \mathcal{Y})$ the expected value function. We set $x_s = \psi(t)$, and let $D_s$ be a finite set of directions. Further, we choose a ``step size" $\alpha_s\in \mathbb{N}.$ Now we generate the ``neighboring" points for $x_s$ via $D_s$ and $\alpha_s.$ We call these points \textit{poll points} and define them as follows:
\begin{align*}
    Poll Points = \{x_s + \alpha_s d_i:d_i\in D_s\}.
\end{align*}
Using a randomly selected subset of our data, denoted $\mathcal{Y}_s,$ we determine and compare the expected value of the neural network with randomly initialized weights according to each poll point architecture. If one of these poll points results in a smaller expected value than $x_s,$ we set $x_s$ equal to the poll point. We then repeat this process starting with this new architecture. This search terminates when the expected value is beneath a previously chosen threshold or after five times of repeating. In practice, we chose the threshold to be a percentage lower than the expected value produced by the original architecture.

Let us walk through an example. Suppose we would like to learn the optimal number of neurons per layer in our hidden layers for training a feedforward neural network (FNN) on a dataset. Because of the data, suppose the input layer is fixed at $784$ neurons and the output layer is fixed at $10$ neurons per layer. Let us fix the number of hidden layers at $2$ beginning with $50$ neurons in each hidden layer. We set $\mathcal{Y}_s$  to be a randomly chosen appropriately sized subset of task $t$ data and previous task data. Our direction set is $D_s = D = \{[0,0,10,0],[0,10,0,0]\}$ and will be the same for every $s.$ We start our search by setting the step size $\alpha_s = 10.$ For the first round, we have $x_s = [784,50,50,10],$ so our poll points are $[784,60,50,10],[784,50,60,10],$ and $[784,60,60,10].$ We then train an FNN of each size on the dataset $Y_s$ and determine the corresponding expected values. If the expected value of the FNN $[784,60,50,10]$ is smaller than that of the FNN $[784,50,50,10],$ for example, then $x_s = [784,60,50,10].$ We would then continue this process until the expected value of our network was less than our set threshold or we exhausted our neighborhood search.

      

\begin{algorithm}[t]
\caption{Neighborhood Directional Direct Search (NDDS)} \label{DDSalg}
\KwIn{Initial architecture $x_s$, step size $\alpha_0 \in \mathbb{N}$}
\KwIn{Threshold parameter $\mathrm{threshold} \in \mathbb{R}$}
\KwOut{Updated architecture $x_s$}

$j \gets 0$\;
$\mathrm{loss} \gets \texttt{training\_loop}(x_s, \mathcal{Y}_s)$\;
\While{$\mathrm{loss} > \mathrm{threshold}$ \textbf{or} $j < 5$}{
    $\mathrm{pollPoints} \gets \{x_s + \alpha_s d_i : d_i \in D_s\}$\;
    \ForEach{$\mathrm{poll} \in \mathrm{pollPoints}$}{
        $\mathrm{loss}_s \gets \texttt{training\_loop}(\mathrm{poll}, \mathcal{Y}_s)$\;
        \eIf{$\mathrm{loss} \leq \mathrm{loss}_s$}{
            $x_s \gets x_s$\;
        }{
            $x_s \gets \mathrm{poll}$\;
            $\mathrm{loss} \gets \mathrm{loss}_s$\;
        }
    }
    $j \gets j + 1$\;
}
\end{algorithm}

\subsection{\underline{Step 2:}  Low-Rank Transfer}\label{LRTsec}

Now that the optimal architecture (i.e., number of neurons per layer) has been determined for the current task, it is immediately obvious that  the size of the weights tensor will not match with the new architecture. Thus, we seek a method to determine a weights tensor corresponding to the new optimal architecture that retains previously learned information and transfers it. We propose a method we call low-rank transfer. 

Before we dive into the steps of this method, let us set some assumptions and notation. Recall that the only component of the architecture we seek to learn is the number of neurons per layer in our neural network. Thus, we assume our network has $d$ layers, and we fix all other architecture parameters except for the number of neurons in our hidden layers. From  Definition \ref{defn:NN},  $\psi_i(t)$  provides the dimensions for the corresponding weights matrix in each layer of our network.

In order of appearance, we assign the values of $\psi(t)$ to the values $r_i,s_i$ for each $i$ such that $1\leq i\leq d.$ If $\psi^*(t)$ represents the optimal architecture returned from the NDDS completed for task $t,$ then similarly we can assign the values of the dimensions of the weights matrices of each layer of $\psi_i^*(t)$ to be $a_i,b_i$ for $1\leq i\leq d.$ Our goal is to utilize the original weights matrices to project into the weights matrix space for the new architecture.

To begin, we initialize dimension $3$ tensors $\mathrm{A}(t),\mathrm{B}(t),$ which  each comprise  $d$ matrices. Let each matrix $\mathrm{A}_i(t)$ in $\mathrm{A}(t)$ be randomly generated with dimensions $a_i\times r_i$ for $1\leq i\leq d.$ Similarly, let each matrix $\mathrm{B}_i $ in $\mathrm{B}(t)$ be randomly generated with size $b_i\times s_i$ for $1\leq i\leq d.$ Then, set $C(t) = \mathrm{A}(t)\weight(t)\mathrm{B}^T(t).$ More specifically $C(t)$  comprises $d$ matrices such that $C_i(t) = \mathrm{A}_i(t)\weight_i(t)B^T_i(t)$ for all $1\leq i\leq d.$ Notice that the dimensions of each $C_i(t)$ are $a_i\times b_i.$ These are the dimensions of the weights matrices for a neural network with the new optimal architecture $\psi^*(t).$ See Figure \ref{LRT} to understand the construction of $C(t)$ more explicitly.
    \begin{figure}
        \centering
        \includegraphics[width=\linewidth]{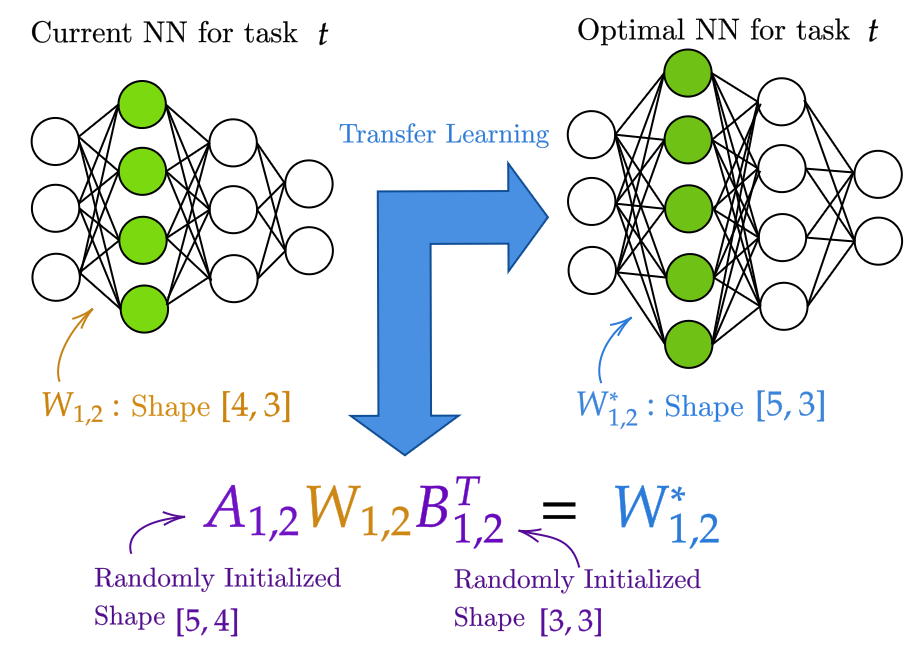}
        \caption{Method of low-rank transfer at a single layer} \label{LRT}
    \end{figure}

 To prevent loss of information from $\weight(t),$ when we make the transfer to $C(t),$ we  train only the $\mathrm{A}(t)$ and $\mathrm{B}(t)$ portions of the new weights tensor $C(t)$ on the task data for a chosen number of epochs, while freezing the $\weight(t)$ tensor. This additional training ensures a transfer of learning to the new weights corresponding to the new optimal architecture. If $C^*(t)$ represents $C(t)$ after this training on just tensors $\mathrm{A}(t)$ and $\mathrm{B}(t)$ and $\psi(t+1) = \psi^*(t),$ then we conclude our process by letting $w(t+1) = C^*(t)$ and completing the standard training of our neural network $\hat{f}(w(t+1),\psi(t+1))$ on task data. In Algorithm \ref{alg:three} we summarize each step of our method described in the preceding subsections.
 
\begin{algorithm}
\caption{Main Training Loop}\label{alg:three}
Choose $\weight(t)$ and $\psi(t)$ to begin\\
Set \texttt{epoch hyper-parameter}\\
\For{$t = 0,1,\ldots, T$}{
 \textbf{Step 1:} Standard Training of $\weight(t)$\\
    $\weight(t) \gets $\texttt{training\_loop}($\weight(t),\psi(t),$ $\tx(t)$, \texttt{epochs})\\

 \textbf{Step 2:} Architecture Search\\
    $\psi^*(t) \gets$ \texttt{NDDS}$(\psi(t),$ $\tx(t)$)\\
    
\textbf{Step 3:} Initialize $\mathrm{A}(t), \mathrm{B}(t)$\\
\For{$i = 1,...,d$}{
    $\mathrm{A}_i(t), \mathrm{B}_i  \gets$\texttt{init\_AB}($a_i,b_i,r_i,s_i)$}

 \textbf{Step 4:} Set $C(t)$\\
    \For{$i = 1,...,d$}{$C_i(t) = \mathrm{A}_i(t)\weight_i(t)B_i^T(t)$}

 \textbf{Step 5:} Fix $\weight(t),$ Train $\mathrm{A}(t),\mathrm{B}(t)$ for $\hat{f}(C(t),\psi(t+1))$\\
    $C^*(t) \gets$ \texttt{train\_AB}($C(t), \psi(t+1),$ $\tx(t)$, \texttt{epochs})\\

 \textbf{Step 6:} Set New Weights and Architecture\\
    $w(t+1) = C^*(t)$\\
    $\psi(t+1) = \psi^*(t)$\\

 \textbf{Step 7:} Standard Training on New NN
    \texttt{training\_loop}\\($\weight(t+1),\psi(t+1),$ $\tx(t)$, \texttt{epochs})
}
\end{algorithm}

\section{Related Works}
As noted in the introduction, theoretical and empirical results have proven that even in traditional machine learning, training weights of the network alone is not enough. By utilizing tools such as NAS (Neural Architecture Search), LoRA (Low-Rank Adaptation), and PEFT (Parameter-Efficient Fine Tuning), the accuracy of a trained model can readily be increased (\cite{liu2021survey}, \cite{LORAhu2022lora},  \cite{han2024parameter}). In the context of CL, however, learning optimal hyperparameters or architecture for each task poses a challenge: 
 learning of previous tasks be transferred.

Toward this end, networks with dynamic structures, such as progressive neural networks (PNGs) and dynamically expandable networks (DENs) were introduced in \cite{rusu2016progressive} and \cite{yoon2017lifelong}. 
PNNs add a layer to the neural network at each observed task \cite{rusu2016progressive}. While the model's accuracy did improve with such structure, extensive compute time is unavoidable as more tasks, and hence layers, are observed. DENs attempted to remedy the computing issues of PNNs. DENs first determine a subcollection of neurons to train on for a given task (via a metric) and then add a layer of neurons to the network where appropriate \cite{yoon2017lifelong}. This method also completes sparse regularization at each task to prune the network in an attempt to reduce compute time compared with PNNs. The use of a DEN did improve the model's accuracy. However, drawbacks such as compute time, overfitting, and transfer of learning still remained. Learning from previous tasks was not fully transferred, since the weights corresponding to the newly added network layers were randomly initialized. 

In more recent years researchers have built on the idea of DENs by developing methods that optimally choose how and when to adjust network architecture. CLEAS (Continual Learning with Efficient Architecture Search) is the first integration of NAS into the CL framework \cite{CLEASgao2022CLEAS}. This method  changes architecture only  when deemed necessary by utilizing a neuron-level NAS, rather than adding an entire layer to the neural network. Moreover, the network is pruned at each step to remove unnecessary layers and/or neurons. The algorithm touts its smaller, purposely learned network structure, which allows for reduced network complexity and hence reduced computation time. In particular, CLEAS improves model accuracy by up to $6.70\%$ and reduces network complexity by up to $51.0\%$ on the three benchmark datasets. A strategy similar to CLEAS is CAS (Continual Architecture Search)  \cite{CASpasunuru2019continual}. The premise of  adding and removing to the architecture only when deemed optimal for a given task remains the same. Rather than using NAS, however, they use what they call Efficient Neural Architecture Search, which incorporates a weight-sharing strategy. SEAL (Searching Expandable Architectures for Incremental Learning) again seeks to change  the hyperparameters and architecture of the neural network only when necessary, but the method differs from CLEAS in how it determines when to alter the network \cite{SEALgambella2025SEAL}. In particular, SEAL uses a capacity estimation metric to make such a decision.

The previous methods all sought to expand the network in some fashion to increase the accuracy of the network. Now we shift to how parameter-efficient fine-tuning methods have been utilized in CL. In particular, we investigate the use of LoRA (Low-Rank Adaptation) \cite{LORAhu2022lora}. The two methods we will discuss here are intended for pretrained transformers that are now attempting to learn tasks. The CoLoRA method utilizes a new LoRA adapter for each new task to train a task expert model \cite{wistuba2023continual}. The expert model is then used to train the transformer. At the time, this method produced state-of-the-art results, but at a high computational cost. Since then, the CLoRA method has been introduced \cite{CLORAmuralidhara2025}. The key difference between these two strategies is the number of LoRA adapters. Rather than reinitializing and training a new LoRA adapter for each task, a single-adapter approach is used. This reduces the compute time, while maintaining accuracy and transferring learning.

In this paper we attempt to combine the most profitable aspects of methods such as CLEAS and CLoRA. In other words, our goal is to learn the optimal architecture for each task and utilize a low-rank transfer method to optimally transfer learning to the new architecture. The use of a LoRA-like method to easily guarantee transfer of learning from one architecture to another in standard CL has not been accomplished. Moreover, no previous algorithm has provided a theoretical mathematics framework to support the empirical data. 

\section{Experimental Results}\label{sec: experiments}
In this section we evaluate the algorithm in \ref{alg:three} and seek to answer the following questions: Does training with weights for a large number of tasks lead to a saturation as discussed in Section~\ref{sec:motivation}? Does changing the architecture actually help with learning over consecutive tasks~(Section~\ref{sec:upper_weight})? Does changing the architecture improve saturation over a series of tasks? And do these ideas carry forward to different types of neural network architectures? We examine these questions for three continual learning problems: regression, image classification, and graph classification.

\subsection{Datasets and Metrics}

\textbf{Datasets and Task Splits:} We utilize three data sets, one for each learning problem.

\begin{itemize}
    \item \textbf{Random Sine:} For the regression experiment, we utilize a randomly generated sine data set. Each task is defined using a sine function $y = a\sin(bx + \phi) + \varepsilon,$ where $a$ is the amplitude, $b$ is the frequency, $\phi$ is the phase shift and $\epsilon \sim \mathcal{N}(0, \sigma^2)$ is Gaussian noise. Inputs are sampled from $x \sim \mathcal{U}([-90, 90])$. Tasks differ by amplitude and phase. This is ideal, as it allows for distinct, but controllable, distribution shifts in data at each task.

    \item \textbf{MNIST:}~\cite{lecun1998mnist} We split MNIST into sequential tasks by grouping digit pairs depending on the number of tasks. For example, if we are to learn $5$ tasks total, Task $0$ consists of digits $0$ and $1,$ and Task 1 consists of digits $2$ and $3.$ Moreover, images are $28 \times 28$ grayscale.

    \item \textbf{Synthetic Graphs:} We generate random graphs using PyTorch Geometric's FakeDataset ~\cite{pyg}. Each graph contains an average of 15 nodes with 10-dimensional node features and belongs to one of 5 classes for graph-level classification. We emulate task distribution shifts in this setting by allowing all tasks to share the same five classes but we introduce progressive perturbations at each additional task learned. More specifically, at each task $t,$ we apply three types of perturbations: (1) Gaussian feature noise with standard deviation $\sigma_t = 0.02t$, (2) random edge dropout with probability $p_t = 0.01t$, and (3) constant feature shift of magnitude $s_t = 0.01t$. Thus, Task $0$ serves as the unperturbed baseline. This is an optimal choice for the continual learning setting, as it tests the models ability to adapt to distribution shift while retaining knowledge.
\end{itemize}

\textbf{Metrics:} We utilize three standard metrics to evaluate continual learning performance ~\cite{lopez2017gradient, chaudhry2018riemannian, diaz2018dont}. Let $R_{j,i}$ denote the performance on task $i$ after training on task $j$. For regression, $R_{j,i}$ is the mean squared error (MSE); for classification, $R_{j,i}$ is accuracy.


\begin{itemize}
    \item \textbf{Average Performance:} The mean performance across all tasks after training:
    \[
    \text{Avg} = \frac{1}{T} \sum_{i=0}^{T-1} R_{T-1,i}
    \]

    \item \textbf{Backward Transfer (BWT):} Measures how learning new tasks affects previous tasks~\cite{lopez2017gradient}.
    \[
    \text{BWT} = \frac{1}{T-1} \sum_{i=0}^{T-2} \left( R_{i,i} - R_{T-1,i} \right)
    \]
    Note that for MSE, positive BWT indicates improvement.

    \item \textbf{Forgetting (Forward Transfer FWT):} The average increase in error from the best achieved performance~\cite{chaudhry2018riemannian}:
    \[
    \text{FWT} = \frac{1}{T-1} \sum_{i=0}^{T-2} \max\left(0, R_{T-1,i} - \min_{j \leq T-1} R_{j,i}\right)
    \]
\end{itemize}

For all experiments, we train 500 epochs per task. We use MSE loss for regression and cross-entropy loss for classification.

\subsection{Experimental Conditions}
To explore and thoroughly evaluate the theoretical work we have developed, we consider four experimental conditions. Each condition allows us to test a particular aspect of our theory by isolating the effects of learning architecture and transferring knowledge. Specific implementation details about each condition is included in Section ~\ref{sec:implementdetails}, but we provide an overview of each condition below for quick reference.

\begin{itemize}
    \item \textbf{C1 (Baseline):} Standard continual learning with fixed architecture. The network architecture $\psi$ remains constant across all tasks as well as the learning rate. For this baseline comparison, we use a replay-driven continual learning approach which seeks to balance forgetting and generalization \cite{raghavan2021formalizing}, ~\cite{raghavan2023learningcontinuallysequencegraphs}. See Figure~\ref{fig:alg_conditions}(a) for a description of the algorithm.

    \item \textbf{C2 (Heuristics):} Fixed architecture with training heuristics. The architecture $\psi$ remains fixed, but we add cosine learning rate decay, task warm-up epochs (Algorithm~\ref{alg:warmup}), and adaptive gradient weights $[\alpha, \beta, \gamma]$ (Algorithm~\ref{alg:adaptive}). This condition tests whether empirical smoothing techniques alone can improve performance. See Figure~\ref{fig:alg_conditions}(b) for a full description of the algorithm.

    \item \textbf{C3 (Architecture Search Only):} Architecture search without knowledge transfer. At each task, we search for a new architecture $\psi(t+1)$ using the NDDS method (Algorithm~\ref{DDSalg}). However, once the new architecture is determined, the weights are randomly re-initialized to the corresponding sizes rather than introducing transfer matrices $A$ and $B.$ The purpose of this method is to isolate the impact of architecture change from that of knowledge transfer. See Figure~\ref{fig:alg_conditions}(c) for an algorithm description.

    \item \textbf{C4 (Knowledge Transfer via $AWB$):} This is our full approach, which combines architecture search with low-rank knowledge transfer via the  $V = A \weight B^T$. When the architecture changes, we train $A$ and $B$ matrices (Algorithm~\ref{alg:train_ab}) to transfer knowledge from the old parameter space to the new one, as described in Theorem~\ref{thm:lower_HJB}. See Figure~\ref{fig:alg_conditions}(d) for an algorithm description.
\end{itemize}

The comparison of conditions C3 and C4 are crucial. Although both of the conditions change (learn) an optimal architecture, only C4 transfers knowledge. This isolates the contribution of the knowledge transfer from architecture search alone.


\subsection{Implementation Details}\label{sec:implementdetails}

\subsubsection{Experimental Infrastructure}  
All experiments were conducted on a 14-inch MacBook Pro (2024) with Apple M4 Pro chip (14-core CPU, 20-core GPU, 24 GB unified memory) and the system ran macOS (Sequoia 15.6.1). The codebase uses JAX~\cite{jax} and Equinox~\cite{kidger2021equinox}. Experimental outputs were logged and visualized using TensorBoard. All runs used fixed random seeds for reproducibility. Code is available at \url{https://github.com/krm9c/ContLearn.git}.

\subsubsection{Experimental Conditions Implementation Details}
All conditions use a replay-driven continual learning approach, which seeks to balance forgetting and generalization demonstrated in ~\cite{raghavan2021formalizing} for regression and image classification experiments. For graph experiments, we use the similar approach outlined in ~\cite{raghavan2023learningcontinuallysequencegraphs} Note that we use an experience replay buffer size for all experiments and conditions.

\paragraph{Conditions C1 and C2: Baseline Training }

For C1, training uses constant learning rate $\eta = 10^{-4}$ with AdamW optimizer. No warm-up or adaptive mechanisms are applied.

For C2, we add three heuristics to the baseline training conditions:
\begin{itemize}
    \item \textit{Heuristic 1:} Cosine learning rate decay from $10^{-4}$ to $10^{-6}.$
    
    \item \textit{Heuristic 2:} Task warm-up for $25$ epochs at the start of each new task. The purpose of the task warm-up is train on the next task for a select number of epochs prior to adjustments in other conditions. See Algorithm~\ref{alg:warmup} for implementation.

    \item \textit{Heuristic 3:} Adaptive gradient weights $[\alpha, \beta, \gamma].$ Such weights adjust based on task transition difficulty, and are for current task, experience replay, and regularization, respectively. See Algorithm~\ref{alg:adaptive} for implementation.
\end{itemize}

\paragraph{Condition C3: Architecture Search }
At each task, we use a neighborhood direct-directional search (NDDS) method to find an optimal architecture. The search explores candidate architectures by incrementally adding neurons to hidden layers and/or adjusting the filter size. Below, we provide implementation examples of NDDS in the feedforward, CNN, and GCN settings, but see Section~\ref{archsearch} and Algorithm~\ref{DDSalg} for specific details.
\begin{itemize}
    \item \textit{Feedforward Networks:} Suppose we begin with an architecture of $[10,40,40,10]$ and our step size of our search is $4.$ Our candidates to be evaluated would be:  $[10, 45, 40, 10]$, $[10, 40, 45, 10]$, and $[10, 45, 45, 10].$ After training each candidate for $100$ epochs, the best-performing architecture (smallest loss value) will be chosen for the next round.

    \item \textit{CNNs and GCNs:} To jointly search for optimal filter size and hidden layers, a nested loop is used. For each filter size, multiple rounds of step-wise candidate evaluations are completed, akin to the feedforward network setting.
\end{itemize}

After determining the new (optimal) architecture for the given task, we must adjust the weights matrices, as the current weights will not longer match the new architecture. In C3, we choose to randomly reinitialized using Glorot uniform~\cite{GlorotB10}. The random reinitialization allows us to evaluate the effect of architecture optimization alone, without transferring previously learned information.

\paragraph{Condition C4: Knowledge Transfer via AWB }
We again use NDDS method to find optimal architecture for our given task. The search explores candidate architectures by incrementally adding neurons to hidden layers and/or adjusting the filter size. Once an optimal architecture is chosen, we must again adjust the weights matrices. In C4, we choose to transfer learned knowledge from $\weight(t)$ to $\weight(t+1)$ through the $A$ and $B$ matrices. The transformation $V = A \weight B^T$ maps parameters between architectures of different sizes, see Section~\ref{LRTsec}. 

A key component of C4 is that we allowing training of the matrices $A$ and $B,$ while freezing the matrix $\weight,$ prior to set the new weights matrix to $V = A \weight B^T.$ This is crucial as it provides tuning of the matrices $A$ and $B$ before transferring the learning in $\weight.$
To complete this process in practice, we perform the following steps:
\begin{itemize}
    \item[(1)] \textit{Freeze the weights matrix $\weight$ and initialize $A$ and $B$ matrices}: The appropriate sizes for the $A$ and $B$ matrices are chosen based on the new (optimal) architecture and the previous architecture. The goal is for $V = A \weight B^T$ to be the correct size for the new architecture. In Appendix~\ref{app: MLP AB}, we walk through an example of determining the correct sizes of $A$ and $B$ for each layer of a feedforward neural network, which is transitioning from an architecture of  $[64,128,128,10]$ to an architecture of $[64,256,256,10].$ In Appendix~\ref{app: CNN AB}, we complete an example similar example but for the CNN setting (i.e. accounting for filter and FFN transfer).
    
    \item[(2)]\textit{Train $A$ and $B$ Matrices:} Training is completed for a fixed number of iterations using Adam. See Algorithm~\ref{alg:train_ab} for implementation details.\\

    \item[(3)] \textit{Set $V = A\weight B^T$ and Train: } The weights are set to $V = A\weight B^T,$ which now corresponds to the new architecture of the network. We then the train the new weights and record necessary information. 
\end{itemize}

The transfer and training process ensures continuity of the loss function across architecture changes, as established in Theorem~\ref{thm:lower_HJB}. For a succinct description of C4, see Algorithm~\ref{alg:c4}.

\subsection{Regression Experiments}~\label{sec:reg}
In this section, we complete and analyze  three regression experiments: learning two tasks, learning ten tasks, and learning five tasks with the introduction of noise into each task. In each experiment, we compare the four experimental conditions previously discussed.

\subsubsection{Overview of Experiments}
Below, we specify the details of data, task splits, and hyperparameter choices for such experiments.

\textbf{Dataset and Task Splits:} As stated in the beginning of Section \ref{sec: experiments}, we will be utilizing a randomly generated sine dataset for our regression experiments, as it is a standard benchmark for continual learning in regression settings. Recall that each task is defined by a sine function with a distinct amplitude and/or phase, where inputs $x,$ are sampled from a fixed domain $x \sim \mathcal{U}([-90, 90]).$ This controlled variation across tasks provides a simple, yet nontrivial, setting to evaluate a model’s ability to learn sequentially and resist catastrophic forgetting.

\textbf{Hyperparameter Choices:}  For the regression problem, we utilize a standard feedforward neural network with 4 layers. Training uses AdamW~\cite{loshchilov2017fixing} with learning rate $10^{-4}$, batch size 1024, and Glorot uniform~\cite{GlorotB10} initialization. For C2,C3, and C4 we use a cosine learning rate scheduling with 25 warm-up epochs at each task transition. Gradient weights are $[\alpha, \beta, \gamma] = [0.4, 0.4, 0.1]$ for current task, experience replay, and regularization respectively. See further hyperparameter details in Tables~\ref{tab:hyperparameters} and~\ref{tab:sine_hyperparams} of Appendix~\ref{app:standard_cl}

\begin{figure}[!htb]
    \centering
    \includegraphics[width=0.95\linewidth]{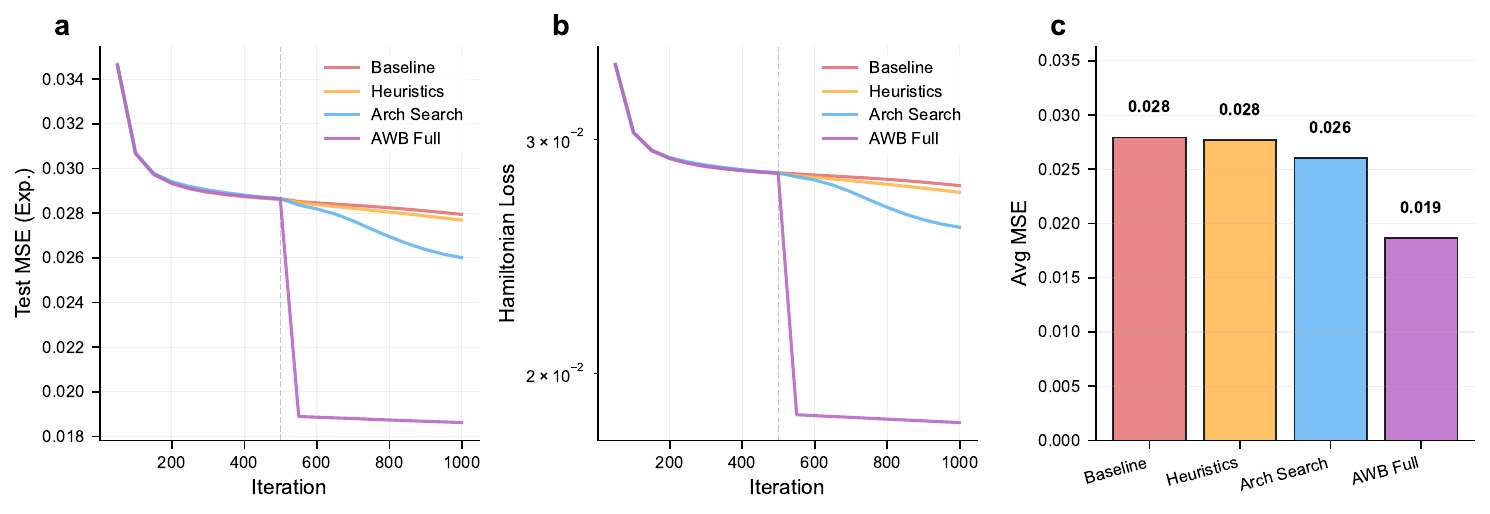}
    \caption{Sine regression (2 tasks): Comparison of conditions C1--C4. (a) Test MSE on experience replay. (b) Hamiltonian loss. (c) Average MSE showing C4 achieves 33\% improvement over baseline.}
    \label{fig:sine_2task_comprehensive}
\end{figure}

\subsubsection{Experiment 1}
 For the first experiment, we learn two tasks and train $500$ epochs on each task. Figure ~\ref{fig:sine_2task_comprehensive} provides a comprehensive comparison of all four conditions for this two task experiment. In ~\ref{fig:sine_2task_comprehensive} (b), with introduction of a new task, we observe that the model's Hamiltonian loss does not reduce under baseline CL condition (C1) or Heuristic condition (C2). This behavior suggests that, with changing tasks, the change in the weight parameters has no effect on the training and the learning stagnates. Notably, we argued in our theory, particularly in Lemma~\ref{lem:upper_weight}, that the change in the tasks can lead to a loss of continuity of the forgetting cost, leading to the conclusion that just changing the weights is not enough. 
 
To obviate this behavior, we seek to change the architecture on the fly. The architecture is changed with random reinitialization of weights in C3, and in C4 the architecture is changed with knowledge transfer of weights via transfer matrices $A$ and $B.$ In ~\ref{fig:sine_2task_comprehensive} (b) and (c), we observe that while the architecture search alone (C3) reduces the Hamiltonian loss compared to the baseline (C1) or Heuristics (C2) by 7\%, 
with the addition of the $AWB$ transfer (C4), we achieve a 33\% improvement over the baseline (C1). 

We note that the transfer between parameter spaces completed in C4 not only counters the effect of changing tasks, as proved in Theorem~\ref{thm:lower_HJB}, but also intuitively provides a better starting point for learning the new tasks. The difference is that learning behavior between the baseline (C1) and $AWB$ Full (C4) implies that the conclusions from Lemmas~\ref{lem:upper_arch} and \ref{lem:upper_weight} are indeed valid, where change in the architecture lowers the upper bound indicated by the lower loss values. All conclusions from Figure~\ref{fig:sine_2task_comprehensive} (b) extend to Figure \ref{fig:sine_2task_comprehensive} (c), where the test MSE values are graphed.

\subsubsection{Experiment 2}
For the second experiment, we seek to determine whether  the performance from the previous experiment extends to a larger number of tasks. In this process, we learn $10$ tasks of data and train $500$ epochs on each task. To ensure statistical reliability we run this $10$ task experiment with three random seeds. The results and an extensive comparison of the four conditions is provided in Figure~\ref{fig:sine_10task_comprehensive} and Table~\ref{tab:sine_cl_metrics}. In Table~\ref{tab:sine_cl_metrics}, we observe that AWB Full (C4) achieves a 39\% reduction in average MSE compared to the baseline (C1). Additionally, we see a 37\% reduction in forgetting over the baseline (0.0144 vs. 0.0227) and an increase in backwards transfer. More importantly though, AWB Full (C4) achieves a 37\% improvement in average MSE and 35\% improvement in forgetting in comparison to architecture search alone without knowledge transfer (C3). This is crucial, as we again observe the significant impact of transferring knowledge over architecture search alone, and we validate the theory we developed. These concludes are also reflected in the test MSE displayed in Figure~\ref{fig:sine_10task_comprehensive} (a), and the gradient norm evolution in (c) indicates training stability.

In Figure~\ref{fig:sine_10task_comprehensive} (d), (e), and (f), error bands ($\pm$ 1 std) across the three seeds. We observe a large error bound in the AWB Full (C4) condition across panels (d), (e), and (f). This is due to the inherent instability of architecture search methods and is expected. We again note that  any architecture search approach can be introduced here, including popular  methods such as DeepHyper~\cite{Balaprakash2018DeepHyper}.

\begin{table}[!htb]
\centering
\caption{Continual learning metrics for sine regression (10 tasks, 3 seeds). Values reported as mean $\pm$ std. Best values are highlighted in \textbf{bold}. Improvement equates to reduced Avg MSE and Forgetting values and increased BWT and FWT.}
\label{tab:sine_cl_metrics}
\begin{tabular}{lcccc}
\toprule
\textbf{Condition} & \textbf{Avg MSE ($\downarrow$)} & \textbf{BWT ($\uparrow$)} & \textbf{FWT ($\uparrow$)} & \textbf{Forgetting ($\downarrow$)} \\
\midrule
C1: Baseline         & $0.0256 \pm 0.0000$ & $0.0014 \pm 0.0000$ & $0.0000 \pm 0.0000$ & $0.0227 \pm 0.0000$ \\
C2: Heuristics       & $0.0255 \pm 0.0001$ & $0.0012 \pm 0.0000$ & $0.0000 \pm 0.0000$ & $0.0226 \pm 0.0000$ \\
C3: Arch Search      & $0.0251 \pm 0.0001$ & $0.0014 \pm 0.0001$ & $0.0000 \pm 0.0000$ & $0.0223 \pm 0.0000$ \\
C4: AWB Full         & $\mathbf{0.0156 \pm 0.0027}$ & $\mathbf{0.0005 \pm 0.0018}$ & $\mathbf{0.0000 \pm 0.0000}$ & $\mathbf{0.0144 \pm 0.0026}$ \\
\bottomrule
\end{tabular}
\end{table}

\begin{figure}[!htb]
    \centering
    \includegraphics[width=0.95\linewidth]{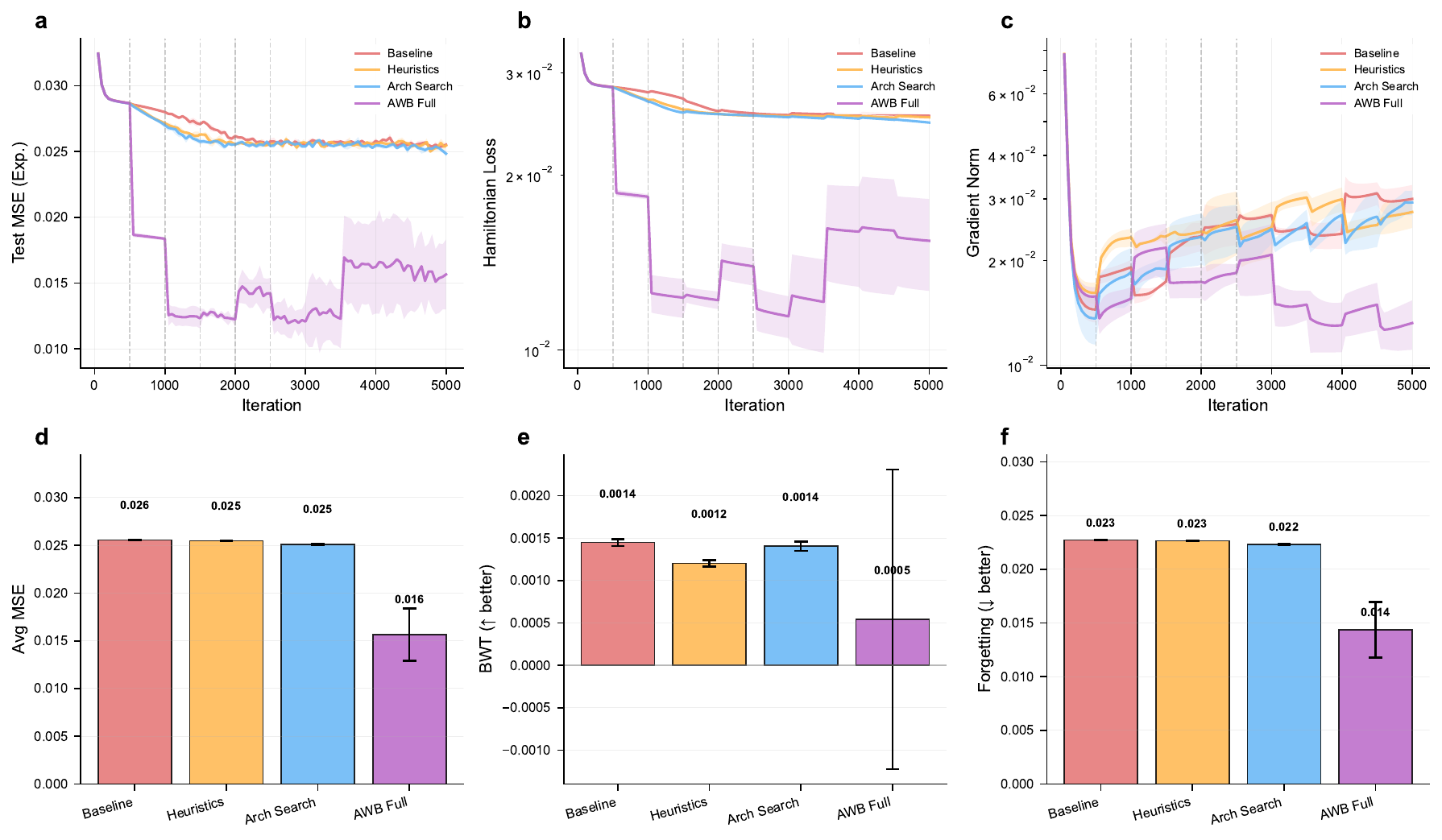}
    \caption{Sine regression (10 tasks, 3 seeds): Comparison of conditions C1--C4. Shaded regions show $\pm$1 std. (a) Test MSE on experience replay. (b) Hamiltonian loss. (c) Gradient norm. (d) Average MSE. (e) Backward transfer (BWT). (f) Forgetting.}
    \label{fig:sine_10task_comprehensive}
\end{figure}

\subsubsection{Experiment 3}
For the third experiment, we learn $5$ tasks of data and train $500$ epochs on each task. To  understand the implications of Lemma~\ref{lem:upper_arch} and \ref{lem:upper_weight}, we introduce noise into each subsequent task and observe the learning behavior.

In Figure~\ref{fig:sine_noise_comprehensive}, we provide a comprehensive comparison across the four conditions on the noisy data with three random seeds. We note that panels (a), (b), and (c) are on a logarithmic scale.  We observe a continual decrease in the Hamiltonian loss of the baseline (C1) and heuristic (C2) conditions, which contrasts results of experiments 1 and 2. Additionally, we achieve consistently lower error throughout training on C4. In panels (d), (e), and (f), we see that C4 provides a 4\% improvement in average MSE over C1. The backward transfer (BWT) metric confirms that C4 maintains positive transfer while minimizing forgetting. These results demonstrate that the AWB transfer mechanism remains effective even in the presence of noisy data, and supports the conclusions from Theorem~\ref{thm:contMeasure} and Lemmas \ref{lem:upper_arch} and \ref{lem:upper_weight},

\begin{figure}[!htb]
    \centering
    \includegraphics[width=0.95\linewidth]{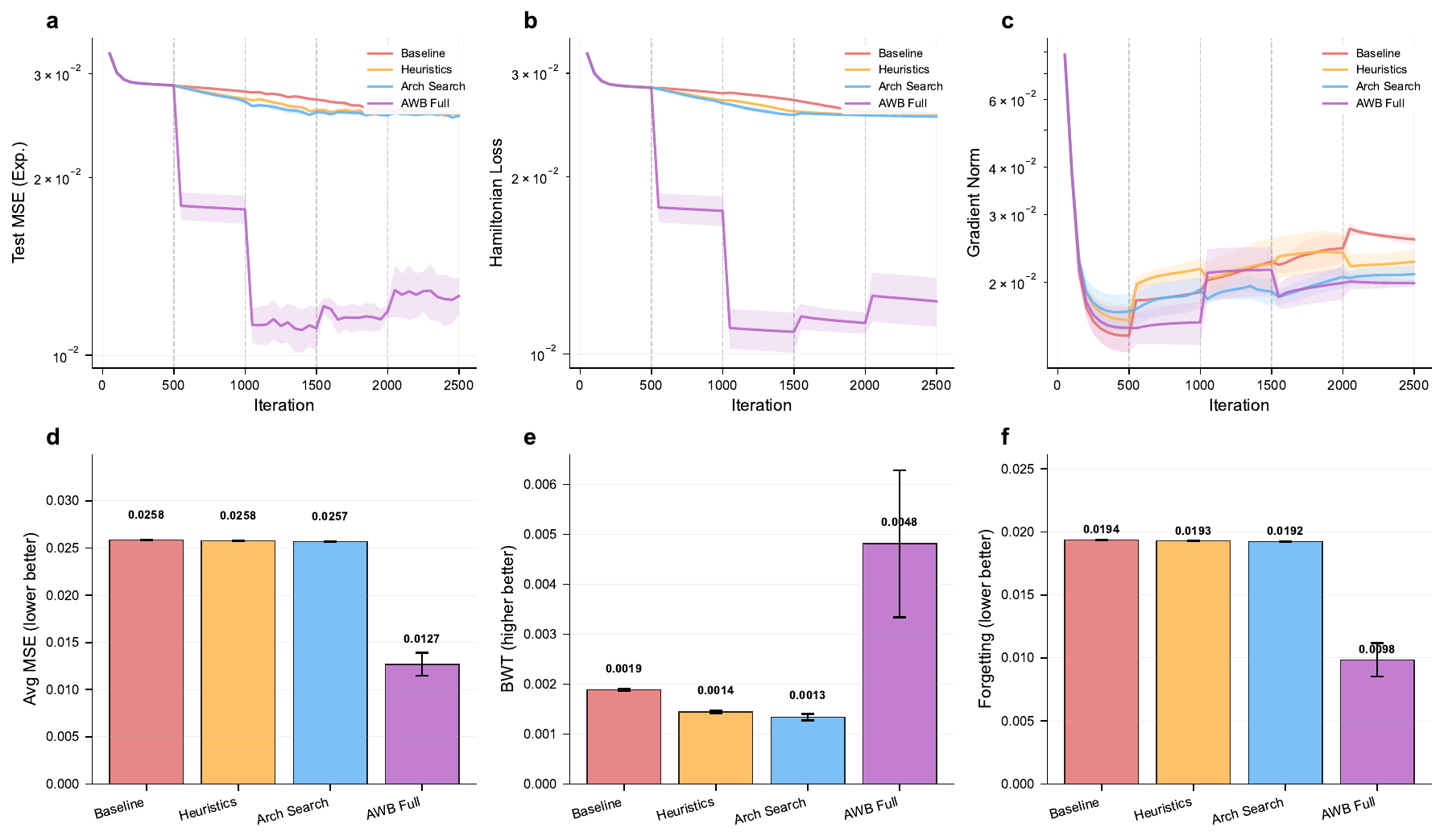}
    \caption{Noisy sine regression (5 tasks, 3 seeds): Comparison of conditions C1--C4 with noise introduced at each task. Shaded regions show $\pm$1 std. (a) Test MSE on experience replay. (b) Hamiltonian loss. (c) Gradient norm. (d) Average MSE. (e) Backward transfer (BWT). (f) Forgetting. C4 (AWB Full) achieves 4\% lower MSE than the baseline while maintaining positive backward transfer.}
    \label{fig:sine_noise_comprehensive}
\end{figure}

\textbf{Architecture Evolution:} Figure~\ref{fig:arch_evolution} visualizes the Hamiltonian loss alongside the change in architecture at each task. Under architecture search alone (C3), such search occurs without previous task knowledge since we randomly reinitialize weights. This results in the spikes in loss values at each task transition. In C4, we observe that transferring the learning via the AWB method, allows for the knowledge of previous tasks upon searching. This leads not only to smoother transitions between tasks, but consistent reduction in Hamiltonian loss. This visualization provides direct evidence that architecture adaptation combined with low-rank transfer (C4) outperforms architecture adaptation alone (C3).

\begin{figure}[!htb]
    \centering
    \includegraphics[width=0.95\linewidth]{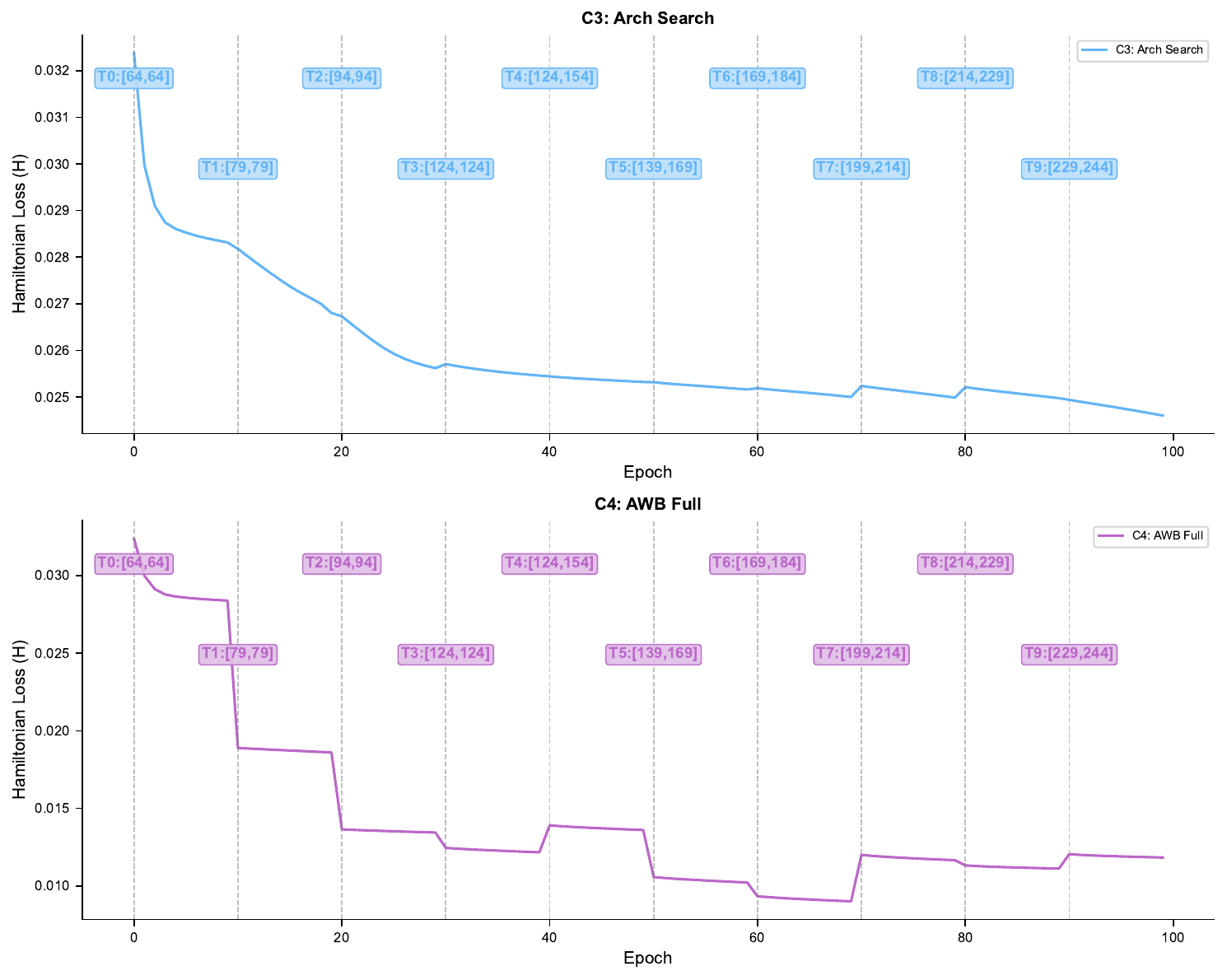}
    \caption{Hamiltonian loss evolution with architecture changes marked at task boundaries. Top: C3 (architecture search without transfer). Bottom: C4 (AWB full with transfer). Labels show network architecture at each task transition. C4 demonstrates smoother transitions and lower overall loss due to the AWB transfer mechanism.}
    \label{fig:arch_evolution}
\end{figure}

\subsection{Image Classification Experiments}
To evaluate our approach on image classification, we conduct experiments on the MNIST dataset~\cite{lecun1998mnist}. We perform two classification experiments: learning two tasks and learning 10 tasks.

\subsubsection{Overview of Experiments}

\textbf{Dataset and Task Splits:} As stated above, we will be utilizing the $28 \times 28$ grayscale images in the MNIST dataset for our experiments. We split MNIST into sequential tasks by grouping digit pairs depending on the number of tasks. For our first experiment, we group digits $0$ through $4$ in Task $0$ and digits $5$ through $9$ in Task $1.$ For our second experiment (where we seek to learn $10$ tasks), each task consists of images of a different digit.

\textbf{Hyperparameter Choices:} We use a convolutional neural network with two convolutional layers (filter size 3, 3 output channels each) followed by $2 \times 2$ max-pooling and ReLU activations. The flattened features pass through a feedforward network with hidden layers of size 512 and 64. Training uses AdamW~\cite{loshchilov2017fixing} with learning rate $10^{-4}$, batch size 1024, for 200 epochs per task. For C2--C4, we use cosine learning rate scheduling with 25 warm-up epochs. Gradient weights are $[\alpha, \beta, \gamma] = [0.4, 0.4, 0.1]$.

\subsubsection{Experiment 1}

In the first experiment, we learn two tasks and train for $500$ epochs on each task. Full hyperparameter details for this experiment are in Appendix~\ref{app:standard_cl}. Below we see that Figure~\ref{fig:mnist_comprehensive} compares the four conditions (C1--C4) for this experiment. In panel (a), we see the test accuracy on the experience replay data, which exhibits that condition C4 (AWB Full) achieves the highest accuracy throughout training. In panel (b), we have the Hamiltonian loss evolution, where C4 shows fastest convergence. Finally, in panel (c), we summarize the average accuracy, showing that C4 achieves 94.5\% accuracy compared to 92.8\% for the baseline (C1).

Notably, C3 (architecture search without transfer) shows 3.4\% forgetting while C4 shows zero forgetting. This demonstrates that random re-initialization after architecture changes leads to catastrophic forgetting, while the AWB transfer mechanism preserves learned knowledge. These results validate the conclusions from Theorem~\ref{thm:lower_HJB} for classification tasks.

\begin{figure}[!htb]
    \centering
    \includegraphics[width=0.95\linewidth]{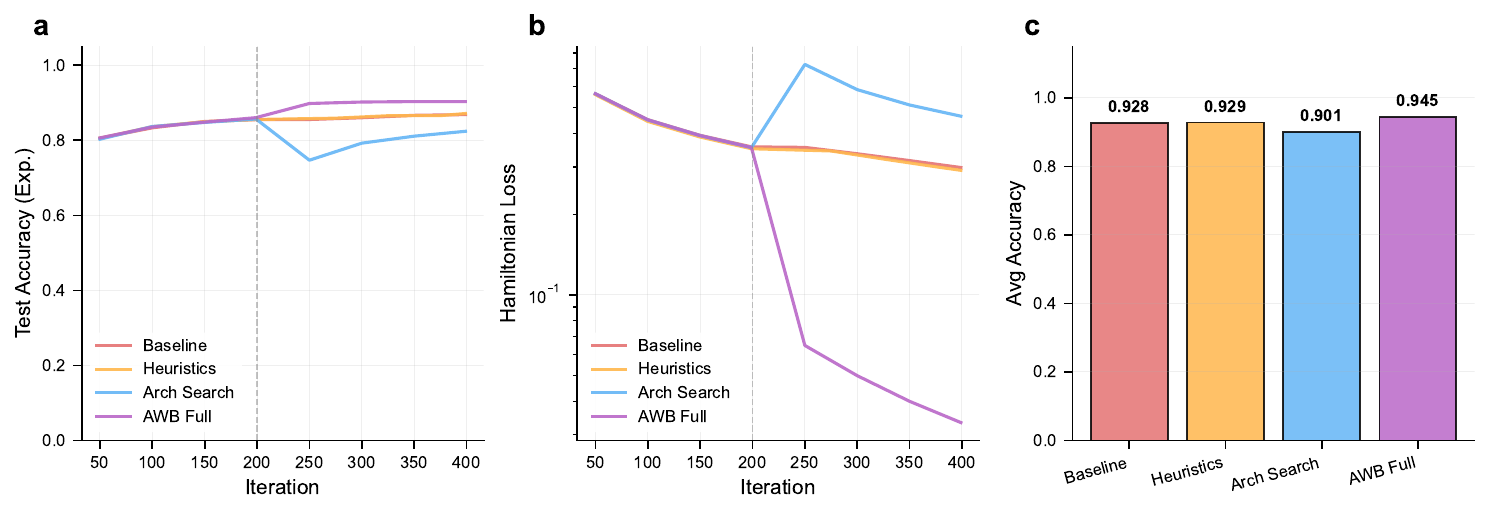}
    \caption{MNIST classification (2 tasks): Comparison of conditions C1--C4. (a) Test accuracy on experience replay. (b) Hamiltonian loss. (c) Average accuracy showing C4 achieves 94.5\% vs 92.8\% for baseline.}
    \label{fig:mnist_comprehensive}
\end{figure}

\subsubsection{Experiment 2}
As an extension to Figure~\ref{fig:mnist_comprehensive}, we conduct an experiment where we learn 10 tasks and plot the results in Figure~\ref{cnn}. We note now, that instead of the blanket improvement we observed in the sine dataset, we observe a significant improvement at the beginning for the first three tasks and then there is a significant jump that leads to significant deterioration. It is observed from Figure.~\ref{fig:diagnostics} that Task 3 introduces a problematic behavior with a sudden and sharp change into the task data, which leads to the model performing poorly. Our proposed A/B training is expected to smooth over this change, but this expected behavior is not observed in Figure.~\ref{fig:mnist_comprehensive}. Upon careful investigation, it was observed that the smoothing behavior is proportional to how many epochs are utilized for A/B training. To understand the effect of A/B training epochs, we perform the following experiment in the next section. 

\begin{figure}
    \centering
    \includegraphics[width=1\linewidth]{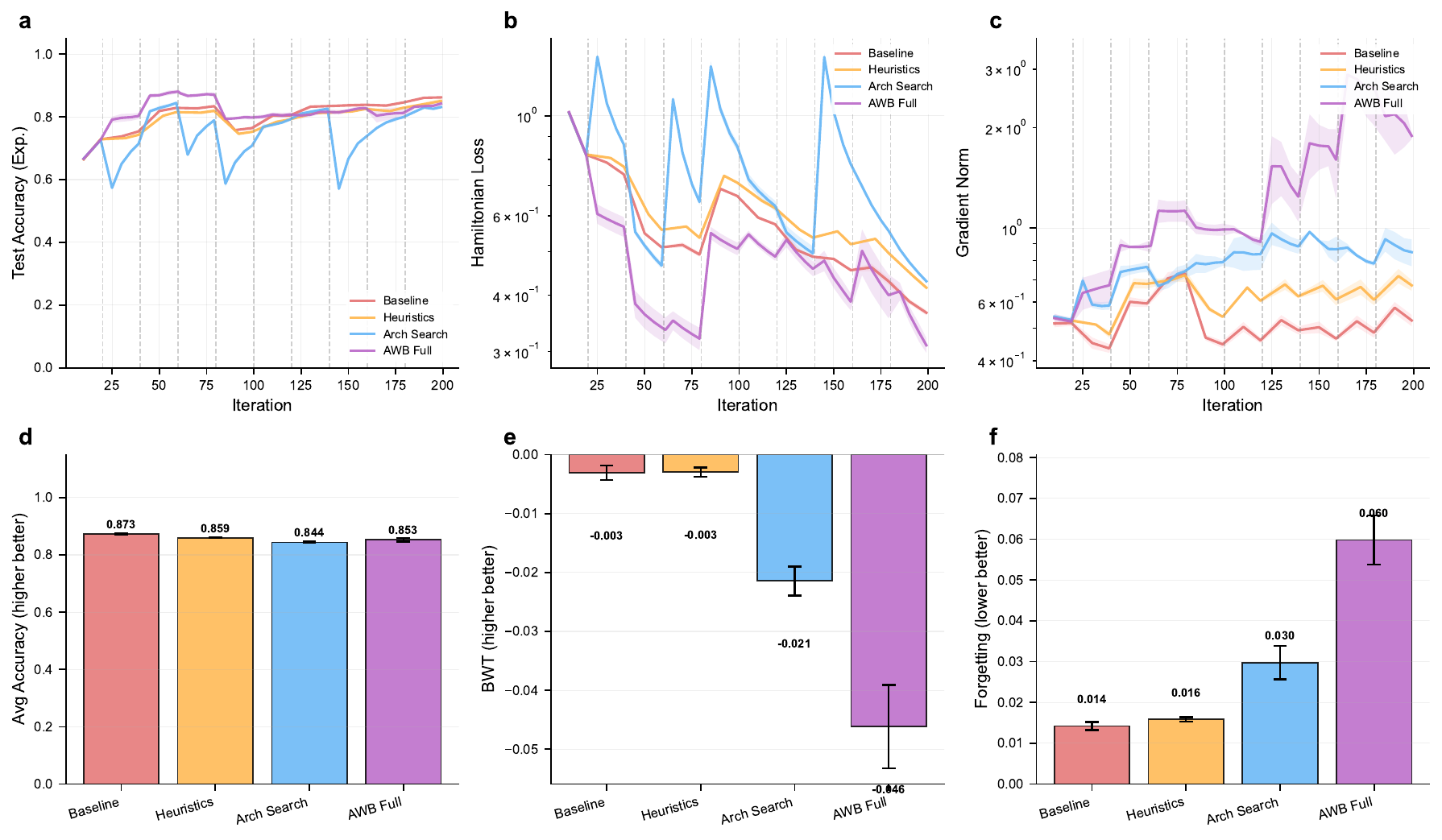}
    \caption{Classification experiment: Comparing loss values on training data for baseline continual learning method and method of learning optimal task architecture for the MNIST dataset}
    \label{cnn}
\end{figure}

\begin{figure}
    \centering
    \includegraphics[width=1\linewidth]{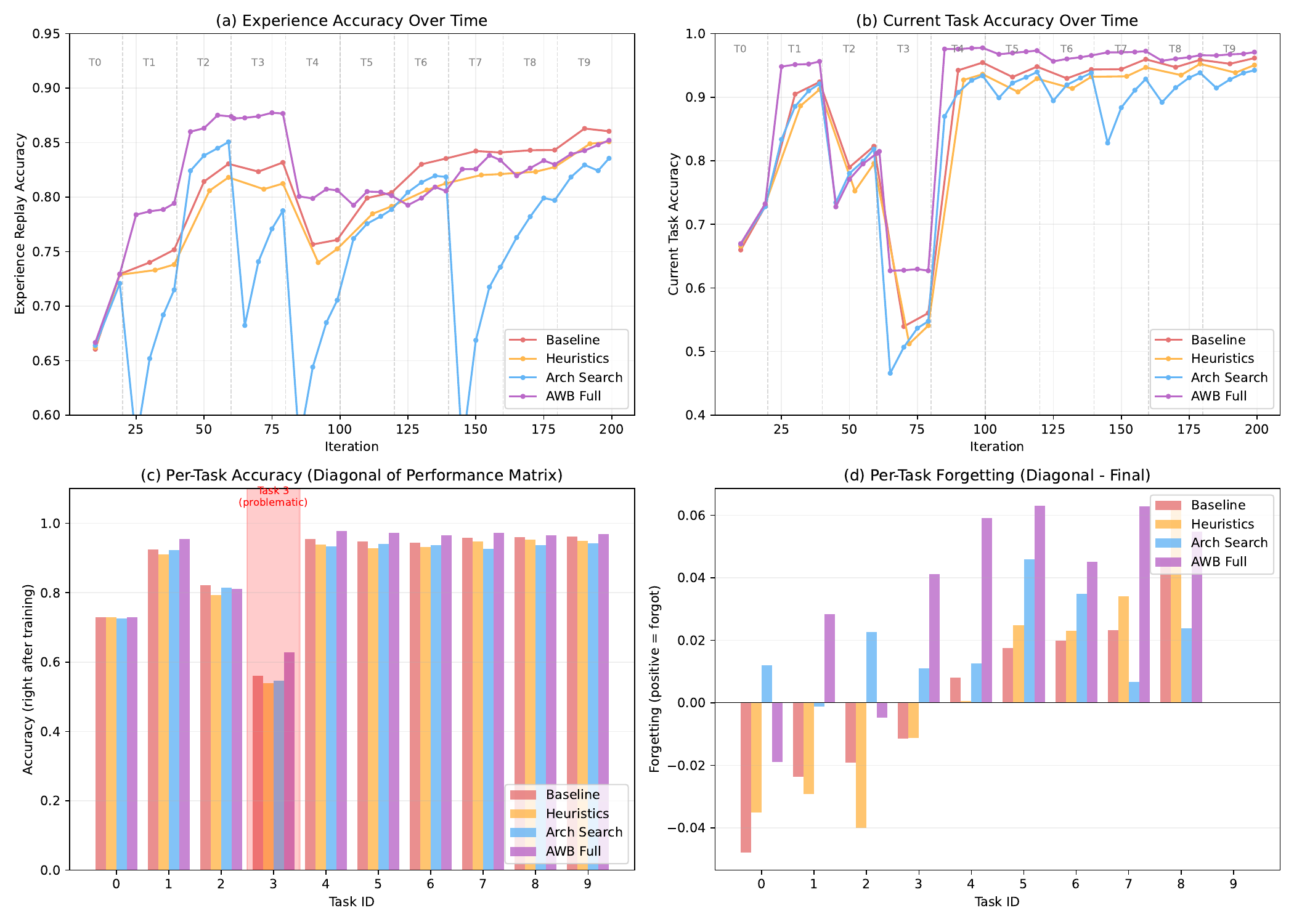}
    \caption{Classification experiment: More insight into the jump observed in the MNIST data set experiment}
    \label{fig:diagnostics}
\end{figure}

\subsubsection{Analysis of Reduced Accuracy in Experiment 2}
\label{sec:transform_analysis}
To understand the source of this performance drop in this experiment, we analyze A/B parameters affect classification difficulty. In our implementation, each task applies a random rotation angle $\theta \in [0, 180]^\circ$ and, importantly, uses the same angle for shearing: $\text{shear} = \theta$. This coupling of rotation and shear creates a compound distortion effect that varies non-linearly with the angle. Each task applies a random rotation and shear transform to MNIST digits, creating distribution shift between tasks. Figure~\ref{fig:transform_heatmap} shows how task difficulty varies with transform angle. At Task 3, we see a $92^\circ$ angle was applied. As an example, this would result in attempting to classify that an image showing a ``figure-eight" is the digit $8.$
\begin{figure}[!htb]
    \centering
    \includegraphics[width=0.95\linewidth]{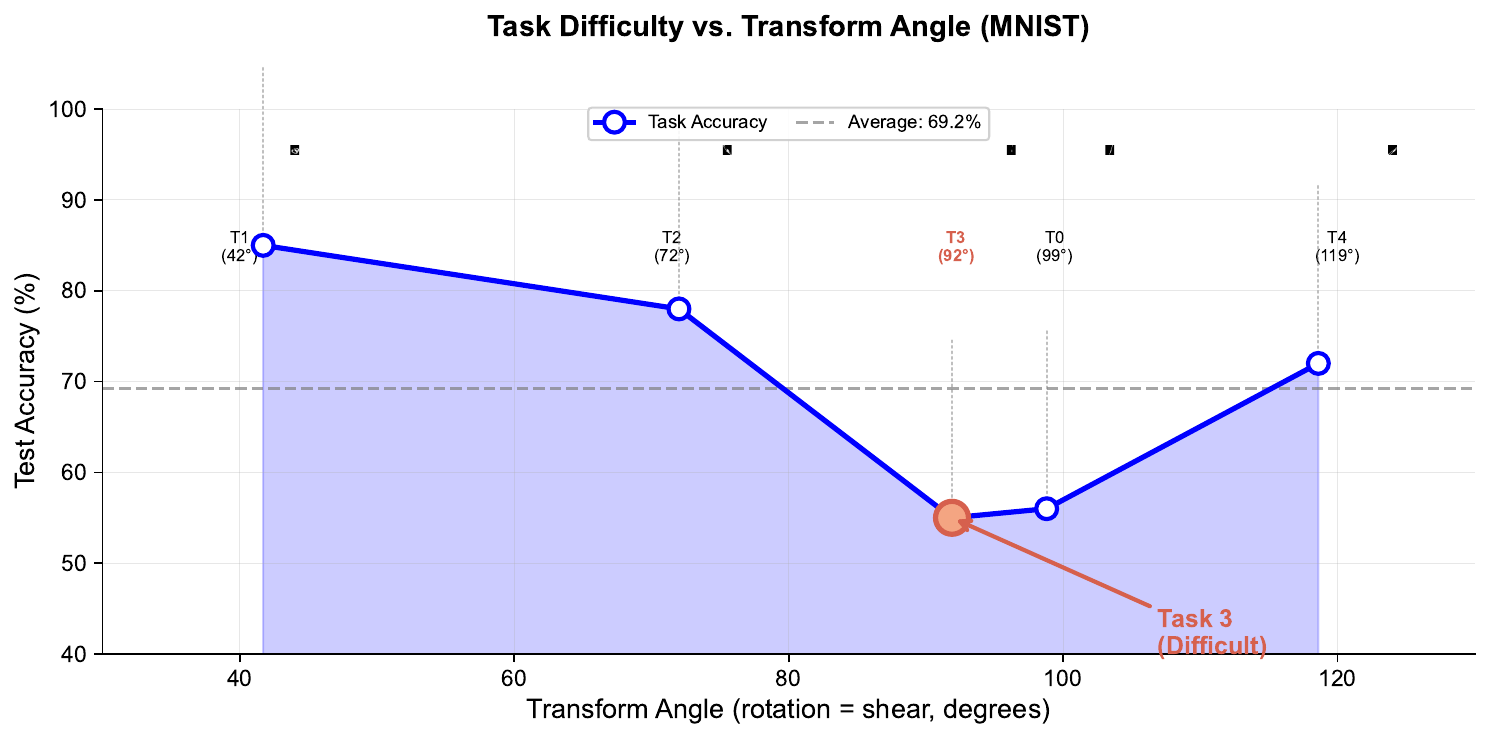}
    \caption{Task difficulty varies with transform angle. Top: transformed digit samples for each task. Task 3 (coral highlight) falls in a difficult region where combined rotation and shear severely distort digit shapes, resulting in lower accuracy.}
    \label{fig:transform_heatmap}
\end{figure}

As discussed in theory, that the prime way of handling such cases is to smoothen out the effect using A/B matrices which we use in this work. The key find is that the quality of smoothing is proportional to the number of epochs A/B matrices are trained for. A study of using different A/B matrices and their effect on the jump introduced by a rotation angle of 92 deg is shown in Figure.~\ref{fig:AB_ablation}.

\begin{figure}[!htb]
    \centering
    \includegraphics[width=0.95\linewidth]{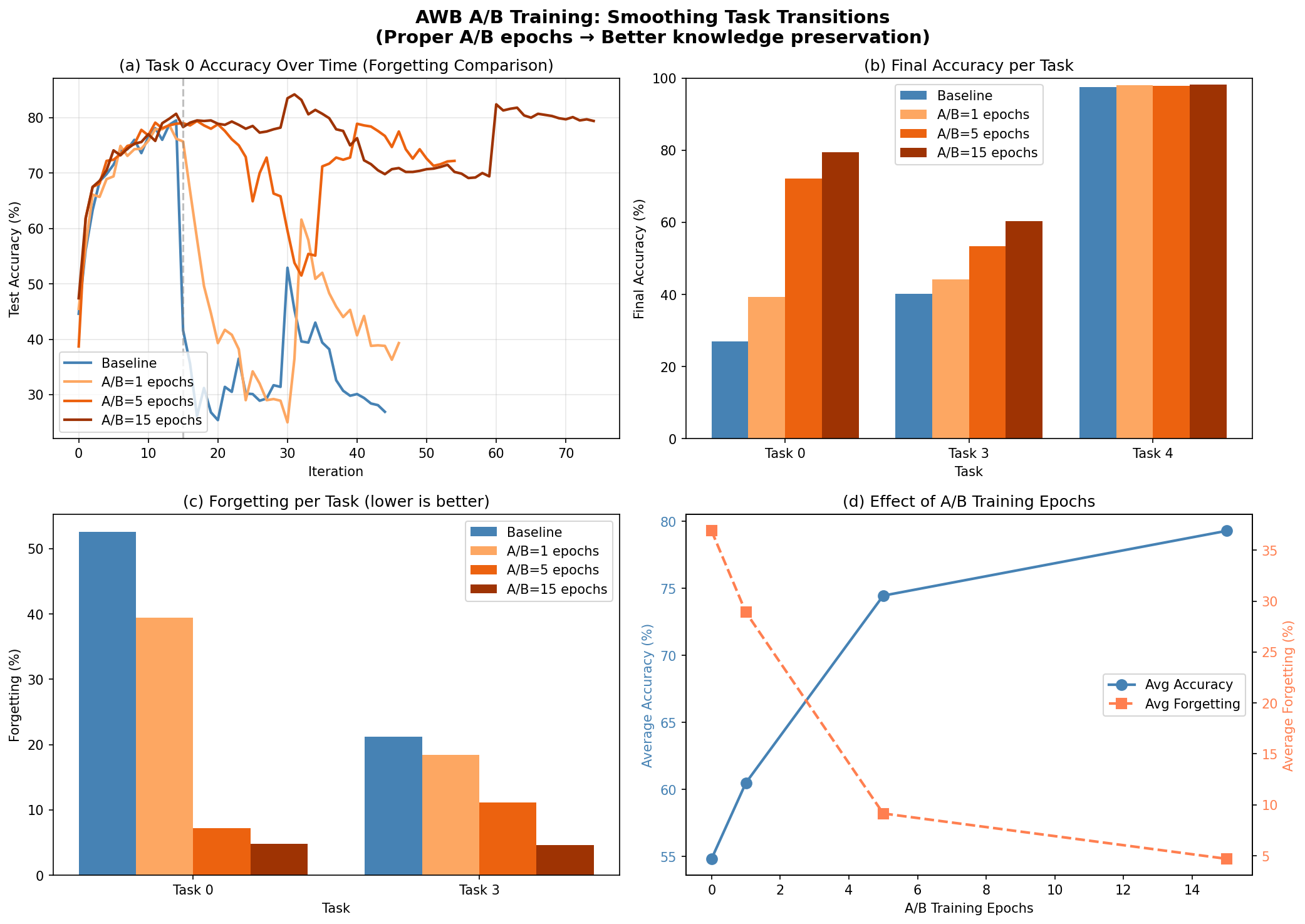}
    \caption{Comparison of how different number of epochs effect the performance of the smoothing.}
    \label{fig:AB_ablation}
\end{figure}

It is clear from Figure.~\ref{fig:AB_ablation} that, with an increase in the number of epochs for AB training, the test accuracy improves and forgetting reduces. This increase in the number of AB training epochs improves the smoothing of the forgetting costs over tasks and thus improves forgetting. 

\subsection{Graph Classification}
To further evaluate the AWB framework under controlled domain shift, we investigate the graph classfication problem.

\textbf{Dataset and Task Split:} As stated at the beginning of Section \ref{sec: experiments}, generate random graphs using PyTorch Geometric's FakeDataset. Each graph contains an average of 15 nodes with 10-dimensional node features and belongs to one of 5 classes for graph-level classification. We emulate task distribution shifts in this setting by allowing all tasks to share the same five classes but we introduce progressive perturbations at each additional task learned. More specifically, at each task $t,$ we apply three types of perturbations: (1) Gaussian feature noise with standard deviation $\sigma_t = 0.02t$, (2) random edge dropout with probability $p_t = 0.01t$, and (3) constant feature shift of magnitude $s_t = 0.01t$. Thus, Task $0$ serves as the unperturbed baseline. This is an optimal choice for the continual learning setting, as it tests the models ability to adapt to distribution shift while retaining knowledge.

\textbf{Hyperparameter Choices:} We utilize a Graph Convolutional Network (GCN) consisting of two graph convolution layers with hidden dimensions $[10 \rightarrow 32 \rightarrow 32]$, followed by max pooling and a 3-layer feedforward classifier with dimensions $[32 \rightarrow 32 \rightarrow 16 \rightarrow 5]$. All layers use LeakyReLU activations except the final classification layer. The model is trained using the Adam optimizer with learning rate $\eta = 10^{-4}$, batch size 128, and gradient clipping at norm 1.0. The Hamiltonian gradient is computed with weights $[\alpha, \beta, \gamma] = [0.4, 0.4, 0.1]$ for current task loss, experience replay loss, and regularization respectively. Each task is trained for 200 epochs with task warm-up enabled (20 epochs at 10\% learning rate). For the AWB Full condition (Condition 4), architecture search is performed with the following settings: 5 preliminary training epochs, 150 A/B matrix training epochs, 50 warmup epochs, and modified Adam betas $(0.5, 0.999)$ for V-training to enable faster adaptation after the warmup-to-main transition.

\textbf{Results:} Figure~\ref{fig:synthetic_graph_results} presents the comprehensive results across three experimental conditions: Baseline (fixed architecture, constant learning rate), Heuristics (task warmup, adaptive learning rate), and AWB Full (architecture search with A/B transfer). The top row shows training dynamics: (a) test accuracy on experience replay, (b) Hamiltonian loss, and (c) gradient norm over the 10-task sequence. The bottom row presents summary metrics: (d) average accuracy, (e) backward transfer (BWT), and (f) forgetting.

\begin{figure}[!htb]
    \centering
    \includegraphics[width=1\linewidth]{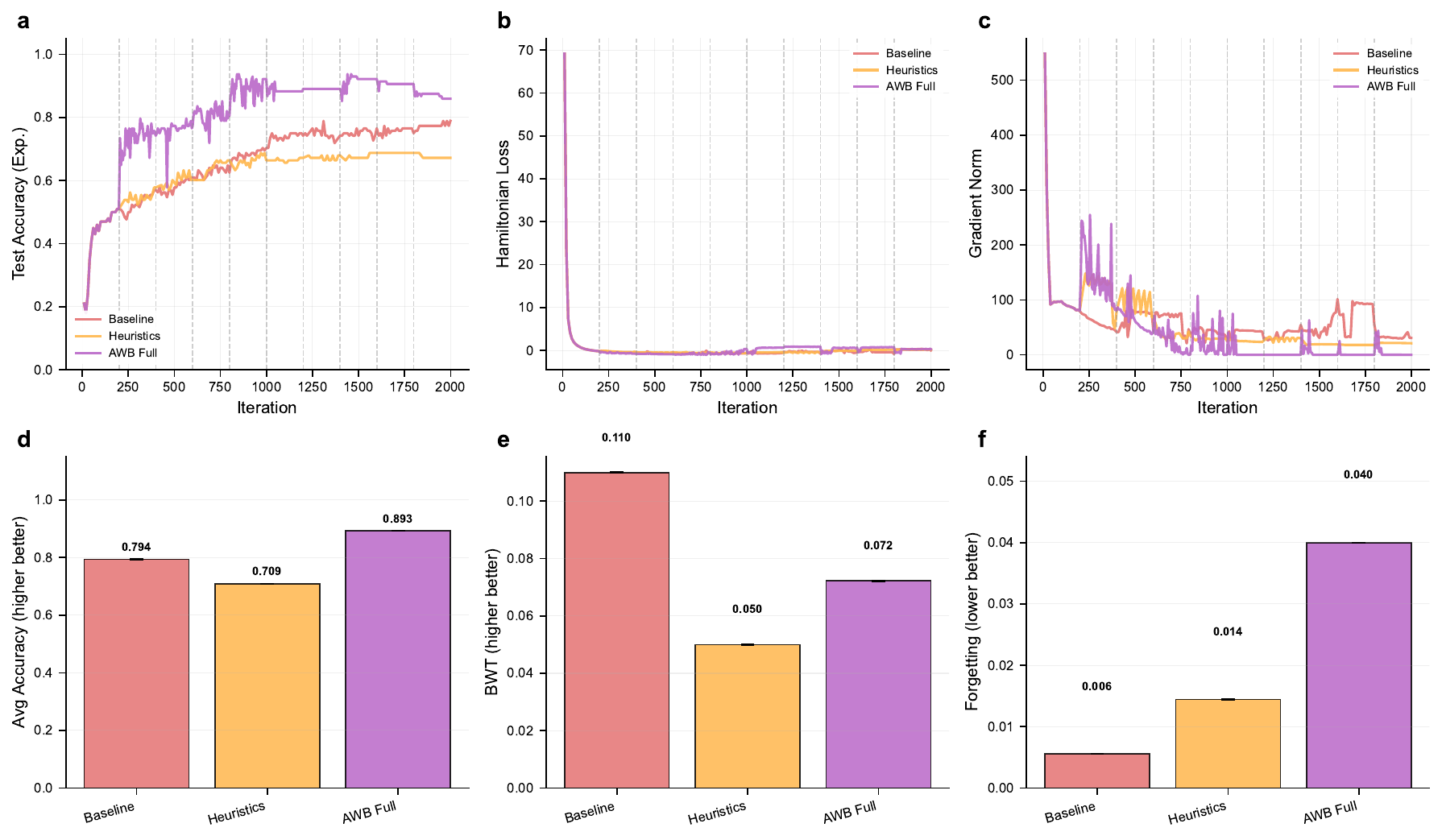}
    \caption{Synthetic graph classification with domain shift (10 tasks, 3 seeds). Top row: (a) Test accuracy on experience replay shows AWB Full maintaining high accuracy throughout training; (b) Hamiltonian loss; (c) Gradient norm. Bottom row: (d) AWB Full achieves 89.3\% average accuracy vs 79.4\% for Baseline; (e) Backward transfer; (f) Forgetting measure.}
    \label{fig:synthetic_graph_results}
\end{figure}

Figure~\ref{fig:task_accuracy_noise} provides deeper insight into the accuracy-forgetting trade-off by showing per-task final accuracy alongside the increasing noise level. Despite the linearly increasing perturbation (dashed line), AWB Full maintains consistently high accuracy (87--92\%) across all tasks, while baseline exhibits flat but lower accuracy (74--81\%).

\begin{figure}[!htb]
    \centering
    \includegraphics[width=0.95\linewidth]{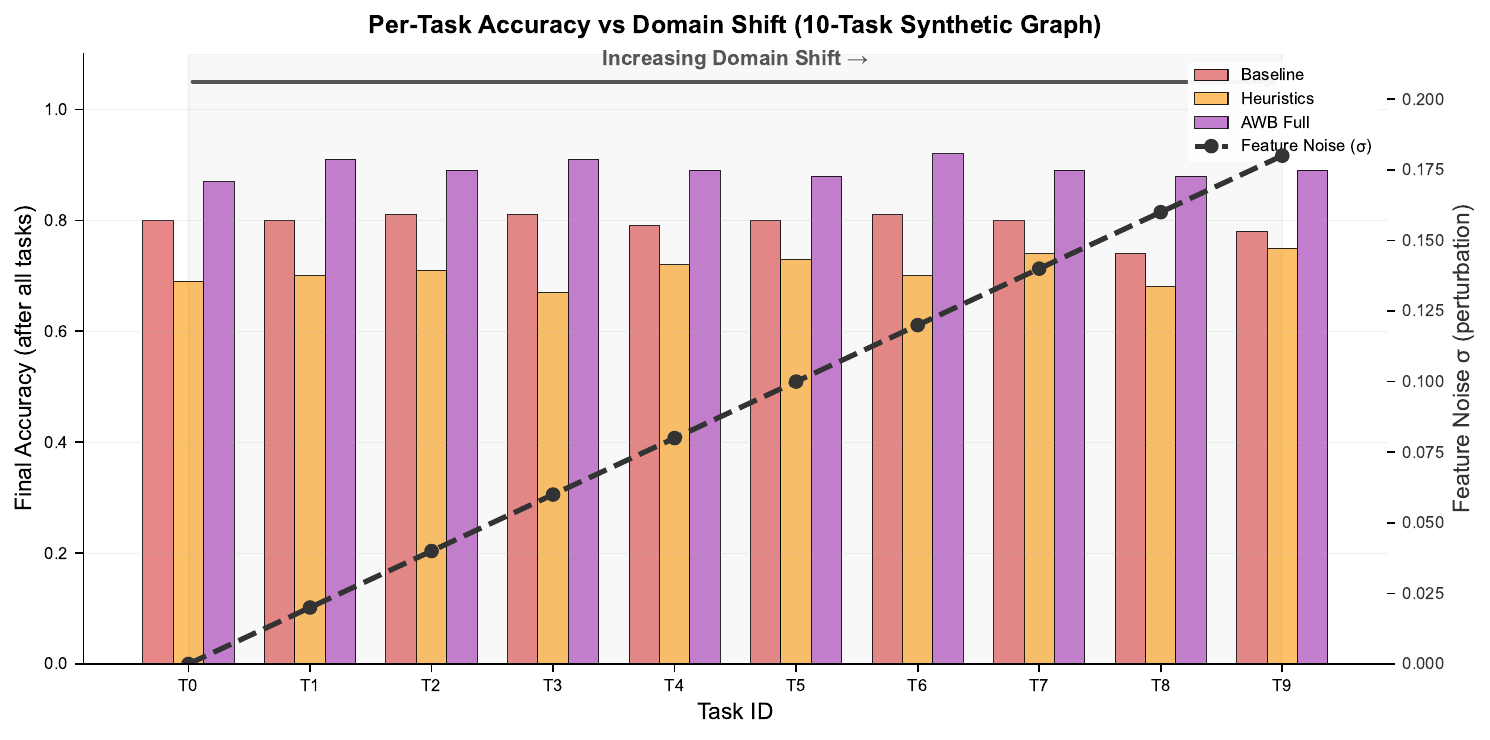}
    \caption{Per-task accuracy vs domain shift. Grouped bars show final accuracy on each task after training all 10 tasks. The dashed line indicates increasing feature noise ($\sigma$). AWB Full maintains high accuracy despite increasing perturbation, while baseline shows flat but lower performance.}
    \label{fig:task_accuracy_noise}
\end{figure}

\begin{table}[!htb]
\centering
\caption{Accuracy range summary for synthetic graph experiments. AWB Full achieves significantly higher mean accuracy with a tighter range, indicating consistent performance despite domain shift.}
\label{tab:accuracy_range}
\begin{tabular}{lccccl}
\toprule
\textbf{Method} & \textbf{Min Acc} & \textbf{Max Acc} & \textbf{Range} & \textbf{Mean $\pm$ Std} & \textbf{Interpretation} \\
\midrule
Baseline   & 0.74 & 0.81 & 0.07 & $0.794 \pm 0.020$ & Flat accuracy, limited learning \\
Heuristics & 0.67 & 0.75 & 0.08 & $0.709 \pm 0.025$ & Variable performance \\
AWB Full   & 0.87 & 0.92 & \textbf{0.05} & $\mathbf{0.893 \pm 0.015}$ & High accuracy, marginal forgetting \\
\bottomrule
\end{tabular}
\end{table}

\begin{table}[!htb]
\centering
\caption{Forgetting statistics for synthetic graph experiments. While AWB Full shows forgetting on more tasks, the absolute magnitude remains small (3.6\% average), representing a favorable trade-off for the 12.5\% accuracy improvement.}
\label{tab:forgetting_stats}
\begin{tabular}{lccc}
\toprule
\textbf{Method} & \textbf{Avg Forgetting} & \textbf{Max Forgetting} & \textbf{Tasks w/ Forgetting} \\
\midrule
Baseline   & \textbf{0.005} & 0.03 & 2/10 \\
Heuristics & 0.013 & 0.04 & 5/10 \\
AWB Full   & 0.036 & 0.07 & 8/10 \\
\bottomrule
\end{tabular}
\end{table}

These results reveal a favorable accuracy-forgetting trade-off for the AWB approach. The baseline model displays accuracy of 74--81\% and range of 7\% across all domain shifts. Thus, it neither learns effectively from new domains nor does it forget significantly. On the other hand, the AWB Full condition (C4) achieves 12.5\% higher average accuracy (89.3\% vs 79.4\%) while maintaining a tighter accuracy range (5\% vs 7\%). This demonstrates a consistent high performance despite the increased perturbation.

The forgetting analysis shown in (Table~\ref{tab:forgetting_stats}) exhibits that AWB Full has marginal forgetting (3.6\% average) on 8 out of 10 tasks, compared to near-zero forgetting for the baseline C1 condition (0.5\% average on 2 tasks). However, this small forgetting cost is offset by dramatically superior learning capacity, as AWB Full (C4) starts at higher accuracy and remains above baseline (C1) even after forgetting. Thus, the low-rank transfer mechanism preserves learning while also adapting to new domains.

\section{Conclusion}
In this paper we elucidated the dependency between the architecture, weights, and  loss function through a Sobolev space design. We derive the necessary condition for the existence of the continual learning problem and show that it requires the forgetting cost to be absolutely continuous. We show both theoretically and with experimentation that just changing the weights of the network is not enough to satisfy the necessary condition, and we devise an approach to change the architecture of the model on the fly. We show  that  our approach  substantially improvement the training and test performance even in the presence of noise in the tasks. We  also  show that, for a series of neural network architectures, we have better performance when changing architectures while learning continually on the series of tasks. 

\section*{Acknowledgments:} 
The first author AH was supported by the SWARM project, supported by the Department of Energy Award DE-SC0024387. The second author, KR, was supported by the U.S. Department of Energy, Office of Science (SC), Advanced Scientific Computing Research (ASCR), Competitive Portfolios Project on Energy Efficient Computing: A Holistic Methodology, under Contract DE-AC02-06CH11357. We also acknowledge the support by the U.S. Department of Energy for the SciDAC 6 RAPIDS institute. This research used resources of the Argonne Leadership Computing Facility, which is a DOE Office of Science User Facility supported under Contract DE-AC02- 06CH11357. We also thank the initial conversation with Dr. Stefan Wild, Lawrence Berkeley National Laboratory, which led to all of this work. We have utilized generative AI, particularly, claude-code from Anthropic to generate the plots, tables in the experiment. Some of the graphics and text in the appendix has been generated, \emph{claude-code}. Rest of the paper is the original product of the authors.

\newpage
\bibliography{NNbib.bib}
\bibliographystyle{tmlr}
\appendix
\section{Appendix}
\section{Proofs}
\subsection{Proof of Lemma~\ref{lem:upper_weight}}~\label{sec:upper_weight}
\begin{proof}
    By Lemma \ref{lem:taylor} and the assumption that $\mu\left( \bigcup_{\tau = 0}^{t}\tx(\tau)\Delta \bigcup_{\tau = 0}^{t+\Delta t}\tx(\tau)\right)\geq \delta,$ we have
    \begin{align*}
       E\left(\bigcup_{\tau = 0}^{t}\tx(\tau)\right) & = E\left(\bigcup_{\tau = 0}^{t+\Delta t}\tx(\tau)\right) - \Delta t\Big[\mu\left( \bigcup_{\tau = 0}^{t}\tx(\tau)\Delta \bigcup_{\tau = 0}^{t+\Delta t}\tx(\tau)\right)\cdot \displaystyle\int_{\bigcup_{\tau = 0}^{t}\tx(\tau)} \ell(\hat{f}(\weight,\psi))d\mu\\
        & +\displaystyle\int_{\bigcup_{\tau = 0}^{t}\tx(\tau)} \ell'(\hat{f}(\weight,\psi))\cdot \partial_{\weight}^1 \hat{f}(\weight,\psi)\cdot \Delta w \hspace{1mm}  d\mu\Big]- o(\Delta t)\\
        & = E\left(\bigcup_{\tau = 0}^{t+\Delta t}\tx(\tau)\right) + \Delta t\left(-\mu\left( \bigcup_{\tau = 0}^{t}\tx(\tau)\Delta \bigcup_{\tau = 0}^{t+\Delta t}\tx(\tau)\right)\right)\cdot \displaystyle\int_{\bigcup_{\tau = 0}^{t}\tx(\tau)\Delta \bigcup_{\tau = 0}^{t+\Delta t}\tx(\tau)} \ell(\hat{f}(\weight,\psi))d\mu\\
        & - \Delta t\displaystyle\int_{\bigcup_{\tau = 0}^{t}\tx(\tau)} \ell'(\hat{f}(\weight,\psi))\cdot \partial_{\weight}^1 \hat{f}(\weight,\psi)\cdot \Delta w \hspace{1mm}  d\mu- o(\Delta t)\\
        & \leq E\left(\bigcup_{\tau = 0}^{t+\Delta t}\tx(\tau)\right) + \Delta t \cdot \delta \cdot \displaystyle\int_{\bigcup_{\tau = 0}^{t}\tx(\tau)\Delta \bigcup_{\tau = 0}^{t+\Delta t}\tx(\tau)} \ell(\hat{f}(\weight,\psi))d\mu\\
        & - \Delta t\displaystyle\int_{\bigcup_{\tau = 0}^{t}\tx(\tau)} \ell'(\hat{f}(\weight,\psi))\cdot \partial_{\weight}^1 \hat{f}(\weight,\psi)\cdot \Delta w \hspace{1mm}  d\mu- o(\Delta t)\\
    \end{align*}
    Subtracting $E\left(\bigcup_{\tau = 0}^{t+\Delta t}\tx(\tau)\right)$ from both sides and combining like terms,
    \begin{align*}
         E\left(\bigcup_{\tau = 0}^{t}\tx(\tau)\right) - E\left(\bigcup_{\tau = 0}^{t+\Delta t}\tx(\tau)\right)
         & \leq \Delta t \cdot \delta \cdot \displaystyle\int_{\bigcup_{\tau = 0}^{t}\tx(\tau)\Delta \bigcup_{\tau = 0}^{t+\Delta t}\tx(\tau)} \ell(\hat{f}(\weight,\psi))d\mu\\
        & - \Delta t\displaystyle\int_{\bigcup_{\tau = 0}^{t}\tx(\tau)} \ell'(\hat{f}(\weight,\psi))\cdot \partial_{\weight}^1 \hat{f}(\weight,\psi)\cdot \Delta w \hspace{1mm}  d\mu- o(\Delta t).
    \end{align*}
    Additionally, we can assume that the loss function $\ell$ is bounded above by a constant $M_0,$ so
    \begin{align*}
        E\left(\bigcup_{\tau = 0}^{t}\tx(\tau)\right) - E\left(\bigcup_{\tau = 0}^{t+\Delta t}\tx(\tau)\right)
        & \leq \Delta t \cdot \delta \cdot M_0- \Delta t\displaystyle\int_{\bigcup_{\tau = 0}^{t}\tx(\tau)} \ell'(\hat{f}(\weight,\psi))\cdot \partial_{\weight}^1 \hat{f}(\weight,\psi)\cdot \Delta w \hspace{1mm}  d\mu- o(\Delta t).
    \end{align*}
    To consider this difference in expected values across future tasks, set $\Delta t =1 $ and let us sum across such future tasks, 
    \begin{align*}
         \sum_{\tau = t}^{T} \left(E\left(\bigcup_{\nu = 0}^{\tau}\tx(\nu)\right) - E\left(\bigcup_{\nu = 0}^{\tau+1}\tx(\nu)\right)\right)
        & \leq \sum_{\tau =t}^{T} \Big( M_0\cdot \delta - \displaystyle\int_{\bigcup_{\nu = 0}^{\tau}\tx(\nu)} \ell'(\hat{f}(\weight,\psi))\cdot \partial_{\weight}^1 \hat{f}(\weight,\psi)\cdot \Delta w \hspace{1mm}  d\mu \Big).
    \end{align*}
    Now, we can choose the following for each $\tau$
    \begin{align*}
        \Delta \weight & =  \ell'(\hat{f}(\weight,\psi))\cdot \partial_{\weight}^1 \hat{f}(\weight,\psi).
    \end{align*}
    Thus,
    \begin{align*}
         \sum_{\tau = t}^{T} \left(E\left(\bigcup_{\nu = 0}^{\tau}\tx(\nu)\right) - E\left(\bigcup_{\nu = 0}^{\tau+1}\tx(\nu)\right)\right)
        & \leq \sum_{\tau =t}^{T} \Big( M_0\cdot \delta - \displaystyle\int_{\bigcup_{\nu = 0}^\tau\tx(\nu)} \left(\ell'(\hat{f}(\weight,\psi))\cdot \partial_{\weight}^1 \hat{f}(\weight,\psi)\right)^2 \hspace{1mm}  d\mu \Big).
    \end{align*}
    Now, notice the following
    \begin{align*}
        J(\weight(\tau),\psi(\tau),\tx(\tau)) & = \int_{\bigcup_{\nu = 0}^\tau \tx(\nu)} \ell(\hat{f}(\weight,\psi))d\mu = E\left(\bigcup_{\nu = 0}^{\tau}\tx(\nu)\right).
    \end{align*}
    Thus,
    \begin{align*}
         \sum_{\tau = t}^{T} \left(J(\tx(\tau)) - J(\tx(\tau+1))\right)
         & = \sum_{\tau = t}^{T} \left(E\left(\bigcup_{\nu = 0}^{\tau}\tx(\nu)\right) - E\left(\bigcup_{\nu = 0}^{\tau+1}\tx(\nu)\right)\right)\\
        & \leq \sum_{\tau =t}^{T} \Big( M_0\cdot \delta - \displaystyle\int_{\bigcup_{\nu = 0}^\tau\tx(\nu)} \left(\ell'(\hat{f}(\weight,\psi))\cdot \partial_{\weight}^1 \hat{f}(\weight,\psi)\right)^2 \hspace{1mm}  d\mu \Big),
    \end{align*}
    as desired.
\end{proof}

\subsection{Proof of Lemma~\ref{lem:upper_arch}}~\label{sec:upper_arch}
\begin{proof}
    By Lemma \ref{lem:taylor} and the assumption that $\mu\left( \bigcup_{\tau = 0}^{t}\tx(\tau)\Delta \bigcup_{\tau = 0}^{t+\Delta t}\tx(\tau)\right)\geq \delta,$ we have
    \begin{align*}
       E\left(\bigcup_{\tau = 0}^{t}\tx(\tau)\right) & = E\left(\bigcup_{\tau = 0}^{t+\Delta t}\tx(\tau)\right) - \Delta t\Big[\mu\left( \bigcup_{\tau = 0}^{t}\tx(\tau)\Delta \bigcup_{\tau = 0}^{t+\Delta t}\tx(\tau)\right)\cdot \displaystyle\int_{\bigcup_{\tau = 0}^{t}\tx(\tau)} \ell(\hat{f}(\weight,\psi))d\mu\\
        & +\displaystyle\int_{\bigcup_{\tau = 0}^{t}\tx(\tau)} \ell'(\hat{f}(\weight,\psi))\cdot \partial_{\weight}^1 \hat{f}(\weight,\psi)\cdot \Delta w \hspace{1mm}  d\mu\\
        & +\displaystyle\int_{\bigcup_{\tau = 0}^{t}\tx(\tau)} \ell'(\hat{f}(\weight,\psi))\cdot \partial_{\psi}^1 \hat{f}(\weight,\psi)\cdot \Delta \psi \hspace{1mm}  d\mu\Big]- o(\Delta t)\\
        & = E\left(\bigcup_{\tau = 0}^{t+\Delta t}\tx(\tau)\right) + \Delta t\left(-\mu\left( \bigcup_{\tau = 0}^{t}\tx(\tau)\Delta \bigcup_{\tau = 0}^{t+\Delta t}\tx(\tau)\right)\right)\cdot \displaystyle\int_{\bigcup_{\tau = 0}^{t}\tx(\tau)\Delta \bigcup_{\tau = 0}^{t+\Delta t}\tx(\tau)} \ell(\hat{f}(\weight,\psi))d\mu\\
        & - \Delta t\displaystyle\int_{\bigcup_{\tau = 0}^{t}\tx(\tau)} \ell'(\hat{f}(\weight,\psi))\cdot \partial_{\weight}^1 \hat{f}(\weight,\psi)\cdot \Delta w \hspace{1mm}  d\mu\\
        &-\Delta t \displaystyle\int_{\bigcup_{\tau = 0}^{t}\tx(\tau)} \ell'(\hat{f}(\weight,\psi))\cdot \partial_{\psi}^1 \hat{f}(\weight,\psi)\cdot \Delta \psi \hspace{1mm}  d\mu- o(\Delta t)\\
        & \leq E\left(\bigcup_{\tau = 0}^{t+\Delta t}\tx(\tau)\right) + \Delta t \cdot \delta \cdot \displaystyle\int_{\bigcup_{\tau = 0}^{t}\tx(\tau)\Delta \bigcup_{\tau = 0}^{t+\Delta t}\tx(\tau)} \ell(\hat{f}(\weight,\psi))d\mu\\
        & - \Delta t\displaystyle\int_{\bigcup_{\tau = 0}^{t}\tx(\tau)} \ell'(\hat{f}(\weight,\psi))\cdot \partial_{\weight}^1 \hat{f}(\weight,\psi)\cdot \Delta w \hspace{1mm}  d\mu\\
        & -\Delta t \displaystyle\int_{\bigcup_{\tau = 0}^{t}\tx(\tau)} \ell'(\hat{f}(\weight,\psi))\cdot \partial_{\psi}^1 \hat{f}(\weight,\psi)\cdot \Delta \psi \hspace{1mm}  d\mu- o(\Delta t).
    \end{align*}
    Subtracting $E\left(\bigcup_{\tau = 0}^{t+\Delta t}\tx(\tau)\right)$ from both sides and combining like terms,
    \begin{align*}
         E\left(\bigcup_{\tau = 0}^{t}\tx(\tau)\right) - E\left(\bigcup_{\tau = 0}^{t+\Delta t}\tx(\tau)\right)
         & \leq \Delta t \cdot \delta \cdot \displaystyle\int_{\bigcup_{\tau = 0}^{t}\tx(\tau)\Delta \bigcup_{\tau = 0}^{t+\Delta t}\tx(\tau)} \ell(\hat{f}(\weight,\psi))d\mu\\
        & - \Delta t\displaystyle\int_{\bigcup_{\tau = 0}^{t}\tx(\tau)} \ell'(\hat{f}(\weight,\psi))\cdot \partial_{\weight}^1 \hat{f}(\weight,\psi)\cdot \Delta w \hspace{1mm}  d\mu\\
        & -\Delta t \displaystyle\int_{\bigcup_{\tau = 0}^{t}\tx(\tau)} \ell'(\hat{f}(\weight,\psi))\cdot \partial_{\psi}^1 \hat{f}(\weight,\psi)\cdot \Delta \psi \hspace{1mm}  d\mu- o(\Delta t).
    \end{align*}
    Additionally, we can assume that the loss function $\ell$ is bounded above by a constant $M_0,$ so
    \begin{align*}
        E\left(\bigcup_{\tau = 0}^{t}\tx(\tau)\right) - E\left(\bigcup_{\tau = 0}^{t+\Delta t}\tx(\tau)\right)
        & \leq \Delta t \cdot \delta \cdot M_0- \Delta t\displaystyle\int_{\bigcup_{\tau = 0}^{t}\tx(\tau)} \ell'(\hat{f}(\weight,\psi))\cdot \partial_{\weight}^1 \hat{f}(\weight,\psi)\cdot \Delta w \hspace{1mm}  d\mu\\
        & -\Delta t \displaystyle\int_{\bigcup_{\tau = 0}^{t}\tx(\tau)} \ell'(\hat{f}(\weight,\psi))\cdot \partial_{\psi}^1 \hat{f}(\weight,\psi)\cdot \Delta \psi \hspace{1mm}  d\mu- o(\Delta t).
    \end{align*}
    To consider this difference in expected values across future tasks, set $\Delta t =1 $ and let us sum across such future tasks, 
    \begin{align*}
         \sum_{\tau = t}^{T} \left(E\left(\bigcup_{\nu = 0}^{\tau}\tx(\nu)\right) - E\left(\bigcup_{\nu = 0}^{\tau+1}\tx(\nu)\right)\right)
        & \leq \sum_{\tau =t}^{T} \Big( M_0\cdot \delta - \displaystyle\int_{\bigcup_{\nu = 0}^{\tau}\tx(\nu)} \ell'(\hat{f}(\weight,\psi))\cdot \partial_{\weight}^1 \hat{f}(\weight,\psi)\cdot \Delta w \hspace{1mm}  d\mu\\
        & -\displaystyle\int_{\bigcup_{\tau = 0}^{t}\tx(\tau)} \ell'(\hat{f}(\weight,\psi))\cdot \partial_{\psi}^1 \hat{f}(\weight,\psi)\cdot \Delta \psi \hspace{1mm}  d\mu\Big).
    \end{align*}
    Now, we can choose the following for each $\nu$
    \begin{align*}
        \Delta \weight & =  \ell'(\hat{f}(\weight,\psi))\cdot \partial_{\weight}^1 \hat{f}(\weight,\psi).
    \end{align*}
    and
    \begin{align*}
        \Delta \psi & =  \ell'(\hat{f}(\weight,\psi))\cdot \partial_{\psi}^1 \hat{f}(\weight,\psi).
    \end{align*}
    Thus,
    \begin{align*}
         \sum_{\tau = t}^{T} \left(E\left(\bigcup_{\nu = 0}^{\tau}\tx(\nu)\right) - E\left(\bigcup_{\nu = 0}^{\tau+1}\tx(\nu)\right)\right)
        & \leq \sum_{\tau =t}^{T} \Big( M_0\cdot \delta - \displaystyle\int_{\bigcup_{\nu = 0}^\tau\tx(\nu)} \left(\ell'(\hat{f}(\weight,\psi))\cdot \partial_{\weight}^1 \hat{f}(\weight,\psi)\right)^2 \hspace{1mm}  d\mu\\
        & -\displaystyle\int_{\bigcup_{\nu = 0}^\tau\tx(\nu)} \left(\ell'(\hat{f}(\weight,\psi))\cdot \partial_{\psi}^1 \hat{f}(\weight,\psi)\right)^2 \hspace{1mm}  d\mu \Big).
    \end{align*}
    Now, notice the following
    \begin{align*}
        J(\weight(\tau),\psi(\tau),\tx(\tau)) & = \int_{\bigcup_{\nu = 0}^\tau \tx(\nu)} \ell(\hat{f}(\weight,\psi))d\mu = E\left(\bigcup_{\nu = 0}^{\tau}\tx(\nu)\right).
    \end{align*}
    Thus,
    \begin{align*}
         \sum_{\tau = t}^{T} \left(J(\tx(\tau)) - J(\tx(\tau+1))\right)
         & = \sum_{\tau = t}^{T} \left(E\left(\bigcup_{\nu = 0}^{\tau}\tx(\nu)\right) - E\left(\bigcup_{\nu = 0}^{\tau+1}\tx(\nu)\right)\right)\\
        & \leq \sum_{\tau =t}^{T} \Big( M_0\cdot \delta - \displaystyle\int_{\bigcup_{\nu = 0}^\tau\tx(\nu)} \left(\ell'(\hat{f}(\weight,\psi))\cdot \partial_{\weight}^1 \hat{f}(\weight,\psi)\right)^2 \hspace{1mm}  d\mu\\
        & -\displaystyle\int_{\bigcup_{\nu = 0}^\tau\tx(\nu)} \left(\ell'(\hat{f}(\weight,\psi))\cdot \partial_{\psi}^1 \hat{f}(\weight,\psi)\right)^2 \hspace{1mm}  d\mu \Big)\\
    \end{align*}
    as desired.
\end{proof}

\subsection{Proof of Proposition~\ref{HJB}}\label{sec:HJB}
\begin{proof}
    Let $J, V,$ and $V^*$ be as defined above. To begin, we split the sum in $V(t)$ over the discrete intervals $[t,t+\Delta t]$ and $[t+\Delta t, T].$ Observe,
    \begin{align}\label{b}
        V^*(t) & = \min_{w\in \mathcal{W}(\psi^*(t))} \int_{t}^T J(w(\tau),\psi^*(t), \tx(\tau))d\tau\nonumber\\
                & = \min_{w\in \mathcal{W}(\psi^*(t))} \biggr[\int_{t}^{t+\Delta t} J(w(\tau),\psi^*(t), \tx(\tau))d\tau+ \int_{t+\Delta t}^{T} J(w(\tau),\psi^*(t), \tx(\tau))d\tau\biggr]\nonumber\\
                & = \min_{w\in \mathcal{W}(\psi^*(t))} \int_{t}^{t+\Delta t} J(w(\tau),\psi^*(t),\tx(\tau))d\tau + V^*(t+\Delta t)\nonumber\\
                & = \min_{w\in \mathcal{W}(\psi^*(t))} J(\weight(t),\psi^*(t),\tx(t))\Delta t + V^*(t+\Delta t).
    \end{align}
    Now, we provide the Taylor series expansion of $V^*(t+\Delta t)$ about $t.$ Notice,
    \begin{align}
        V^*(t+\Delta t) & = V^*(t) + \Delta t \left[\partials_{t}V^*(t)  + \partials_{\tx}V^*(t) d_{t} \tx + \partials_{\weight}V^*(t)  \partials_t \weight \right] + o(\Delta t)\nonumber\\
                & = V^*(t) + \Delta t\biggr[\partials_{t} V^*(t)  + \partials_{\tx}V^*(t) d_{t} \tx + \partials_{\weight}V^*(t) [\mathrm{A}^*(t)\weight(t)( \mathrm{B}^*(t))^T + u(t)]\biggr] + o(\Delta t),\label{a}
    \end{align}
    where $u(t)$ represents the updates made to the each weights matrix of the new dimensions.
    Substituting \ref{a} into \ref{b}, we have
    \begin{align*}
        V^*(t) & = \min_{w\in \mathcal{W}(\psi^*(t))} J(\weight(t),\psi^*(t),\tx(t))\Delta t+ V^*(t) + \Delta t\biggr[\partials_{t} V^*(t)  + \partials_{\tx}V^*(t) d_{t} \tx\\ 
                &+ \partials_{\weight}V^*(t) [\mathrm{A}^*(t)\weight(t)( \mathrm{B}^*(t))^T + u(t)]\biggr] + o(\Delta t).
    \end{align*}
    Canceling $V^*(t)$ gives
    \begin{align*}
        0 & = \min_{w\in \mathcal{W}(\psi^*(t))} J(\weight(t),\psi^*(t),\tx(t))\Delta t+ \Delta t\left[\partials_{t}V^*(t)  + \partials_{\tx}V^*(t) d_{t} \tx \right.  \\ &+
        \left. \partials_{\weight}V^*(t) [\mathrm{A}^*(t)\weight(t)( \mathrm{B}^*(t))^T + u(t)]\right] + o(\Delta t).
    \end{align*}
    Dividing both sides by $\Delta t$ produces
    \begin{align*}
        0 & = \min_{w\in \mathcal{W}(\psi^*(t))} J(\weight(t),\psi^*(t),\tx(t)) +
        \partials_{t}V^*(t)  + \partials_{\tx}V^*(t) V^*(t) d_{t} \tx + \partials_{\weight}V^*(t) [\mathrm{A}^*(t)\weight(t)( \mathrm{B}^*(t))^T + u(t)].
    \end{align*}
    Finally, reordering gives
    \begin{align*}
        -\partials_{t} V^*(t) & = \min_{w\in \mathcal{W}(\psi^*(t))} J(\weight(t),\psi^*(t),\tx(t))+ \partials_{\tx}V^*(t) d_{t} \tx + \partials_{\weight}V^*(t) V^*(t) [\mathrm{A}^*(t)\weight(t)( \mathrm{B}^*(t))^T + u(t)],
    \end{align*}
    as desired.
\end{proof}

\subsection{Proof of Theorem~\ref{thm:lower_HJB}}\label{sec:lower_HJB}
\begin{proof}
Given proposition~\ref{HJB}, assume there exists an stochastic gradient based optimization procedure  and choose $ u(t) =- \sum^{I} \alpha(i) \mathrm{g}^{(i)}$  such that $\min_{\weight \in \mathcal{W}(\psi^*(t))} J(\weight(t),\psi^*(t),\tx(t)) \leq \varepsilon,$ where $I$ is the number of updates. Then, we have 
\begin{align*}
    -\partials_{t}V^*(t) & \leq \varepsilon + \partials_{\tx}V^*(t) d_{t} \tx +  \partials_{\weight}V^*(t)  \left( \mathrm{A}^*(t)\weight(t)( \mathrm{B}^*(t))^T - \sum^{I} \alpha(i) \mathrm{g}^{(i)}  \right)  \\    
\end{align*}
For any $dt$ in terms of tasks $t$  let $\partials_{\tx}V^*(t) d_{t} \tx \leq \norm{\partials_{\tx}V^*(t) } \norm{d_{t} \tx} \leq \sing \delta$ as $\norm{d_{t} \tx}  \geq \mu(\tx(t)\bigtriangleup \tx(t+1)) \geq \delta$ and $\norm{\partials_{\tx}V^*(t) }$ is less than the largest singular value of $\partials_{\tx}V^*(t) .$  Thus we write 
\begin{align*}
    -\partials_{t}V^*(t) & \leq \varepsilon + \sing \delta +  \partials_{\weight}V^*(t)  \left( \mathrm{A}^*(t)\weight(t)( \mathrm{B}^*(t))^T - \sum^{I} \alpha(i) \mathrm{g}^{(i)}  \right) 
\end{align*}
Taking $g_{MIN}$ to be the smallest value of the gradient over all update iterations, we have 
\begin{align*}
    -\partials_{t} & \leq \varepsilon + \sing \delta +  \partials_{\weight}V^*(t)  \left( \mathrm{A}^*(t)\weight(t) (\mathrm{B}^*(t))^T - \mathrm{g}_{\mathrm{MIN} } \sum^{I} \alpha(i)  \right) \\
     & \leq \varepsilon + \sing \delta +  \norm{\partials_{\weight}V^*(t) } \norm{\mathrm{A}^*(t)\weight(t)( \mathrm{B}^*(t))^T} - \partials_{\weight}V^*(t)  \left(  \mathrm{g}_{\mathrm{MIN} } \sum^{I} \alpha(i)  \right) \\
     & \leq \varepsilon + \left[\sing \delta +  \singw \norm{\mathrm{A}^*(t)\weight(t)( \mathrm{B}^*(t))^T} - \singmin \norm{\mathrm{g}_{\mathrm{MIN} }} \norm{\sum^{I} \alpha(i) } \right]
\end{align*}
This finally provides 
\begin{align*}
    \partials_{t} & \leq \mathrm{sup} \varepsilon - \mathrm{inf} \left[\sing \delta +  \singw \norm{\mathrm{A}^*(t)\weight(t)( \mathrm{B}^*(t))^T} - \singmin \norm{\mathrm{g}_{\mathrm{MIN} }} \norm{\sum^{I} \alpha(i) } \right]
\end{align*}
For the total variation to be upper bounded, we need the quantity in the brackets to go to zero, this provides.
\begin{align*}
\left[\sing \delta +  \singw \norm{\mathrm{A}^*(t)\weight(t)( \mathrm{B}^*(t))^T} - \singmin \norm{\mathrm{g}_{\mathrm{MIN} }} \norm{\sum^{I} \alpha(i) } \right] = 0 \\
\end{align*}
\begin{align*}
   \sing \delta &= \left[ \singmin \norm{\mathrm{g}_{\mathrm{MIN} }}  \norm{\sum^{I} \alpha(i) } - \singw \norm{\mathrm{A}^*(t)\weight(t)( \mathrm{B}^*(t))^T} \right]
\end{align*}
providing the result
\end{proof}










\newpage

\section{Feedforawrd Neural Network $AWB$ Matrix Calculation Example}\label{app: MLP AB}
In this section, complete an explicit example determining the sizes for $A$ and $B$ matrices for each layer of a feedforward neural network to transfer learning from the architecture of $[64,128,128,10]$ to the architecture of $[64,256,256,10].$ 

\subsection{Original Architecture and Weights Matrix Specifications}
We begin with the following architecture:
\begin{equation}
\text{Original architecture} = [64, 128, 128, 10]
\end{equation}

This defines a 3-layer network:
\begin{itemize}
    \item Layer 0: $64 \rightarrow 128$ (input layer to first hidden)
    \item Layer 1: $128 \rightarrow 128$ (first hidden to second hidden)
    \item Layer 2: $128 \rightarrow 10$ (second hidden to output)
\end{itemize}

This 3-layer network provides us with the following weights matrices:
\begin{align}
\mathbf{W}_0 &\in \mathbb{R}^{128 \times 64} \quad \text{(maps $64 \rightarrow 128$)} \\
\mathbf{W}_1 &\in \mathbb{R}^{128 \times 128} \quad \text{(maps $128 \rightarrow 128$)} \\
\mathbf{W}_2 &\in \mathbb{R}^{10 \times 128} \quad \text{(maps $128 \rightarrow 10$)}
\end{align}

\subsection{Optimal Architecture Specifications}
After completing an architecture search, it was determined that the following architecture was optimal:
\begin{equation}
\text{new\_arch} = [64, 256, 256, 10]
\end{equation}

This expands hidden layers while preserving input (64) and output (10):
\begin{itemize}
    \item Layer 0: $64 \rightarrow 256$
    \item Layer 1: $256 \rightarrow 256$
    \item Layer 2: $256 \rightarrow 10$
\end{itemize}

\subsection{Determining Appropriate $A$ and $B$ Matrices}

Our goal is to obtain weights matrices $V$ such that which correspond to the new architecture. This is completed by determing matrices $A_i$ and $B_i$ for each layer $i$ such that $A_iW_iB_i^T = V_i.$  The transformation matrices are constructed as:

\begin{align}
\mathbf{A}_i &\in \mathbb{R}^{\text{new\_out}_i \times \text{old\_out}_i} \\
\mathbf{B}_i &\in \mathbb{R}^{\text{new\_in}_i \times \text{old\_in}_i}
\end{align}

For each layer:

\textbf{Layer 0:}
\begin{align}
\mathbf{A}_0 &\in \mathbb{R}^{256 \times 128} \quad \text{(output: $128 \rightarrow 256$)} \\
\mathbf{B}_0 &\in \mathbb{R}^{64 \times 64} \quad \text{(input: $64 \rightarrow 64$, identity-like)}
\end{align}

\textbf{Layer 1:}
\begin{align}
\mathbf{A}_1 &\in \mathbb{R}^{256 \times 128} \quad \text{(output: $128 \rightarrow 256$)} \\
\mathbf{B}_1 &\in \mathbb{R}^{256 \times 128} \quad \text{(input: $128 \rightarrow 256$)}
\end{align}

\textbf{Layer 2:}
\begin{align}
\mathbf{A}_2 &\in \mathbb{R}^{10 \times 10} \quad \text{(output: $10 \rightarrow 10$, identity-like)} \\
\mathbf{B}_2 &\in \mathbb{R}^{256 \times 128} \quad \text{(input: $128 \rightarrow 256$)}
\end{align}

\subsection{Computing and Verifying Matrix $V$}

Below we visually represent each layer transformation and complete a dimension verification.

\textbf{Layer 0 Transformation:}

\begin{equation}
\mathbf{V}_0 = \mathbf{A}_0 \mathbf{W}_0 \mathbf{B}_0^T
\end{equation}

\begin{center}
\begin{tikzpicture}[scale=0.8]
    \draw[fill=matrixA, draw=black, thick] (0,1) rectangle (2,5);
    \draw[step=0.5,black,thin] (0,1) grid (2,5);
    \node at (1, 3) {\Large $\mathbf{A}_0$};
    \node[below] at (1, 0.7) {128};
    \node[left] at (-0.3, 4) {256};

    \node at (2.7, 3) {\huge $\cdot$};

    \draw[fill=matrixW, draw=black, thick] (3.4,1) rectangle (5.4,5);
    \draw[step=0.5,black,thin] (3.4,1) grid (5.4,5);
    \node at (4.4, 3) {\Large $\mathbf{W}_0$};
    \node[below] at (4.4, 0.7) {64};
    \node[left] at (3.1, 4) {128};

    \node at (6.1, 3) {\huge $\cdot$};

    \draw[fill=matrixB, draw=black, thick] (6.8,1) rectangle (8.8,5);
    \draw[step=0.5,black,thin] (6.8,1) grid (8.8,5);
    \node at (7.8, 3) {\Large $\mathbf{B}_0^T$};
    \node[below] at (7.8, 0.7) {64};
    \node[left] at (6.5, 4) {64};

    \node at (9.5, 3) {\huge $=$};

    \draw[fill=matrixV, draw=black, thick] (10.2,1) rectangle (12.2,5);
    \draw[step=0.5,black,thin] (10.2,1) grid (12.2,5);
    \node at (11.2, 3) {\Large $\mathbf{V}_0$};
    \node[below] at (11.2, 0.7) {64};
    \node[left] at (9.9, 4) {256};
\end{tikzpicture}
\end{center}

\textbf{Dimension verification:}
\begin{align}
\mathbf{V}_0 &= \mathbf{A}_0 \mathbf{W}_0 \mathbf{B}_0^T \\
&= (256 \times 128) \cdot (128 \times 64) \cdot (64 \times 64) \\
&= (256 \times 64) \quad \checkmark
\end{align}

\textbf{Layer 1 Transformation}

\begin{equation}
\mathbf{V}_1 = \mathbf{A}_1 \mathbf{W}_1 \mathbf{B}_1^T
\end{equation}

\begin{center}
\begin{tikzpicture}[scale=0.8]
    \draw[fill=matrixA, draw=black, thick] (0,1) rectangle (2,5);
    \draw[step=0.5,black,thin] (0,1) grid (2,5);
    \node at (1, 3) {\Large $\mathbf{A}_1$};
    \node[below] at (1, 0.7) {128};
    \node[left] at (-0.3, 3.6) {256};

    \node at (2.7, 3) {\huge $\cdot$};

    \draw[fill=matrixW, draw=black, thick] (3.4,1) rectangle (5.4,5);
    \draw[step=0.5,black,thin] (3.4,1) grid (5.4,5);
    \node at (4.4, 3) {\Large $\mathbf{W}_1$};
    \node[below] at (4.4, 0.7) {128};
    \node[left] at (3.1, 3.6) {128};

    \node at (6.1, 3) {\huge $\cdot$};

    \draw[fill=matrixB, draw=black, thick] (6.8,1.5) rectangle (10.8,4.5);
    \draw[step=0.5,black,thin] (6.8,1.5) grid (10.8,4.5);
    \node at (8.8, 3) {\Large $\mathbf{B}_1^T$};
    \node[below] at (8.8, 1.2) {256};
    \node[left] at (6.5, 3.6) {128};

    \node at (11.5, 3) {\huge $=$};

    \draw[fill=matrixV, draw=black, thick] (12.2,1.5) rectangle (16.2,4.5);
    \draw[step=0.5,black,thin] (12.2,1.5) grid (16.2,4.5);
    \node at (14.2, 3) {\Large $\mathbf{V}_1$};
    \node[below] at (14.2, 1.2) {256};
    \node[left] at (11.9, 3.6) {256};
\end{tikzpicture}
\end{center}

\textbf{Dimension verification:}
\begin{align}
\mathbf{V}_1 &= \mathbf{A}_1 \mathbf{W}_1 \mathbf{B}_1^T \\
&= (256 \times 128) \cdot (128 \times 128) \cdot (128 \times 256) \\
&= (256 \times 256) \quad \checkmark
\end{align}

\textbf{Layer 2 Transformation}

\begin{equation}
\mathbf{V}_2 = \mathbf{A}_2 \mathbf{W}_2 \mathbf{B}_2^T
\end{equation}

\begin{center}
\begin{tikzpicture}[scale=0.8]
    \draw[fill=matrixA, draw=black, thick] (0,1.5) rectangle (1.5,3.5);
    \draw[step=0.5,black,thin] (0,1.5) grid (1.5,3.5);
    \node at (0.75, 2.5) {\Large $\mathbf{A}_2$};
    \node[below] at (0.75, 1.2) {10};
    \node[left] at (-0.3, 3.6) {10};

    \node at (2.2, 2.5) {\huge $\cdot$};

    \draw[fill=matrixW, draw=black, thick] (2.9,1.5) rectangle (6.9,3.5);
    \draw[step=0.5,black,thin] (2.9,1.5) grid (6.9,3.5);
    \node at (4.9, 2.5) {\Large $\mathbf{W}_2$};
    \node[below] at (4.9, 1.2) {128};
    \node[left] at (2.6, 3.6) {10};

    \node at (7.6, 2.5) {\huge $\cdot$};

    \draw[fill=matrixB, draw=black, thick] (8.3,1) rectangle (12.3,4);
    \draw[step=0.5,black,thin] (8.3,1) grid (12.3,4);
    \node at (10.3, 2.5) {\Large $\mathbf{B}_2^T$};
    \node[below] at (10.3, 0.7) {256};
    \node[left] at (8, 3.6) {128};

    \node at (13, 2.5) {\huge $=$};

    \draw[fill=matrixV, draw=black, thick] (13.7,1) rectangle (17.7,4);
    \draw[step=0.5,black,thin] (13.7,1) grid (17.7,4);
    \node at (15.7, 2.5) {\Large $\mathbf{V}_2$};
    \node[below] at (15.7, 0.7) {256};
    \node[left] at (13.4, 3.6) {10};
\end{tikzpicture}
\end{center}

\textbf{Dimension verification:}
\begin{align}
\mathbf{V}_2 &= \mathbf{A}_2 \mathbf{W}_2 \mathbf{B}_2^T \\
&= (10 \times 10) \cdot (10 \times 128) \cdot (128 \times 256) \\
&= (10 \times 256) \quad \checkmark
\end{align}

\subsection{Summary}
In the table below, we summarize the matrix size information. 

\begin{table}[h]
\centering
\begin{tabular}{@{}lcccc@{}}
\toprule
Layer & Original $\mathbf{W}$ & $\mathbf{A}$ & $\mathbf{B}$ & Result $\mathbf{V = AWB^T}$ \\
\midrule
0 & $128 \times 64$ & $256 \times 128$ & $64 \times 64$ & $256 \times 64$ \\
1 & $128 \times 128$ & $256 \times 128$ & $256 \times 128$ & $256 \times 256$ \\
2 & $10 \times 128$ & $10 \times 10$ & $256 \times 128$ & $10 \times 256$ \\
\bottomrule
\end{tabular}
\caption{FFN transformation dimensions}
\end{table}

\newpage

\section{CNN3D AWB Calculation}\label{app: CNN3d AB}\label{app: CNN AB}
CNN processes single-channel images (e.g., MNIST) with one convolutional layer followed by feed-forward layers. In this section, we complete an explicit example that determines the sizes for $A$ and $B$ matrices for each layer of a convolutional neural network to transfer learning. This includes the $A$ and $B$ matrices for both the filter and the feedforward portion of the neural network.

\subsection{Original Architecture}
We begin with the following original architecture: 
\begin{equation}
\text{feed\_sizes} = [2304, 256, 10]
\end{equation}
and
\begin{equation}
    \text{filter\_size} = 3.
\end{equation}
After two conv+pool layers, the flattened feature size is 2304. This feeds into:
\begin{itemize}
    \item Layer 0: $2304 \rightarrow 256$
    \item Layer 1: $256 \rightarrow 10$
\end{itemize}

\subsection{Optimal Architecture Specifications}

After completing an architecture search, it was determined that the following architecture was optimal:
\begin{equation}
\text{new\_arch} = [1600, 512, 10]
\end{equation}
and
\begin{equation}
    \text{filter\_size} = 5.
\end{equation}

\subsection{Determining Appropriate $A$ and $B$ Matrices}

We need to obtain weights matrices for both the covolutional layer and the feed forward layers which correspond to the new architecture. This is completed by determines transfer matrices. We let $A_\text{filter}$ and $B_\text{filter}$ be the transfer matrices used to determine the new weights, denoted $V_\text{filter}$ for the convolutional layer. We let $A_i$ and $B_i$ be the transfer matrices used to determine the new weights, denoted $V_i$ for the feedforward layers. We break these calculations into two sections for clarity.\\

\textbf{Convolutional Filter Transformation}\\

Recall that the original filter size is $3$ and the new filter size is $5.$ For a single convolutional filter $\mathbf{W}_{\text{filter}}^{(i,c)}$ where $i$ is the output channel and $c$ is the input channel:

\begin{equation}
\mathbf{V}_{\text{filter}}^{(i,c)} = \mathbf{A}_{\text{conv}}^{(i,c)} \mathbf{W}_{\text{filter}}^{(i,c)} \mathbf{B}_{\text{conv}}^{(i,c)T}
\end{equation}

\begin{center}
\begin{tikzpicture}[scale=1.0]
    \draw[fill=matrixA, draw=black, thick] (0,1) rectangle (1.5,3.5);
    \draw[step=0.5,black,thin] (0,1) grid (1.5,3.5);
    \node at (0.75, 2.25) {$\mathbf{A}$};
    \node[below] at (0.75, .7) {3};
    \node[left] at (-0.3, 3) {5};

    \node at (2.1, 2.25) {$\cdot$};

    \draw[fill=matrixW, draw=black, thick] (2.7,1) rectangle (4.2,3.5);
    \draw[step=0.5,black,thin] (2.7,1) grid (4.2,3.5);
    \node at (3.45, 2.25) {$\mathbf{W}$};
    \node[below] at (3.45, 0.7) {3};
    \node[left] at (2.4, 3) {3};

    \node at (4.8, 2.25) {$\cdot$};

    \draw[fill=matrixB, draw=black, thick] (5.4,1.5) rectangle (7.9,3);
    \draw[step=0.5,black,thin] (5.4,1.5) grid (7.9,3);
    \node at (6.65, 2.25) {$\mathbf{B}^T$};
    \node[below] at (6.65, 1.2) {5};
    \node[left] at (5.1, 3) {3};

    \node at (8.5, 2.25) {$=$};

    \draw[fill=matrixV, draw=black, thick] (9.1,1) rectangle (11.6,3.5);
    \draw[step=0.5,black,thin] (9.1,1) grid (11.6,3.5);
    \node at (10.35, 2.25) {$\mathbf{V}$};
    \node[below] at (10.35, 0.7) {5};
    \node[left] at (8.8, 3) {5};
\end{tikzpicture}
\end{center}

\textbf{Dimension verification:}
\begin{align}
\mathbf{V}_{\text{filter}} &= \mathbf{A} \mathbf{W}_{\text{filter}} \mathbf{B}^T \\
&= (5 \times 3) \cdot (3 \times 3) \cdot (3 \times 5) \\
&= (5 \times 5) \quad \checkmark
\end{align}

This transformation is applied to each filter in each convolutional layer, expanding the spatial kernel size from $3 \times 3$ to $5 \times 5$. Note that a single-channel would be used when training on the MNIST dataset. When using RGB dataset, we would use a three-channel network.

\textbf{Feed-Forward Layers}

With expanded filters (5×5 instead of 3×3), the flattened size becomes 1600:
\begin{itemize}
    \item Layer 0: $1600 \rightarrow 512$
    \item Layer 1: $512 \rightarrow 10$
\end{itemize}

\textbf{Feed Layer 0 Transformation}

\begin{equation}
\mathbf{V}_0 = \mathbf{A}_0 \mathbf{W}_0 \mathbf{B}_0^T
\end{equation}

\begin{center}
\begin{tikzpicture}[scale=0.7]
    \draw[fill=matrixA, draw=black, thick] (0,1) rectangle (2,6);
    \draw[step=0.5,black,thin] (0,1) grid (2,6);
    \node at (1, 3.5) {\Large $\mathbf{A}_0$};
    \node[below] at (1, 0.7) {256};
    \node[left] at (-0.3, 4.5) {512};

    \node at (2.7, 3.5) {\huge $\cdot$};

    \draw[fill=matrixW, draw=black, thick] (3.4,2) rectangle (9.4,5);
    \draw[step=0.5,black,thin] (3.4,2) grid (9.4,5);
    \node at (6.4, 3.5) {\Large $\mathbf{W}_0$};
    \node[below] at (6.4, 1.7) {2304};
    \node[left] at (3.1, 4.5) {256};

    \node at (10.1, 3.5) {\huge $\cdot$};

    \draw[fill=matrixB, draw=black, thick] (10.8,1.5) rectangle (16.8,5.5);
    \draw[step=0.5,black,thin] (10.8,1.5) grid (16.8,5.5);
    \node at (13.8, 3.5) {\Large $\mathbf{B}_0^T$};
    \node[below] at (13.8, 1.2) {1600};
    \node[left] at (10.5, 4.5) {2304};

    \node at (17.5, 3.5) {\huge $=$};

    \draw[fill=matrixV, draw=black, thick] (18.2,1.5) rectangle (24.2,5.5);
    \draw[step=0.5,black,thin] (18.2,1.5) grid (24.2,5.5);
    \node at (21.2, 3.5) {\Large $\mathbf{V}_0$};
    \node[below] at (21.2, 1.2) {1600};
    \node[left] at (17.9, 4.5) {512};
\end{tikzpicture}
\end{center}

\textbf{Dimension verification:}
\begin{align}
\mathbf{V}_0 &= \mathbf{A}_0 \mathbf{W}_0 \mathbf{B}_0^T \\
&= (512 \times 256) \cdot (256 \times 2304) \cdot (2304 \times 1600) \\
&= (512 \times 1600) \quad \checkmark
\end{align}

\textbf{Feed Layer 1 Transformation}

\begin{equation}
\mathbf{V}_1 = \mathbf{A}_1 \mathbf{W}_1 \mathbf{B}_1^T
\end{equation}

\begin{center}
\begin{tikzpicture}[scale=0.8]
    \draw[fill=matrixA, draw=black, thick] (0,1.5) rectangle (1.5,3.5);
    \draw[step=0.5,black,thin] (0,1.5) grid (1.5,3.5);
    \node at (0.75, 2.5) {\Large $\mathbf{A}_1$};
    \node[below] at (0.75, 1.2) {10};
    \node[left] at (-0.3, 4.5) {10};

    \node at (2.2, 2.5) {\huge $\cdot$};

    \draw[fill=matrixW, draw=black, thick] (2.9,1.5) rectangle (6.9,3.5);
    \draw[step=0.5,black,thin] (2.9,1.5) grid (6.9,3.5);
    \node at (4.9, 2.5) {\Large $\mathbf{W}_1$};
    \node[below] at (4.9, 1.2) {256};
    \node[left] at (2.6, 4.5) {10};

    \node at (7.6, 2.5) {\huge $\cdot$};

    \draw[fill=matrixB, draw=black, thick] (8.3,1) rectangle (12.3,4);
    \draw[step=0.5,black,thin] (8.3,1) grid (12.3,4);
    \node at (10.3, 2.5) {\Large $\mathbf{B}_1^T$};
    \node[below] at (10.3, 0.7) {512};
    \node[left] at (8, 4.5) {256};

    \node at (13, 2.5) {\huge $=$};

    \draw[fill=matrixV, draw=black, thick] (13.7,1) rectangle (17.7,4);
    \draw[step=0.5,black,thin] (13.7,1) grid (17.7,4);
    \node at (15.7, 2.5) {\Large $\mathbf{V}_1$};
    \node[below] at (15.7, 0.7) {512};
    \node[left] at (13.4, 4.5) {10};
\end{tikzpicture}
\end{center}

\textbf{Dimension verification:}
\begin{align}
\mathbf{V}_1 &= \mathbf{A}_1 \mathbf{W}_1 \mathbf{B}_1^T \\
&= (10 \times 10) \cdot (10 \times 256) \cdot (256 \times 512) \\
&= (10 \times 512) \quad \checkmark
\end{align}

\newpage

\section{CL Implementation \& Condition Details}
\label{app:standard_cl}

This appendix provides implementation details for the four experimental conditions (C1--C4) used in this paper. All algorithms use the notation from the main text: $\weight(t)$ denotes weights at task $t$, $\psi(t)$ denotes architecture, $\tx(t)$ denotes task data, $\fhat(\weight,\psi)$ denotes the neural network, and $\mathcal{E}$ denotes the experience replay buffer.

\subsection{Experimental Conditions Overview \& Algorithm}

Figure~\ref{fig:alg_conditions} presents the four experimental conditions as a unified algorithmic framework. Each condition builds upon the previous one:
\begin{itemize}
    \item \textbf{C1 (Baseline)}: Fixed architecture, constant learning rate, no warmup
    \item \textbf{C2 (Heuristics)}: Adds task warmup, cosine LR schedule, adaptive gradient weights
    \item \textbf{C3 (Arch Search)}: Adds architecture search at task boundaries with random reinitialization
    \item \textbf{C4 (AWB Full)}: Adds knowledge transfer via $A$, $B$ matrices instead of random reinitialization
\end{itemize}

\begin{figure}[!htb]
\centering

\begin{minipage}[t]{0.48\textwidth}
\begin{algorithm}[H]
\caption*{\textbf{(a) C1: Baseline}}
\label{alg:c1}
\KwIn{$\weight(0), \psi_0, \{\tx(t)\}_{t=0}^T$}
\KwOut{$\weight(T)$}
\textbf{Init:} $\eta = 10^{-4}$, $[\alpha,\beta,\gamma] = [0.4, 0.4, 0.1]$\;
\For{$t = 0$ \KwTo $T$}{
    \For{$e = 1$ \KwTo $N_{\text{epochs}}$}{
        $\mathcal{B}_c \gets$ \texttt{sample}$(\tx(t))$\;
        $\mathcal{B}_e \gets$ \texttt{sample}$(\mathcal{E})$\;
        $\nabla H \gets$ \textsc{HamGrad}$(\weight, \mathcal{B}_c, \mathcal{B}_e)$\;
        $\weight \gets \weight - \eta \cdot \nabla H$\;
    }
    $\mathcal{E} \gets \mathcal{E} \cup \tx(t)$\;
}
{\bfseries Return} $\weight$\;
\end{algorithm}
\vspace{0.3em}
{\footnotesize \textit{Fixed $\psi_0$, constant $\eta$, no warmup.}}
\end{minipage}
\hfill
\begin{minipage}[t]{0.48\textwidth}
\begin{algorithm}[H]
\caption*{\textbf{(b) C2: Heuristics}}
\label{alg:c2}
\KwIn{$\weight(0), \psi_0, \{\tx(t)\}_{t=0}^T$}
\KwOut{$\weight(T)$}
\textbf{Init:} $\eta_0 = 10^{-4}$, $[\alpha,\beta,\gamma] = [0.4, 0.4, 0.1]$\;
\For{$t = 0$ \KwTo $T$}{
    \If{$t > 0$}{
        $\weight \gets$ \textsc{Warmup}$(\weight, \tx(t))$ \tcp*[f]{Alg.~\ref{alg:warmup}}\;
        $[\alpha,\beta,\gamma] \gets$ \textsc{AdaptGrad}$(J_t, J_{t-1})$ \tcp*[f]{Alg.~\ref{alg:adaptive}}\;
    }
    \For{$e = 1$ \KwTo $N_{\text{epochs}}$}{
        $\mathcal{B}_c, \mathcal{B}_e \gets$ \texttt{sample}$(\tx(t)), $ \textsc{BalReplay}$(\mathcal{E})$\;
        $\nabla H \gets$ \textsc{HamGrad}$(\weight, \mathcal{B}_c, \mathcal{B}_e)$\;
        $\eta \gets$ \textsc{CosineDecay}$(\eta_0, e)$\;
        $\weight \gets \weight - \eta \cdot \nabla H$\;
    }
    $\mathcal{E} \gets \mathcal{E} \cup \tx(t)$\;
}
{\bfseries Return} $\weight$\;
\end{algorithm}
\vspace{0.3em}
{\footnotesize \textit{Fixed $\psi_0$, adds warmup + cosine LR + adaptive gradients.}}
\end{minipage}

\vspace{1em}

\begin{minipage}[t]{0.48\textwidth}
\begin{algorithm}[H]
\caption*{\textbf{(c) C3: Architecture Search}}
\label{alg:c3}
\KwIn{$\weight(0), \psi(0), \{\tx(t)\}_{t=0}^T$}
\KwOut{$\weight(T), \psi(T)$}
\textbf{Init:} As C2\;
\For{$t = 0$ \KwTo $T$}{
    \If{$t > 0$}{
        $\weight \gets$ \textsc{Warmup}$(\weight, \tx(t))$\;
        $[\alpha,\beta,\gamma] \gets$ \textsc{AdaptGrad}$(J_t, J_{t-1})$\;
        \If{\textsc{ShouldChange}$(J_t, J_{t-1})$}{
            $\psi^* \gets$ \textsc{ArchSearch}$(\psi, \tx(t))$ \tcp*[f]{Alg.~\ref{alg:arch_search}}\;
            $\weight \gets$ \textsc{GlorotInit}$(\psi^*)$ \tcp*[f]{Random reinit}\;
            $\psi \gets \psi^*$\;
        }
    }
    \For{$e = 1$ \KwTo $N_{\text{epochs}}$}{
        \tcp{Same training loop as C2}
        $\nabla H \gets$ \textsc{HamGrad}$(\weight, \mathcal{B}_c, \mathcal{B}_e)$\;
        $\weight \gets \weight - \eta \cdot \nabla H$\;
    }
    $\mathcal{E} \gets \mathcal{E} \cup \tx(t)$\;
}
{\bfseries Return} $\weight, \psi$\;
\end{algorithm}
\vspace{0.3em}
{\footnotesize \textit{Adaptive $\psi(t)$, random reinit after arch change.}}
\end{minipage}
\hfill
\begin{minipage}[t]{0.48\textwidth}
\begin{algorithm}[H]
\caption*{\textbf{(d) C4: AWB Full}}
\label{alg:c4}
\KwIn{$\weight(0), \psi(0), \{\tx(t)\}_{t=0}^T$}
\KwOut{$\weight(T), \psi(T)$}
\textbf{Init:} As C2\;
\For{$t = 0$ \KwTo $T$}{
    \If{$t > 0$}{
        $\weight \gets$ \textsc{Warmup}$(\weight, \tx(t))$\;
        $[\alpha,\beta,\gamma] \gets$ \textsc{AdaptGrad}$(J_t, J_{t-1})$\;
        \If{\textsc{ShouldChange}$(J_t, J_{t-1})$}{
            $\psi^* \gets$ \textsc{ArchSearch}$(\psi, \tx(t))$\;
            $A, B \gets$ \textsc{TrainAB}$(\weight, \psi, \psi^*, \tx(t))$ \tcp*[f]{Alg.~\ref{alg:train_ab}}\;
            $\weight \gets A \cdot \weight \cdot B^T$ \tcp*[f]{AWB transfer}\;
            $\psi \gets \psi^*$\;
        }
    }
    \For{$e = 1$ \KwTo $N_{\text{epochs}}$}{
        \tcp{Same training loop as C2}
        $\nabla H \gets$ \textsc{HamGrad}$(\weight, \mathcal{B}_c, \mathcal{B}_e)$\;
        $\weight \gets \weight - \eta \cdot \nabla H$\;
    }
    $\mathcal{E} \gets \mathcal{E} \cup \tx(t)$\;
}
{\bfseries Return} $\weight, \psi$\;
\end{algorithm}
\vspace{0.3em}
{\footnotesize \textit{Adaptive $\psi(t)$, knowledge transfer via $V = A\weight B^T$.}}
\end{minipage}

\caption{Four experimental conditions for Hamiltonian Continual Learning. (a) C1 uses fixed architecture and constant hyperparameters. (b) C2 adds heuristics: task warmup, cosine LR decay, and adaptive gradient weights. (c) C3 adds architecture search with random reinitialization. (d) C4 replaces random reinitialization with AWB transfer to preserve learned knowledge.}
\label{fig:alg_conditions}
\end{figure}

\subsection{Supporting Algorithm References}
\label{sec:supporting_algs}

The supporting algorithms referenced in Figure~\ref{fig:alg_conditions} are described as follows:

\begin{itemize}
\item \textbf{Task Warmup}\phantomsection\label{alg:warmup}: Initialize with reduced learning rate $\eta_w = 0.1 \cdot \eta_0$ and gradient weights $[\alpha,\beta,\gamma] = [1, 0, 0]$. Train on new task data for $N_{\text{warmup}}$ epochs before engaging experience replay.

\item \textbf{Adaptive Gradient Weights}\phantomsection\label{alg:adaptive}: Compute loss ratio $r = J_{\text{curr}} / J_{\text{prev}}$ and adjust weights: $\alpha = \min(0.7, 0.3 + 0.4(r-1))$, $\beta = \max(0.2, 0.6 - 0.4(r-1))$, $\gamma = 0.1$.

\item \textbf{Architecture Search}\phantomsection\label{alg:arch_search}: Extract layer dimensions $\{d_1, \ldots, d_L\}$, generate candidate architectures $\mathcal{S} = \{d_i \pm k \cdot 16 : k \in \{1,2,3\}\}$, select $\psi^* = \arg\min_{\psi' \in \mathcal{S}} J(\texttt{init}(\psi'), \psi', \tx_{\text{val}})$.

\item \textbf{Train A/B Matrices}\phantomsection\label{alg:train_ab}: Initialize $A, B$ as identity matrices. For $N_{\text{AB}}$ epochs: sample batch, compute $V = A \cdot \weight \cdot B^T$, update $A, B$ via gradient descent on reconstruction loss.

\item \textbf{Hamiltonian Gradient}\phantomsection\label{alg:hamiltonian}: Compute $\nabla_c = \nabla_{\weight} \ell(\mathcal{B}_{\text{curr}})$, $\nabla_e = \nabla_{\weight} \ell(\mathcal{B}_{\text{exp}})$, $\delta V = \textsc{PerturbationGrad}$. Return $\nabla H = \alpha \nabla_c + \beta \nabla_e + \gamma \delta V$.
\end{itemize}

\subsection{Hyperparameter Reference}

Table~\ref{tab:hyperparameters} summarizes all default hyperparameters used in the standard CL implementation. Tables~\ref{tab:sine_hyperparams} and~\ref{tab:mnist_hyperparams} provide dataset-specific configurations.

{\small
\begin{table}[h]
\centering
\caption{Default Hyperparameters}
\label{tab:hyperparameters}
\begin{tabular}{lll}
\hline
\textbf{Parameter} & \textbf{Default} & \textbf{Description} \\
\hline
\multicolumn{3}{l}{\textit{Gradient Computation}} \\
$\alpha$ & 0.4 & Current task gradient weight \\
$\beta$ & 0.4 & Experience replay gradient weight \\
$\gamma$ & 0.1 & Regularization gradient weight \\
$\sigma_x^2$ & $10^{-4}$ & Input perturbation variance \\
$\sigma_{\weight}^2$ & $10^{-8}$ & Parameter perturbation variance \\
\hline
\multicolumn{3}{l}{\textit{Optimization}} \\
$\eta_0$ & $10^{-4}$ & Base learning rate \\
$\eta_{\min}$ & $10^{-6}$ & Minimum learning rate \\
Optimizer & AdamW & Default optimizer \\
Gradient clip & 1.0 & Maximum gradient norm \\
LR schedule & Cosine & Learning rate decay (C2--C4) \\
\hline
\multicolumn{3}{l}{\textit{Task Warmup (C2--C4)}} \\
$N_{\text{warmup}}$ & 25 & Warmup epochs per task \\
LR factor & 0.1 & Warmup LR multiplier \\
\hline
\multicolumn{3}{l}{\textit{Experience Replay}} \\
Buffer size & 200,000 & Maximum samples \\
Recent quota & 10\% & From task $t-1$ \\
Older quota & 80\% & From tasks $0$ to $t-2$ \\
Random quota & 10\% & Uniform random \\
\hline
\multicolumn{3}{l}{\textit{AWB Pipeline (C4)}} \\
$N_{\text{prelim}}$ & 1--2 & Preliminary training epochs \\
$N_{\text{AB}}$ & 500 & A/B training epochs \\
$\eta_{\text{AB}}$ & $10^{-3}$ & A/B training learning rate \\
$\theta_{\text{loss}}$ & 1.1 & Loss ratio threshold for arch change \\
Search step & 16 & Dimension search step \\
\hline
\end{tabular}
\end{table}
}

{\small
\begin{table}[h]
\centering
\caption{Sine Regression Hyperparameters}
\label{tab:sine_hyperparams}
\begin{tabular}{lll}
\hline
\textbf{Parameter} & \textbf{Value} & \textbf{Description} \\
\hline
\multicolumn{3}{l}{\textit{Network Architecture}} \\
Network type & FCNN & Fully-connected neural network \\
Layers & 4 & Number of layers \\
Hidden size & 64 & Neurons per hidden layer \\
Activation & ReLU & Hidden layer activation \\
\hline
\multicolumn{3}{l}{\textit{Training}} \\
Tasks & 10 & Number of continual learning tasks \\
Epochs/task & 500 & Training epochs per task \\
Batch size & 1024 & Mini-batch size \\
Learning rate & $10^{-4}$ & Base learning rate \\
Optimizer & AdamW & Optimizer with weight decay \\
\hline
\multicolumn{3}{l}{\textit{Dataset}} \\
Input domain & $[-90, 90]$ & Uniform sampling domain for $x$ \\
Task variation & Amplitude, phase & How tasks differ \\
Train/test split & 80\%/20\% & Data partition \\
\hline
\end{tabular}
\end{table}
}

{\small
\begin{table}[h]
\centering
\caption{MNIST Classification Hyperparameters}
\label{tab:mnist_hyperparams}
\begin{tabular}{lll}
\hline
\textbf{Parameter} & \textbf{Value} & \textbf{Description} \\
\hline
\multicolumn{3}{l}{\textit{Network Architecture}} \\
Network type & CNN & Convolutional neural network \\
Conv layers & 2 & Number of convolutional layers \\
Filter size & 3 & Convolutional kernel size \\
Output channels & 3 & Channels per conv layer \\
Pool size & $2 \times 2$ & Max pooling kernel \\
FF hidden & [512, 64] & Feedforward hidden layer sizes \\
Activation & ReLU & Hidden layer activation \\
\hline
\multicolumn{3}{l}{\textit{Training}} \\
Tasks & 2 & Number of continual learning tasks \\
Epochs/task & 200 & Training epochs per task \\
Batch size & 1024 & Mini-batch size \\
Learning rate & $10^{-4}$ & Base learning rate \\
Optimizer & AdamW & Optimizer with weight decay \\
\hline
\multicolumn{3}{l}{\textit{Dataset}} \\
Input size & $28 \times 28$ & MNIST image dimensions \\
Task 0 & Digits 0--4 & First task classes \\
Task 1 & Digits 5--9 & Second task classes \\
Train/test split & 80\%/20\% & Data partition \\
\hline
\end{tabular}
\end{table}
}

\subsection{Implementation Notes}

\subsubsection{Gradient Normalization}

The $\nabla_{\weight} \delta V$ component is normalized by task count:
\begin{equation}
    \text{dV\_norm} = \frac{\|\nabla_{\weight} \delta V\|_2}{t + 1}
\end{equation}
This prevents the regularization term from dominating as more tasks are learned.

\subsubsection{Learning Rate Schedules}

The framework supports multiple LR schedules: constant, step, exponential, cosine, and linear. Default is cosine annealing with warm restarts.

\subsubsection{JAX Implementation}

All gradient computations are JIT-compiled using JAX with 8 variants for different problem types (regression/classification, standard/AWB, vector/graph).

\subsubsection{Model Partitioning}

Equinox models are partitioned for selective training: (1) standard training (all weights), (2) A/B training (freeze W), (3) V training (freeze A, B).

\begin{flushright}
  \scriptsize

  \framebox{\parbox{0.9\textwidth}{
  The submitted manuscript has been created by UChicago Argonne, LLC, Operator of Argonne National Laboratory (“Argonne”).
  Argonne, a U.S. Department of Energy Office of Science laboratory, is operated under Contract No. DE-AC02-06CH11357.
  The U.S. Government retains for itself, and others acting on its behalf, a paid-up nonexclusive, irrevocable worldwide
  license in said article to reproduce, prepare derivative works, distribute copies to the public, and perform publicly
  and display publicly, by or on behalf of the Government.  The Department of Energy will provide public access to these
  results of federally sponsored research in accordance with the DOE Public Access Plan.
  \url{http://energy.gov/downloads/doe-public-access-plan}
  }}
  \end{flushright}
\end{document}